\documentclass[11pt]{article}
\usepackage{amsmath}
\usepackage{amsthm}
\usepackage{amssymb}
\usepackage{graphicx}
\usepackage{mathtools}
\usepackage[margin=2cm]{geometry}
\usepackage{setspace}
\usepackage{epstopdf}
\usepackage{subfig}
\usepackage{algorithm}
\usepackage{algorithmicx}
\usepackage{algpseudocode}
\usepackage{multirow}
\usepackage{url}
\usepackage[title]{appendix}
\usepackage{enumitem}
\usepackage{sidecap}

\sidecaptionvpos{figure}{t}

\makeatletter

\theoremstyle{plain}
\newtheorem{thm}{\protect\theoremname}
  \theoremstyle{definition}
  \newtheorem{defn}[thm]{\protect\definitionname}
  \theoremstyle{plain}
  \newtheorem{lem}[thm]{\protect\lemmaname}
  \newtheorem{prop}[thm]{\protect\propname}

\makeatother

  \providecommand{\definitionname}{Definition}
  \providecommand{\lemmaname}{Lemma}
\providecommand{\theoremname}{Theorem}
\providecommand{\propname}{Proposition}

\usepackage{endnotes}
\usepackage{color}
\usepackage{xparse}

\newtheorem{remark}{Remark}

\newcommand{\man}{
\mathcal{M}
}

\newcommand{\CD}{
\mathbb{C}^\mathcal{D}
}

\DeclareMathOperator*{\argmin}{arg\,min}

\DeclareMathOperator*{\argmax}{arg\,max}

\begin{document}

\title{The steerable graph Laplacian \\ and its application to filtering image datasets}

\author{Boris Landa and Yoel Shkolnisky}


\maketitle

\bigskip

\noindent Boris Landa \\
Department of Applied Mathematics, School of Mathematical Sciences \\
Tel-Aviv University \\
{\tt sboris20@gmail.com}

\bigskip
\bigskip

\noindent Yoel Shkolnisky \\
Department of Applied Mathematics, School of Mathematical Sciences \\
Tel-Aviv University \\
{\tt yoelsh@post.tau.ac.il}

\bigskip
\bigskip

\begin{center}
Please address manuscript correspondence to Boris Landa,
{\tt sboris20@gmail.com}, (972)~549427603.
\end{center}
\newpage

\begin{abstract}
In recent years, improvements in various image acquisition techniques gave rise to the need for adaptive processing methods, aimed particularly for large datasets corrupted by noise and deformations. In this work, we consider datasets of images sampled from a low-dimensional manifold (i.e. an image-valued manifold), where the images can assume arbitrary planar rotations. To derive an adaptive and rotation-invariant framework for processing such datasets, we introduce a graph Laplacian (GL)-like operator over the dataset, termed \textit{steerable graph Laplacian}. Essentially, the steerable GL extends the standard GL by accounting for all (infinitely-many) planar rotations of all images. As it turns out, similarly to the standard GL, a properly normalized steerable GL converges to the Laplace-Beltrami operator on the low-dimensional manifold. However, the steerable GL admits an improved convergence rate compared to the GL, where the improved convergence behaves as if the intrinsic dimension of the underlying manifold is lower by one. Moreover, it is shown that the steerable GL admits eigenfunctions of the form of Fourier modes (along the orbits of the images' rotations) multiplied by eigenvectors of certain matrices, which can be computed efficiently by the FFT. For image datasets corrupted by noise, we employ a subset of these eigenfunctions to ``filter'' the dataset via a Fourier-like filtering scheme, essentially using all images and their rotations simultaneously. We demonstrate our filtering framework by de-noising simulated single-particle cryo-EM image datasets.
\end{abstract}

\section{Introduction}
Developing efficient and accurate processing methods for scientific image datasets is a central research task, which poses many theoretical and computational challenges.
In this work, motivated by certain experimental imaging and tomography problems~\cite{eisenstein2016field,thompson2016introduction}, we put our focus on the task of reducing the noise in a large dataset of images, where the in-plane rotation of each image is arbitrary. 

To accomplish such a task, it is generally required to have prior knowledge, or a model assumption, on the dataset at hand. One popular approach is to assume that the data lies on a low-dimensional linear subspace, whose parameters can then be estimated by the ubiquitous Principal Components Analysis (PCA). In our setting, where images admit arbitrary planar rotations, it is reasonable to incorporate all rotations of all images into the PCA procedure, resulting in what is known as ``steerable PCA'' (sPCA)~\cite{landa2017steerable,zhao2014fast,zhao2013fourier}.

In practice, however, experimental datasets typically admit more complicated non-linear structures. Therefore, we adopt the more flexible notion that the images were sampled from a low-dimensional manifold $\man$ embedded in a high-dimensional Euclidean space, an assumption that lies at the heart of many effective machine learning, dimensionality reduction, and signal processing techniques (see for example~\cite{roweis2000nonlinear,tenenbaum2000global,belkin2003laplacian,coifman2006diffusion,peyre2009manifold}).

When processing and analyzing manifold data, a fundamental object of interest is the Laplace-Beltrami operator $\Delta_\man$~\cite{rosenberg1997laplacian}, which encodes the geometry and topology of $\man$. Essentially, the Laplace-Beltrami operator is a second-order differential operator generalizing the classical Laplacian, and can be therefore considered as accounting for the smoothness of functions on $\man$. In this context, it is a common approach to leverage the Laplace-Beltrami operator and its discrete counterpart, the graph Laplacian~\cite{belkin2003laplacian}, to process surfaces, images, and general manifold data~\cite{meyer2014perturbation,taubin1995signal,vallet2008spectral,liu2014progressive,desbrun1999implicit,hein2006manifold,osher2017low}.
Incorporating the Laplacian in data processing algorithms typically follows one of two approaches. The first is based on solving an inverse problem which includes a regularization term involving the Laplacian of the estimated data coordinates, and the second is based on directly using the Laplacian or its eigenfunctions for filtering the dataset. We mention that the eigenfunctions of the Laplace-Beltrami operator, which we refer to as ``manifold harmonics'', are analogous to classical Fourier modes as they constitute a basis on $\man$ favorable for expanding smooth functions. Here, we focus on the the second approach, namely on filtering the dataset by the manifold harmonics. In particular, we assume (as mentioned above) that each point in the dataset is a high-dimensional point that lies on some manifold $\man$ with low intrinsic dimension. Thus, the coordinate functions of $\man$ (which are functions defined on the manifold) can be expanded by the manifold harmonics of $\man$ and filtered by truncating the expansion (see for example~\cite{meyer2014perturbation,vallet2008spectral} for similar approaches in the contexts of image processing and surface fairing).

As the manifold $\man$ is unknown a priori, we do not have access to its Laplace-Beltrami operator directly. Consequently, it must approximated from the data, which can be achieved through the graph Laplacian~\cite{belkin2003laplacian,coifman2006diffusion}. Specifically, given points $\left\{ x_1,\ldots,x_N\right\}\subset\mathbb{R}^\mathcal{D}$, we consider the fully connected graph Laplacian, denoted as $L\in\mathbb{R}^{N\times N}$ and given by
\begin{equation}
L = D-W, \quad\quad W_{i,j} = \exp{\left\{-\left\Vert x_i -x_j\right\Vert^2/\varepsilon\right\}}, \quad\quad D_{i,i} = \sum_{j=1}^N W_{i,j},	\label{eq:standard GL}
\end{equation}
where $W$ is known as the affinity matrix (using the Gaussian kernel parametrized by $\varepsilon$), and $D$ is a diagonal matrix with $\left\{D_{i,i}\right\}$ on its diagonal.
Then, as was shown in~\cite{coifman2006diffusion,singer2006graph,belkin2003laplacian,belkin2008towards},  the normalized graph Laplacian $\tilde{L} = D^{-1}L$ converges to the negative-defined Laplace-Beltrami operator $\Delta_\man$ when $\varepsilon\rightarrow 0$ and $N\rightarrow\infty$. In particular, it was shown in~\cite{singer2006graph} that for a smooth function $f:\man\rightarrow\mathbb{R}$
\begin{equation}
\frac{4}{\varepsilon} \sum_{j=1}^N\tilde{L}_{i,j} f(x_j) = \Delta_\man f(x_i) + O(\frac{1}{N^{1/2}\varepsilon^{1/2+d/4}}) + O(\varepsilon), \label{eq:standard GL convergence}
\end{equation}
where $d$ is the intrinsic dimension of $\man$.
Therefore, it is evident that for a fixed parameter $\varepsilon$, the error in the approximation of $\Delta_\man$ depends directly on the intrinsic dimension $d$ and inversely on the number of data points $N$. In this context, it is important to stress that the error does not depend on the dimension of the embedding space $\mathcal{D}$, but rather only on the intrisic dimension $d$ which is typically much smaller. If $d$ is large, then we need a large number of samples to achieve high accuracy. In our scenario, as images admit arbitrary planar rotations, the number of images required to use the approximation~\eqref{eq:standard GL convergence} may be prohibitively large as images which differ only by an in-plane rotation may not be encoded as similar by the affinity matrix $W$ (since the Euclidean distance between them may be large).
To overcome this obstacle, we construct the \textit{steerable graph Laplacian}, which is conceptually similar to the standard graph Laplacian, except that it also accounts for all rotations of the images in the dataset. We then propose to employ the eigenfunctions of this operator to filter our image dataset in a Fourier-like filtering scheme, allowing for an efficient procedure for mitigating noise.

Numerous works have proposed incorporating group-action invariance (and rotation invariance in particular) in image processing algorithms (see for example~\cite{eller2017rotation,zimmer2008rotationally,ji2009moment,singer2012vector,singer2011viewing,zhao2014rotationally} and references therin). The common approach towards rotation invariance is defining a rotationally-invariant distance for measureing pairwise affinities and constructiong graph Laplacians. Here, our approach is fundamentally different, as we consider not only the distance between best matching rotations of image pairs (nor any other type of a rotationally-invariant distance), but rather the standard (Euclidean) distance between all rotations of all pairs of images. This enables us to preserve the geometry of the underlying manifold (in contrast to various rotation-invariant distances) while making the resulting operator (the steerable graph Laplacian) invariant to rotations of the images in the dataset.
Furthermore, in the particular context of rotationally-invariant filtering and noise reduction, it is important to mention that classical algorithms such as~\cite{yu1996translation} are only applicable to one image at a time, whereas our approach builds upon large datasets of images and exploits all images simultaneously for noise reduction (see Sections~\ref{section:Analysis under Gaussian noise} and~\ref{section:Denoising rotationally-invariant image datasets}). 

The contributions of this paper are as follows. First, we introduce and analyze the steerable graph Laplacian operator, characterize its eigen-decomposition (together with a general family of operators), and show that it can be diagonalized by Fourier modes multiplied by eigenvectors of certain matrices. Second, we introduce the normalized steerable graph Laplacian, and demonstrate that it is more accurate than the standard graph Laplacian in approximating the Laplace-Beltrami operator, in the sense that it admits a smaller variance error term. Essentially, the improved variance error term can be obtained by replacing $d$ in equation~\eqref{eq:standard GL convergence} with $d-1$. Third, we propose to employ the eigenfunctions of the (normalized) steerable graph Laplacian for filtering image datasets, where the explicit appearance of Fourier modes in the form of the eigenfunctions allows for a particularly efficient filtering procedure. To motivate and justify our approach, we provide a bound on the error incurred by approximating an embedded manifold by a truncated expansion of its manifold harmonics. We also analyze our approach in the presence of white Gaussian noise, and argue that in a certain sense our method is robust to the noise, and moreover, allows us to reduce the amount of noise inversely to the number of images in the dataset.

The paper is organized as follows. Section~\ref{subsection:Problem se-tup} lays down the setting and provides the basic notation and assumptions. Then, Section~\ref{subsection:The steerable graph Laplacian for image-manifolds} defines the steerable graph Laplacian and derives some of its properties, including its eigen-decomposition. Section~\ref{subsection:Normalized steerable graph Laplacians and the Laplace-Beltrami operator} presents the normalized steerable graph Laplacian and derives its convergence rate to the Laplace-Beltrami operator, while providing its eigen-decomposition similarly to the preceding section. Section~\ref{subsection:Numerical example: Steerable graph Laplacian convergence rate} numerically corroborates the convergence rate of the normalized steerable graph Laplacian by a simple toy example, and Section~\ref{subsection:Expanding image datasets by the steerable manifold harmonics} proposes and analyzes a filtering scheme for image datasets based on the eigenfunctions of the (normalized) steerable graph Laplacian. Section~\ref{section:Algorithm summary and computational cost} summarizes all relevant algorithms and presents the computational complexities involved. Section~\ref{section:Analysis under Gaussian noise} provides an analysis of our approach in the presence of white Gaussian noise, followed by Section~\ref{section:Denoising rotationally-invariant image datasets} which demonstrates our method for de-noising a simulated cryo-EM image dataset. Lastly, Section~\ref{section:Summary and discussion} provides some concluding remarks and possible future research directions.

\section{Setting and main results} \label{section:Setting and main results}

\subsection{The setting} \label{subsection:Problem se-tup}
Suppose that we have $N$ points $\left\{x_1,\ldots,x_N\right\}\subset \CD$ sampled from a probability distribution $p(x)$, which is restricted to a smooth and compact $d$-dimensional submanifold $\man$ without boundary.
Furthermore, each point $x\in\man$ is associated with an image through a correspondence between points in the ambient space $\CD$ and images. 
Specifically, each point $x\in\CD$ corresponds to an image $I(r,\theta)\in\mathcal{L}^2(\mathbf{D})$, where $\mathbf{D}$ is the unit disk, by
\begin{equation}
I(r,\theta) = \sum_{m=-M}^M\sum_{\ell=1}^{\ell_m} x_{m,\ell} \psi_{m,\ell}(r,\theta), \quad\quad\quad \psi_{m,\ell}(r,\theta) = R_{m,\ell}(r) e^{\imath m \theta}, \label{eq:image subspace and psi polar form}
\end{equation} 
where $x_{m,\ell}$ is the $(m,\ell)$'th coordinate of $x$, and $\left\{ \psi_{m,\ell}\right\}$ is an orthogonal basis of $\mathcal{L}^2(\mathbf{D})$ whose radial part is $\left\{R_{m,\ell}\right\}_{\ell}$ (orthogonal on $[0,1)$ w.r.t the measure $r dr$). In other words, the points $x_i$ sampled from the manifold $\man$ are the expansion coefficients of some underlying images in the basis $\left\{\psi_{m,\ell}\right\}$. We mention that the points $x_i$ do not correspond to the pixels of the images directly since such a representation does not allow for a natural incorporation of planar rotations.
We shall refer to $m\in\mathbb{Z}$ as the angular index, and to $\ell\in\mathbb{N}_+$ as the radial index, where $\left\{\ell_m\right\}$ of~\eqref{eq:image subspace and psi polar form} are the numbers of radial indices taking part in the expansion for each angular index $m$, satisfying $\sum_{m=-M}^M \ell_m = \mathcal{D}$. Therefore, the dataset can be organized as the $N\times \mathcal{D}$ matrix
\begin{equation}
X = 
\begin{bmatrix}
\overbrace{
\begin{pmatrix}
x_{1,(-M,1)} & \ldots & x_{1,(-M,\ell_{-M})} \\
x_{2,(-M,1)} & \ldots & x_{2,(-M,\ell_{-M})} \\
\vdots & & \vdots \\
x_{N,(-M,1)} & \ldots & x_{N,(-M,\ell_{-M})}
\end{pmatrix}
}^{m=-M}
& \textbf{\ldots} &
\overbrace{
\begin{pmatrix}
x_{1,(M,1)} & \ldots & x_{1,(M,\ell_{M})} \\
x_{2,(M,1)} & \ldots & x_{2,(M,\ell_{M})} \\
\vdots & & \vdots \\
x_{N,(M,1)} & \ldots & x_{N,(M,\ell_{M})}
\end{pmatrix}
}^{m=M}
\end{bmatrix}, \label{eq:dataset X}
\end{equation}
where $x_{i,(m,\ell)}$ denotes the $(m,\ell)$'th coordinate (with angular frequency $m$ and radial frequency $\ell$) of the $i$'th data-point $x_i$.

Representing image datasets via their expansion coefficients obviously does not impose any restrictions, as any image dataset can be first expanded in basis functions of the form of $\psi_{m,\ell}$ (see Remark~\ref{remark:basis fun} below), and all subsequent analysis can be carried out in the domain of the resulting expansion coefficients $\left\{ x_{i,(m,\ell)}\right\}$. Additionally, our framework can also accommodate for images sampled from (or mapped to) a polar grid (see Remark~\ref{remark: polar grid} below).

Basis functions of the form of $\psi_{m,\ell}$, which are separable in polar coordinates into radial functions $R_{m,\ell}(r)$ multiplied by Fourier modes $e^{\imath m \theta}$, are called ``steerable''~\cite{freeman1991design,perona1995deformable}, as they allow for simple and efficient rotations. In particular, every $\psi_{m,\ell}$ can be rotated by multiplying it with a complex constant
\begin{equation}
\psi_{m,\ell}(r,\theta+\varphi) = e^{\imath m \varphi}\psi_{m,\ell}(r,\theta) \label{eq:psi rot}
\end{equation}
and thus, we can describe image rotation by modulation of the expansion coefficients, with each coefficient $x_{m,\ell}$ transformed into $x_{m,\ell} e^{\imath m \varphi}$. Consequently, we endow the ambient space $\CD$ with the rotation operation $\mathcal{R}:\CD\times [0,2\pi) \rightarrow \CD$, defined as
\begin{equation}
\mathcal{R}(x,\varphi) \triangleq x^\varphi, \quad\quad\quad  x^\varphi_{m,\ell} = x_{m,\ell} e^{\imath m \varphi}. \label{eq:extrinsic rotation}
\end{equation}
Therefore, if $x$ is the coefficients vector of the image $I$, then $x^\varphi$ is the coefficients vector of the image $I$ rotated by $\varphi$, obtained by modulating each coefficient appropriately.

Lastly, we assume that the manifold $\man$ is \textit{rotationally-invariant}, that is, it is closed under $\mathcal{R}$, such that for every $x\in\man$ and $\varphi\in [0,2\pi)$ we have that 
$\mathcal{R}(x,\varphi) \in \man$. A key observation here, is that this property enables us to generate new data-points on $\man$ by rotating existing images.

Our goal is to derive adaptive processing methods for our image dataset, allowing for filtering and de-noising, while making use of the rotation-invariance of $\man$ to provide accurate and efficient algorithms.

\begin{remark} \label{remark:basis fun}
Examples for bases of the form of $\psi_{m,\ell}$ include the 2D Prolate Spheroidal Wave Functions (PSWFs)~\cite{slepian1964prolate,landa2016approximation,shkolnisky2007prolate,lederman2017numerical}, the Fourier-Bessel functions~\cite{zhao2013fourier}, and data-adaptive steerable principal components~\cite{landa2017steerable,zhao2014fast}, all of which allow approximating image datasets provided by their samples. We note that the choice of the particular basis may depend on the application and specific model assumptions.
\end{remark}

\begin{remark} \label{remark: polar grid}
It is important to mention that our framework can also support functions/images sampled on a polar grid (for example, see~\cite{beylkin2007grids} for a Cartesian--polar mapping), and as a special case -- 1D periodic signals. That is, in place of eq.~\eqref{eq:image subspace and psi polar form}, every point $x\in\CD$ can be defined via the correspondence
\begin{equation}
I(r_\ell,\theta) = \sum_{m=-M}^M x_{m,\ell} e^{\imath m \theta},
\end{equation}
where $\ell=1,\ldots,\hat{\ell}$ enumerates over the different radii of $I$. In the case that $\hat{\ell}=1$, each point $x$ corresponds exactly to a 1D periodic signal. Then, if images/functions sampled on a polar grid are provided, $x_{m,\ell}$ can be computed efficiently by the FFT of the (equally-spaced) angular samples of $I$ for each radius. 
\end{remark}

\subsection{The steerable graph Laplacian for image-manifolds} \label{subsection:The steerable graph Laplacian for image-manifolds}

To derive a natural basis on the manifold $\man$, we employ graph Laplacian operators which encode the geometry and topology of $\man$. To this end, since the manifold $\man$ is rotationally-invariant, we propose to form a graph Laplacian over the points $\left\{x_1,\ldots,x_N\right\}$ and all of their (infinitely many) rotations. 

We start by defining an appropriate function space for constructing our operators. Consider the domain $\Gamma = \left\{1,...,N\right\} \times \mathbb{S}^1$, where $\mathbb{S}^1$ is unit circle (parametrized by an angle $\vartheta \in [0,2\pi)$), and functions $f:\Gamma\rightarrow \mathbb{C}$ of the form $f(i,\vartheta) = f_i(\vartheta)$, with $\left\{f_i\right\}_{i=1}^N \in \mathcal{L}^2(\mathbb{S}^1)$.
The space of the functions $f$ is defined as $\mathcal{H}=\mathcal{L}^2(\Gamma)$, which is a Hilbert space endowed with the inner product
\begin{equation}
\left\langle g,f\right\rangle_{\mathcal{H}} = \sum_{i=1}^N \int_0^{2\pi} g_i^*(\vartheta) f_i(\vartheta) d\vartheta
\end{equation}
for any $f,g\in\mathcal{H}$, where $(\cdot)^*$ denotes complex-conjugation. Loosely speaking, every $f\in\mathcal{H}$ can be considered as a column vector of periodic functions, namely $f = \left[ f_1(\vartheta),\ldots,f_N(\vartheta)\right]^T$, assigning a value to every index $i\in \left\{1,...,N\right\}$ and an angle $\vartheta$. 
 
In order to capture pairwise similarities between different points and rotations in our dataset, we define the steerable affinity operator $W:\mathcal{H}\rightarrow \mathcal{H}$ as
\begin{equation}
\left\{W f\right\}(i,\vartheta) = \sum_{j=1}^N \int_0^{2\pi} W_{i,j}(\vartheta,\varphi)f_j(\varphi)d\varphi, \quad \quad 
W_{i,j}(\vartheta,\varphi) = {\exp{\left\lbrace-{\left\Vert x_i^\vartheta - x_j^\varphi\right\Vert^2}{/\varepsilon}\right\rbrace}},	\label{eq:steerable affinity matrix}
\end{equation}
where $f\in\mathcal{H}$, $1\leq i,j \leq N$, $\left\{\vartheta,\varphi\right\}\in [0,2\pi)$, $\varepsilon$ is a tunable parameter, and $x_i^\vartheta$ stands for the rotation of $x_i$ by an angle $\vartheta$ (via~\eqref{eq:extrinsic rotation}). Therefore, $W$ can be considered as describing the affinity between any two rotations of any two points in our dataset. Note that since $\left\{ \psi_{m,\ell} \right\}$ (of~\eqref{eq:image subspace and psi polar form}) are orthonormal, the distance $\left\Vert x_i^\vartheta - x_j^\varphi\right\Vert^2$ agrees with the natural distance (in $\mathcal{L}^2(\mathbf{D})$) between the images corresponding to $x_i$ and $x_j$, after rotating them by $\vartheta$ and $\varphi$, respectively.

Before we proceed to define the steerable graph Laplacian over $\mathcal{H}$, we mention that we lift any complex-valued matrix $A\in\mathbb{C}^{N\times N}$ to act over $\mathcal{H}$ by
\begin{equation}
\left\{Af\right\}(i,\vartheta) = \sum_{j=1}^N A_{i,j} f_j(\vartheta),
\end{equation} 
for any $f\in\mathcal{H}$.
Then, we define the (un-normalized) steerable graph Laplacian $L:\mathcal{H}\rightarrow\mathcal{H}$ by
\begin{equation}
Lf = Df - Wf, \quad\quad D_{i,i} = \sum_{j=1}^{N} \int_0^{2\pi} W_{i,j}(0,\alpha) d\alpha, \label{eq:steerable graph Laplacian} 
\end{equation}
where $D$ is a diagonal matrix with $\left\{D_{i,i}\right\}_{i=1}^N$ on its diagonal.
If we implicitly augment our dataset to include all planar rotations of all images, then the steerable graph Laplacian can be viewed as the standard graph Laplacian (equation~\eqref{eq:standard GL}) constructed from the (infinitely-many) data points of the augmented dataset.

Similarly to the standard graph Laplacian, we show in Appendix~\ref{appendix:Quadratic form of L} that $L$ admits the quadratic form 
\begin{equation}
\left\langle f,Lf\right\rangle_{\mathcal{H}} = \frac{1}{2}\sum_{i,j=1}^N \int_0^{2\pi}\int_0^{2\pi} W_{i,j}(\vartheta,\varphi) \left\vert f_i(\vartheta) - f_j(\varphi) \right\vert^2 d\vartheta d\varphi, \label{eq:L quad form}
\end{equation}
which is analogous to the quadratic form of the standard graph Laplacian (see~\cite{belkin2003laplacian}) in the sense that it accounts for the regularity of the function $f$ over the domain $\Gamma$ w.r.t the pairwise similarities $W_{i,j}(\vartheta,\varphi)$ (measured between different data-points and rotations). In other words, the quantity $\left\langle f,Lf\right\rangle_{\mathcal{H}}$ penalizes large differences $\left\vert f_i(\vartheta) - f_j(\varphi) \right\vert$ particularly when $W_{i,j}(\vartheta,\varphi)$ is large, i.e. when the images corresponding to $x_i$ and $x_j$, rotated by $\vartheta$ and $\varphi$, respectively, are similar. Therefore, $\left\langle f,Lf\right\rangle_{\mathcal{H}}$ is expected to be small for functions $f$ which are smooth (in a certain sense) over $\Gamma$ with the geometry induced by $W$.

As we expect the operator $L$ to encode certain geometrical aspects of our dataset, as in the case of the standard graph Laplacian (see~\cite{belkin2003laplacian,coifman2006diffusion}), it is beneficial to investigate its eigen-decomposition. 
In this context, it is important to mention that a naive evaluation of $W$ (and consequently $L$) by discretizing all rotation angles is generally computationally prohibitive, and moreover, is less accurate then considering the continuum of all rotation angles.
To obtain the eigen-decomposition of $L$, we demonstrate in Appendix~\ref{appendix: LRI operators} that the steerable graph Laplacian $L$ is related to a family of operators, which we term \textit{Linear and Rotationally-Invariant} (LRI), that admit eigenfunctions with a convenient analytic form. In particular, we show that LRI operators (and an extened family of operators which includes $L$) can be diagonalized by tensor products between Fourier modes and vectors in $\mathbb{C}^N$, where the vectors can be computed efficiently by diagonalizing a certain sequence of matrices. In the case of the steerable graph Laplacian $L$, this inherently stems from the fact that $W_{i,j}(\vartheta,\varphi)$ is only a function of $\varphi-\vartheta$ (following immediately from~\eqref{eq:extrinsic rotation} and~\eqref{eq:steerable affinity matrix}), and therefore can be expanded in a Fourier series as
\begin{equation}
W_{i,j}(\vartheta,\varphi) = W_{i,j}(0,\varphi-\vartheta) = \frac{1}{2\pi}\sum_{m=-\infty}^{\infty} \hat{W}_{i,j}^{(m)} e^{ -\imath m (\varphi-\vartheta)}, \quad\quad\quad  \hat{W}^{(m)}_{i,j} = \int_0^{2\pi} W_{i,j}(0,\alpha) e^{ \imath m \alpha} d\alpha,\label{eq:w_hat_ij mat fourier}
\end{equation} 
where $1\leq i,j \leq N$. We define the matrix $\hat{W}^{(m)}$ whose $(i,j)$'th entry is $\hat{W}^{(m)}_{i,j}$, and observe from~\eqref{eq:w_hat_ij mat fourier} that the sequence of matrices $\left\{\hat{W}^{(m)}\right\}_{m=-\infty}^\infty$ provides a complete characterization of the steerable affinity operator $W$ (and consequently $L$). Therefore, the sequence of matrices $\left\{\hat{W}^{(m)}\right\}_{m=-\infty}^\infty$ also plays a key role in the evaluation of the eigen-decomposition of $L$, as detailed by the following theorem.

\begin{thm} \label{thm:eigenfunctions and eigenvalues of L}
The steerable graph Laplacian $L$ admits a sequence of non-negative eigenvalues $\left\{\lambda_{m,1},\ldots,\lambda_{m,N}\right\}_{m=-\infty}^\infty$, and a sequence of eigenfunctions $\left\{ \Phi_{m,1},\ldots,\Phi_{m,N} \right\}_{m=-\infty}^\infty$ which are orthogonal and complete in $\mathcal{H}$ and are given by
\begin{equation}
\Phi_{m,k} = v_{m,k}\cdot e^{\imath m \vartheta}, \label{eq:eigenfunctions and eigenvalues of L}
\end{equation}
where $v_{m,k}$ and $\lambda_{m,k}$ are the $k$'th eigenvector and eigenvalue, respectively, of the matrix
\begin{equation}
{S}_m = D - \hat{W}^{(m)}, \label{eq:S_m matrix def}
\end{equation}
and $D$ and $\hat{W}^{(m)}$ are given by~\eqref{eq:steerable graph Laplacian} and~\eqref{eq:w_hat_ij mat fourier}, respectively.
\end{thm}
The the proof is provided in Appendix~\ref{appendix:proof of spectral properties of L}. 

We point out that $D_{i,i} = \sum_{j=1}^N\hat{W}_{i,j}^{(0)}$, hence all quantities involving $S_m$ can be computed directly from the matrices $\hat{W}^{(m)}$, which in turn can be approximated (to arbitrary precision) by 
\begin{equation}
\hat{W}_{i,j}^{(m)} \approx \frac{2\pi}{K} \sum_{k=0}^{K-1} W_{i,j}(0,2\pi k/K) e^{ \imath 2\pi m k/K} \label{eq:W_hat FFT approx}
\end{equation}
for a sufficiently large integer $K$, and evaluated rapidly using the FFT.

Analogously to the separation of variables of the basis functions $\psi_{m,\ell}$ of~\eqref{eq:image subspace and psi polar form}, the basis functions $\Phi_{m,k}$ of~\eqref{eq:eigenfunctions and eigenvalues of L} adopt a separation into products of vectors $v_{m,k}\in\mathbb{C}^N$ and Fourier modes $e^{\imath m \vartheta}$. As such, we consider $\Phi_{m,k}$ as ``steerable'' over $\mathcal{H}$, and hence the term steerable in ``steerable graph Laplacian''. Note that the angular parts of the functions $\Phi_{m,k}$ (given by Fourier modes) correspond to different rotations of the images in the dataset, where these rotations are orbits on the manifold $\man$ passing through the original points (images) of the dataset.

\subsection{Normalized steerable graph Laplacian and the Laplace-Beltrami operator} \label{subsection:Normalized steerable graph Laplacians and the Laplace-Beltrami operator}
In the previous section, we constructed and analyzed the steerable graph Laplacian $L$, which can be considered as a generalization of the standard graph Laplacian. In particular, the steerable graph Laplacian inherits many of the favorable properties of the graph Laplacian. Based upon the construction in Section~\ref{subsection:The steerable graph Laplacian for image-manifolds}, in what follows we consider a certain normalized variant of $L$ which not only provides us with steerable basis functions adapted to our dataset, but moreover, is shown to approximate the continuous (negative-defined) Laplace Beltrami operator $\Delta_\man$. 

We start by defining the normalized steerable graph Laplacian $\tilde{L}:\mathcal{H}\rightarrow\mathcal{H}$, similarly to the normalized variant of the standard graph Laplacian (see~\cite{coifman2006diffusion}), as
\begin{equation}
\tilde{L} = D^{-1}L,	\label{eq:normalized steerable graph laplacian}
\end{equation}
where $D^{-1}$ is the inverse of the matrix $D$ from~\eqref{eq:steerable graph Laplacian}. Explicitly, we have that $\tilde{L}f = f - D^{-1} W f$ for every $f\in\mathcal{H}$.
It then turns out that the normalized steerable graph Laplacian $\tilde{L}$ converges to the negative-defined Laplace-Beltrami operator $\Delta_\man$~\cite{rosenberg1997laplacian} when $\varepsilon\rightarrow 0$ and $N\rightarrow\infty$, while improving on the convergence rate of the standard (normalized) graph Laplacian (equation~\eqref{eq:standard GL convergence}), as reported by the next theorem.
\begin{thm} \label{thm:steerable graph Laplacian convergence}
Suppose that $\sum_{m\neq 0}\sum_{\ell=1}^{\ell_m} \left\vert x_{m,\ell}\right\vert^2 >0 $ for all $x\in\man$ (up to a set of measure zero), and let $\left\{ x_1,\ldots,x_N\right\}\in \man$ be i.i.d with probability distribution $p(x)=1/\operatorname{Vol}\left\lbrace\man\right\rbrace$, i.e. uniform sampling distribution. If $f:\man\rightarrow \mathbb{R}$ is a smooth function, and we define $g\in\mathcal{H}$ s.t. $g(i,\vartheta) = f(x_i^\vartheta)$ (where $x_i^\vartheta$ is given by~\eqref{eq:extrinsic rotation}), then with high probability we have that
\begin{equation} \label{eq:steerable graph Laplacian convegence}
\frac{4}{\varepsilon}\left\{\tilde{L} g\right\}(i,\vartheta) = \Delta_\man f(x_i^\vartheta) + O(\frac{1}{N^{1/2}\varepsilon^{1/2+(d-1)/4}}) + O(\varepsilon).
\end{equation}
\end{thm}
The proof is provided in Appendix~\ref{appendix:proof of steerable graph Laplacian convergence}. 
Comparing~\eqref{eq:steerable graph Laplacian convegence} with~\eqref{eq:standard GL convergence}, it is evident that both graph Laplacians converge to $\Delta_\man$ with the same bias error term of $O(\varepsilon)$. However, the steerable graph Laplacian admits a smaller variance error term (second term from the right in~\eqref{eq:steerable graph Laplacian convegence}), which depends on $d-1$ instead of $d$. Note that the improvement in the convergence rate (from $d$ in~\eqref{eq:standard GL convergence} to $d-1$ in~\eqref{eq:steerable graph Laplacian convegence}) is significant and in no way depends on the dimension of the ambient space $\mathcal{D}$. The intuition behind this improvement is that the steerable graph Laplacian takes all rotations of all images into consideration, and so it analytically accounts for one of the intrinsic dimensions of $\man$, that is, the dimension corresponding to the rotation $\mathcal{R}$ (see~\eqref{eq:extrinsic rotation}).
A numerical example demonstrating the improved convergence rate due to Theorem~\ref{thm:steerable graph Laplacian convergence} can be found in Section~\ref{subsection:Numerical example: Steerable graph Laplacian convergence rate}.

\begin{remark}
The condition $\sum_{m\neq 0} \sum_{\ell=1}^{\ell_m} \left\vert x_{m,\ell}\right\vert ^2>0$ in Theorem~\ref{thm:steerable graph Laplacian convergence} essentially requires that the images associated with the points of $\man$ are not radially-symmetric (i.e. have a non-constant angular part). This is because the coordinates $x_{m,\ell}$ of $x$ corresponding to the angular index $m=0$ contribute only to the radial part of the image (see equation~\eqref{eq:image subspace and psi polar form}). Of course, if the images are all radially-symmetric, then the steerable graph Laplacian would not provide any improvement over the convergence rate of the standard graph Laplacian.
\end{remark}

In the case that the sampling density $p(x)$ in Theorem~\ref{thm:steerable graph Laplacian convergence} is not uniform, we argue in Appendix~\ref{appendix: Non uniform sampling} that instead of the Laplace-Beltrmi operator $\Delta_\man$, the steerable graph Laplacian $\tilde{L}$ approximates the weighted Laplacian (Fokker-Planck operator) $\tilde{\Delta}_\man$ given by
\begin{equation}
\tilde{\Delta}_\man f(x) = \Delta_\man f(x) - 2\frac{\left\langle \nabla_\man f(x),\nabla_\man \tilde{p}(x) \right\rangle}{\tilde{p}(x)},
\end{equation}
where $f:\man\rightarrow \mathbb{R}$ is a smooth function, and $\tilde{p}$ is the rotationally-invariant density
\begin{equation}
\tilde{p}(x) = \frac{1}{2\pi} \int_0^{2\pi} p(x^\varphi) d\varphi.
\end{equation} 
Additionally, we explain in Appendix~\ref{appendix: Non uniform sampling} how to normalize the sampling density such that the resulting operator still converges to the Laplace-Beltrami operator $\Delta_\man$ (analogously to the density-invariant normalization in~\cite{coifman2006diffusion}). We include this procedure as an optional step in the algorithms` summery in Section~\ref{section:Algorithm summary and computational cost}.

Next, we evaluate the eigenfunctions and eigenvalues of the normalized steerable graph Laplacian $\tilde{L}$ of~\eqref{eq:normalized steerable graph laplacian}, where analogously to Theorem~\ref{thm:eigenfunctions and eigenvalues of L}, the next theorem relates the eigenfunctions and eigenvalues of $\tilde{L}$ to the matrices $\hat{W}^{(m)}$ of~\eqref{eq:w_hat_ij mat fourier}. 

\begin{thm} \label{thm:eigenfunctions and eigenvalues of L_tilde}
The normalized steerable graph Laplacian $\tilde{L}$ admits a sequence of non-negative eigenvalues $\left\{\tilde{\lambda}_{m,1},\ldots,\tilde{\lambda}_{m,N}\right\}_{m=-\infty}^\infty$, and a sequence of eigenfunctions $\left\{ \tilde{\Phi}_{m,1},\ldots,\tilde{\Phi}_{m,N} \right\}_{m=-\infty}^\infty$ which are complete in $\mathcal{H}$ and are given by
\begin{equation}
\tilde{\Phi}_{m,k} = \tilde{v}_{m,k} \cdot e^{\imath m \vartheta},  \label{eq:eigenfunctions and eigenvalues of L_tilde}
\end{equation}
where $\tilde{v}_{m,k}$ and $\tilde{\lambda}_{m,k}$ are the $k$'th eigenvector and eigenvalue, respectively, of the matrix
\begin{equation}
\tilde{S}_m = I - D^{-1}\hat{W}^{(m)}, \label{eq:S_tilde_m matrix def}
\end{equation}
$I$ is the $N\times N$ identity matrix, and $D$ and $\hat{W}^{(m)}$ are given by~\eqref{eq:steerable graph Laplacian} and~\eqref{eq:w_hat_ij mat fourier}, respectively.
\end{thm}
The proof is provided in Appendix~\ref{appendix:proof of spectral properties of L_tilde}.

Let us denote the basis $\left\{\tilde{\Phi}_{m,k}\right\}$ of~\eqref{eq:eigenfunctions and eigenvalues of L_tilde} by $\tilde{\Phi}$. Due to the convergence of $\tilde{L}$ to the Laplace-Beltrami operator $\Delta_\man$, we consider $\tilde{\Phi}$ as a basis adapted to our dataset through the geometry and topology of $\man$, and hence a favorable basis for expanding and filtering our dataset. Since $\left\{\tilde{\Phi}_{m,k}\right\}$ are also steerable, we shall refer to them (with a slight abuse of notation) as \textit{steerable manifold harmonics}. We illustrate one of these eigenfunctions in the numerical example of Section~\ref{subsection:Numerical example: Steerable graph Laplacian convergence rate} (where the manifold is the unit sphere).

\subsection{Toy example} \label{subsection:Numerical example: Steerable graph Laplacian convergence rate}
At this point, we wish to demonstrate our setting as well as the improved convergence rate of the steerable graph Laplacian by the following example.
Consider images of the form
\begin{equation}
I(r,\theta) = x_{0,1}R_{0,1}(r) + x_{1,1}R_{1,1}(r)e^{\imath \theta},	\label{eq:toy exm image def}
\end{equation}
which is a special case of~\eqref{eq:image subspace and psi polar form}, where $R_{0,1}, R_{1,1}$ are arbitrary radial functions, and $M=1,\;\ell_{-1}=0,\;\ell_0=1,\;\ell_1=1$.
Additionally, we take the unit sphere $\mathbb{S}^2$ ($d=2$) in $\mathbb{R}^3$, and embed it in $\mathbb{C}^2$ by mapping every point $p = [p_x,p_y,p_z]\in\mathbb{S}^2$ ($p_x,p_y,p_z$ are the $x,y,z$ coordinates) to the point $x=[x_{0,1},x_{1,1}]\in\man$ via
\begin{equation}
x_{0,1} = p_z, \quad x_{1,1} = p_x + \imath p_y.
\end{equation} 
Note that the rotation operation $\mathcal{R}$ of~\eqref{eq:extrinsic rotation} in this case is
\begin{equation}
\mathcal{R}(x,\varphi) = 
\begin{bmatrix}
1 & 0 \\
0 & e^{\imath \varphi}
\end{bmatrix}
\begin{bmatrix}
x_{0,1} \\
x_{1,1}
\end{bmatrix},
\end{equation}
which is equivalent to rotating the point $p\in\mathbb{R}^3$ (corresponding to $x$) in the $xy$-plane as
\begin{equation}
\begin{bmatrix}
\cos(\varphi) & \sin(\varphi) & 0 \\
-\sin(\varphi) & \cos(\varphi) & 0 \\
0 & 0 & 1
\end{bmatrix}
\begin{bmatrix}
p_x \\
p_y \\
p_z 
\end{bmatrix}.
\end{equation}
Hence, all rotations of all images sampled from the sphere remain on the sphere, and therefore $\man$ is rotationally-invariant (as defined in Section~\ref{subsection:Problem se-tup}).

In order to demonstrate numerically the convergence rate of the (normalized) steerable graph Laplacian to the Laplace-Beltrami operator (as asserted by Theorem~\ref{thm:steerable graph Laplacian convergence}), we chose a test function $f:\man\rightarrow\mathbb{R}$
\begin{equation}
f(x) = \operatorname{Re}\left\{x_{1,1}\right\} + x_{0,1},
\end{equation}
and a testing point $x_0=\left[0,1\right]$ (corresponding to $p=\left[1,0,0\right]$ on $\mathbb{S}^2$), for which $\Delta_\man{f} (x_0) = -2$ (see example in~\cite{singer2006graph}).
We then uniformly sampled $N=2,000$ points $\left\{x_1,\ldots,x_N\right\}$ from $\man$ and approximated $\Delta_\man{f}$ by applying the steerable graph Laplacian $\tilde{L}$. Specifically, $\Delta_\man{f}(x_0)$ was approximated from~\eqref{eq:steerable graph Laplacian convegence} and~\eqref{eq:normalized steerable graph laplacian} by defining $g(i,\vartheta)=f(x_i^\vartheta)$ for $i=0,1,\ldots,N$
and computing
\begin{align}
\frac{4}{\varepsilon}\left\{ \tilde{L}g\right\}(0,0) &= \frac{4}{\varepsilon} \left[ f(x_0) - \sum_{j=0}^N \int_0^{2\pi} D^{-1}_{0,j} W_{0,j}(0,\varphi)f(x_j^\varphi)d\varphi \right]
= \frac{4}{\varepsilon} \left[f(x_0) - \frac{\sum_{j=0}^N \int_0^{2\pi} {W}_{0,j}(0,\varphi)f(x_j^\varphi)d\varphi}{\sum_{j=0}^N \int_0^{2\pi} {W}_{0,j}(0,\varphi)d\varphi}\right] \nonumber \\
&\approx \frac{4}{\varepsilon} \left[f(x_0) - \frac{\sum_{j=0}^N \sum_{k=0}^{K-1} {W}_{0,j}(0,{2\pi k}/{K})f(x_j^{2\pi k/K})}{\sum_{j=0}^N  \sum_{k=0}^{K-1} {W}_{0,j}(0,{2\pi k}/{K})}\right],	\label{eq:toy exm laplacian approx}
\end{align}
where ${W}$ is given by~\eqref{eq:steerable affinity matrix}, $D$ is given in~\eqref{eq:steerable graph Laplacian}, and we replaced integration with summation using a sufficiently large integer $K$. 
Note that $f(x_0) = 1$, and by~\eqref{eq:extrinsic rotation} we have that
\begin{equation}
f(x_j^{2\pi k/K}) = 
\operatorname{Re}{\left\{ x_{j,(1,1)} e^{\imath 2\pi k/K} \right\}}
+ x_{j,(0,1)},
\end{equation}
where $x_{j,(m,\ell)}$ is the $(m,\ell)$'th coordinate of the $j$'th point.
Figure~\ref{fig:steerable laplacian toy exm} depicts the errors of estimating $\Delta_\man{f} (x_0)$ using the steerable graph Laplacian (equation~\eqref{eq:toy exm laplacian approx}) versus the standard graph Laplacian (equations~\eqref{eq:standard GL} and~\eqref{eq:standard GL convergence}), for $K=256$ and different values of $\varepsilon$. The slope of the log-error in the variance-dominated region (obtained by a linear curve fit and averaged over $1,000$ experiments) is $-0.97$ for the standard graph Laplacian, and $-0.74$ for the steerable graph Laplacian, agreeing with equation~\eqref{eq:standard GL convergence} and Theorem~\ref{thm:steerable graph Laplacian convergence}, which predict slopes of $-1$ and  $-0.75$, respectively, when substituting $d=2$. Moreover, the errors due to the steerable and standard graph Laplacians coincide in the region where the errors are dominated by the bias error term, also in agreement with Theorem~\ref{thm:steerable graph Laplacian convergence}. 

\begin{figure}
  \centering  	
    \includegraphics[width=0.5\textwidth]{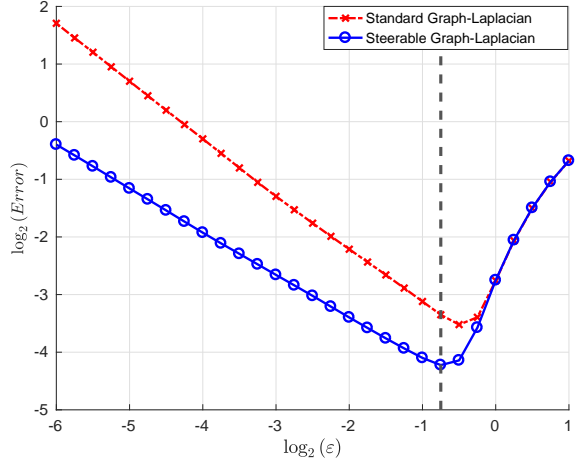}
	\caption[LaplacianError] 
	{Errors in approximating $\Delta_\man f(x_0)$ by the standard graph Laplacian (equations~\eqref{eq:standard GL} and~\eqref{eq:standard GL convergence}) and by the steerable graph Laplacian (equation~\eqref{eq:toy exm laplacian approx}) as a function of $\varepsilon$ (in $\log$ scale). The region to the left of the dashed vertical line is dominated by the variance error term, whereas the region to the right of the dashed vertical line is dominated by the bias error term.
	}  \label{fig:steerable laplacian toy exm}
\end{figure}

Additionally, we computed the eigenvalues of $\tilde{L}$ as described in Section~\ref{subsection:Normalized steerable graph Laplacians and the Laplace-Beltrami operator}, and compared them with the eigenvalues of the standard (normalized) graph Laplacian. The results can be seen in Figure~\ref{fig:sphere eigenvalues}. It is evident that the eigenvalues in both cases agree with the well-known multiplicities of the spherical harmonics (the eigenfunctions of the Laplacian on the unit sphere). However, is clear that the eigenvalues of $\tilde{L}$ admit smaller fluctuations compared to the eigenvalues of the standard (normalized) graph Laplacian, owing to the improved convergence rate of $\tilde{L}$ to the Laplace-Beltrami operator. 

Lastly, in Figure~\ref{fig:sphere eigenfunctions} we illustrate a single eigenfunction of the steerable graph Laplacian (computed via Theorem~\ref{thm:eigenfunctions and eigenvalues of L_tilde}), corresponding to the indices $m=3,k=4$, where we used $N=512$ points and $\varepsilon=1$.  The figure highlights the difference between the vector $\tilde{v}_{m,k}$ in~\eqref{eq:eigenfunctions and eigenvalues of L_tilde} and the eigenfunction $\tilde{\Phi}_{m,k}$ itself. While the former is analogous to an eigenvector of the standard graph Laplacian (in the sense that it is defined only over the original data points), the latter extends its domain of definition by additionally assigning values to all rotations of the original data points (images). Note that the behavior of the eigenfunctions $\tilde{\Phi}_{m,k}$ over the orbits of the images' rotations is given by Fourier modes, which is in agreement with the explicit formula for the spherical harmonics (given by Fourier modes in the azimuthal direction).

\begin{figure}
  \centering
  	\subfloat  	
  	{
    \includegraphics[width=0.5\textwidth]{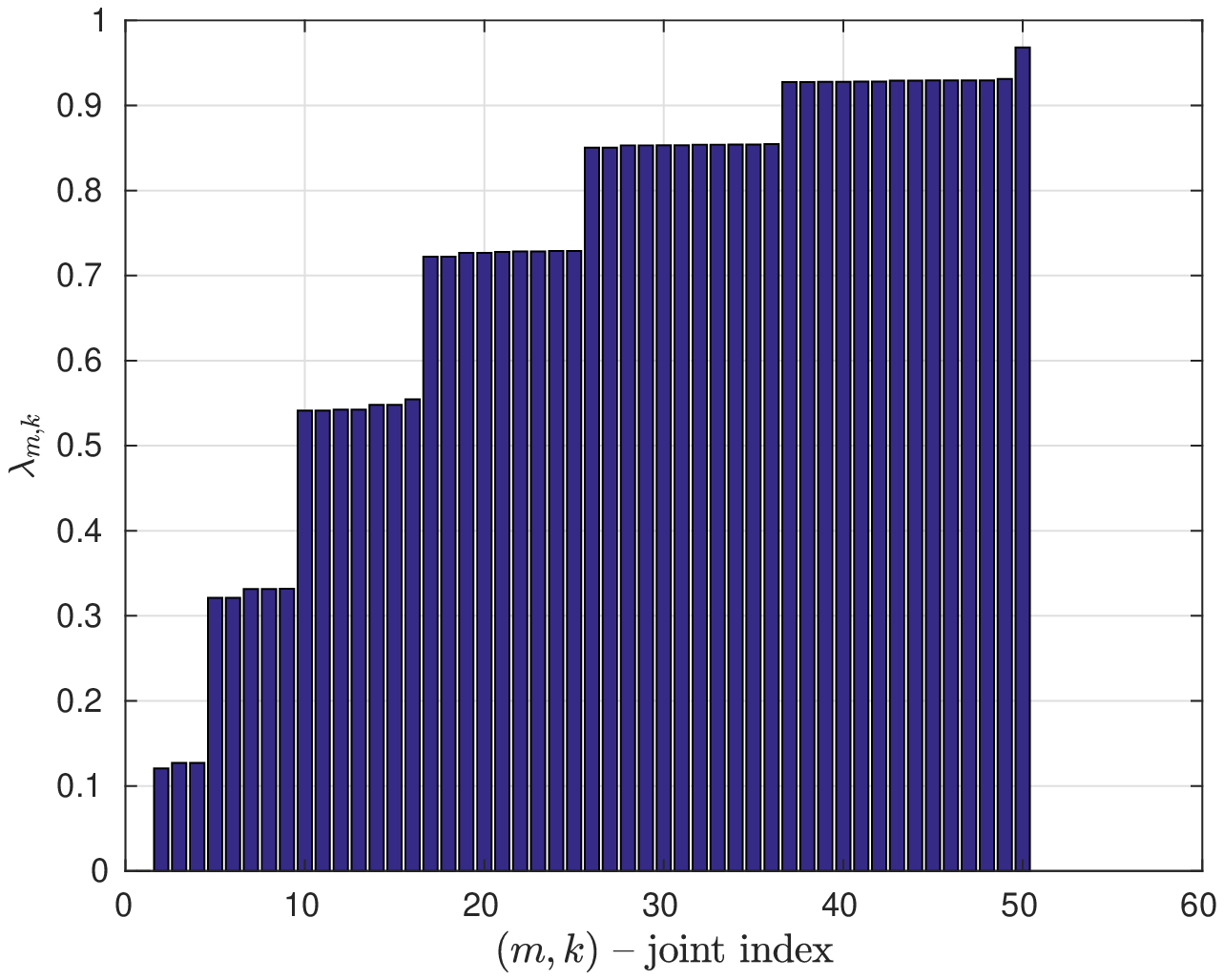}
    }
    \subfloat    
    { 
    \includegraphics[width=0.5\textwidth]{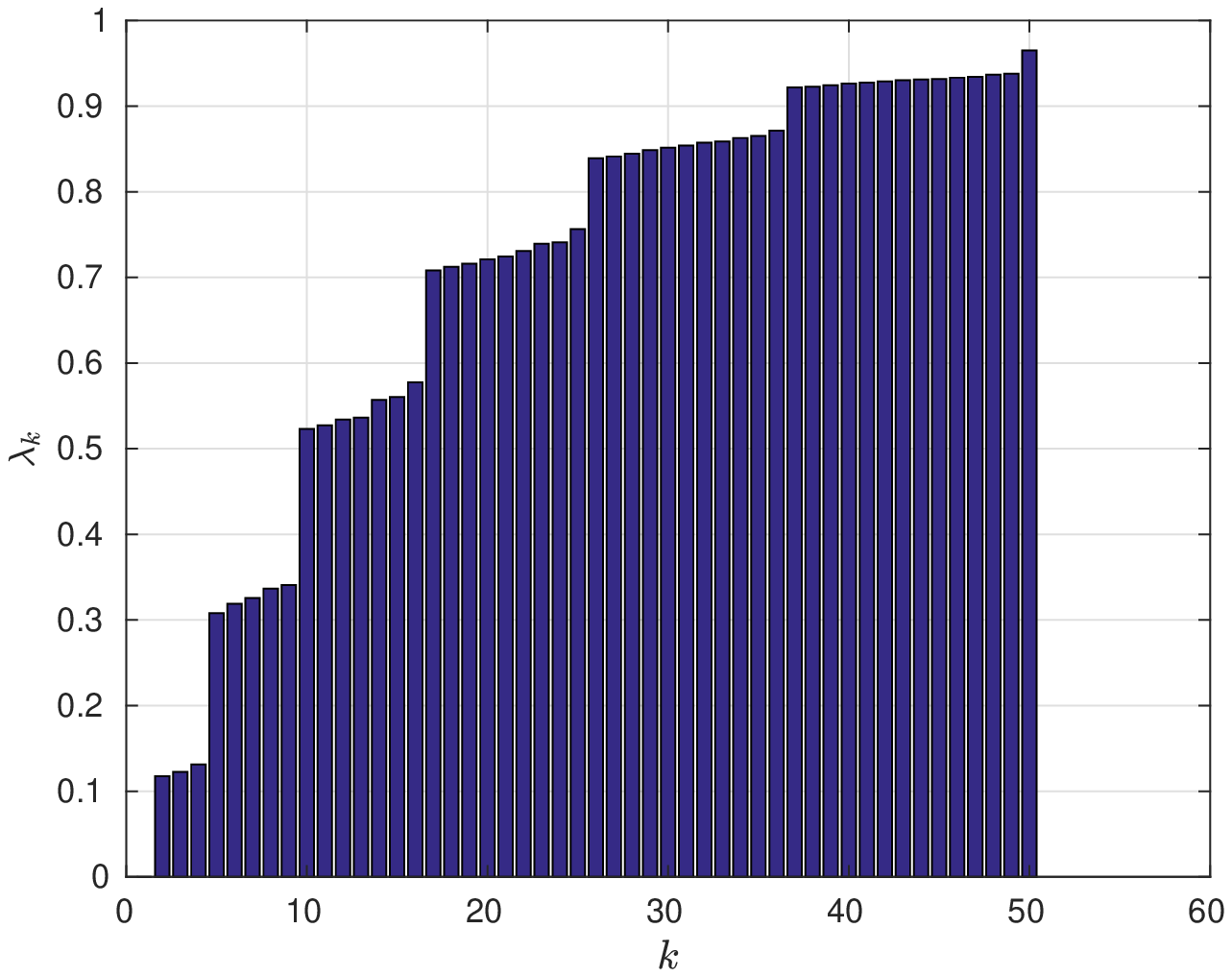}
    }
	\caption[sphere eigenvalues]
	{Eigenvalues of the steerable (left) and standard (right) normalized graph Laplacians, for $\varepsilon=1/4$ and $2,000$ data points sampled uniformly from the sphere. For the steerable graph Laplacian, the eigenvalues were sorted in ascending order and enumerated over $(m,k)$ using a single joint index.} \label{fig:sphere eigenvalues}
\end{figure}

\begin{figure}
  \centering
  	\subfloat[$\tilde{v}_{m,k}$]  	
  	{
    \includegraphics[width=0.5\textwidth]{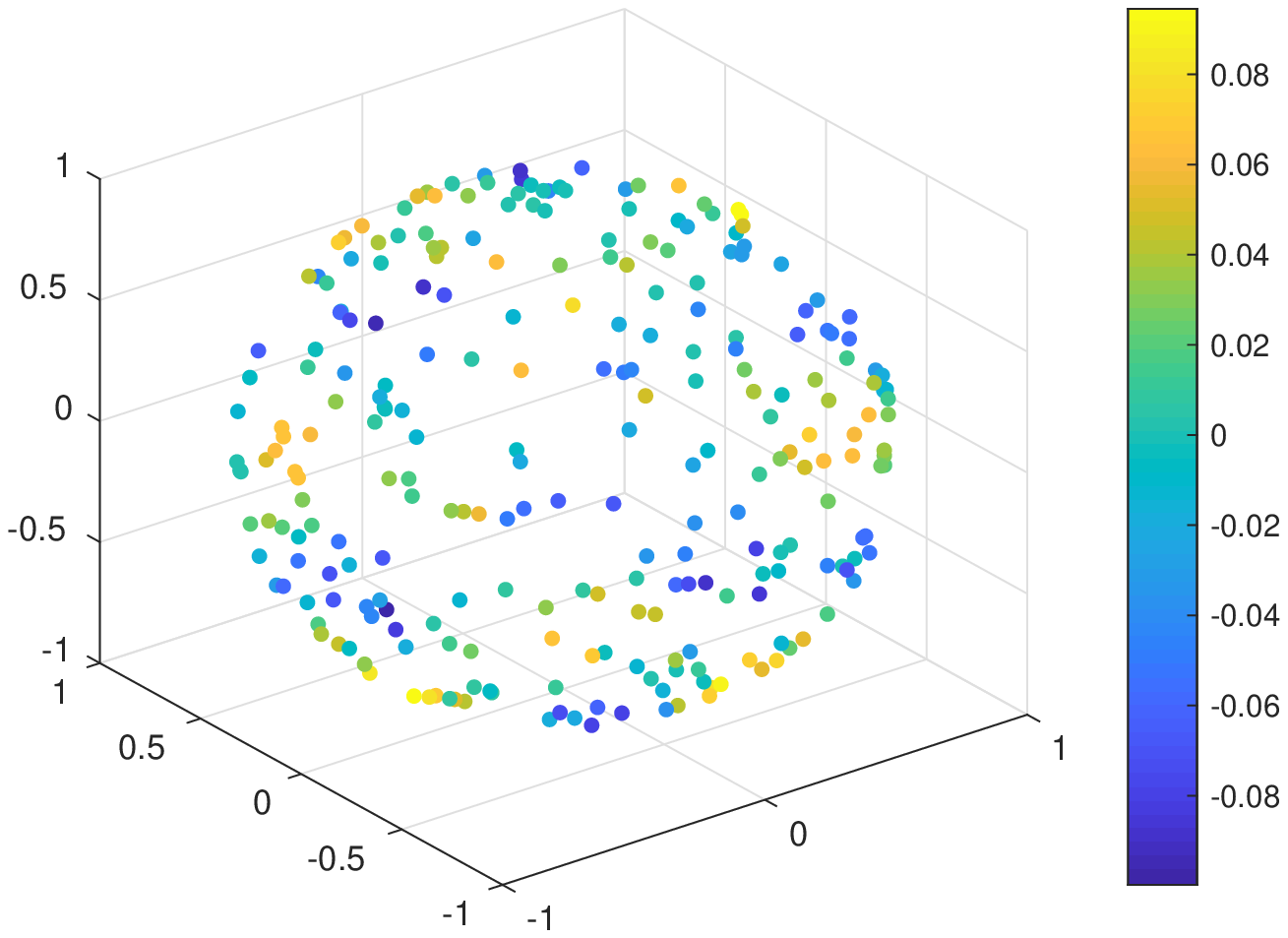}
    }
    \subfloat[$\tilde{\Phi}_{m,k} = \tilde{v}_{m,k}\cdot e^{\imath m \vartheta}$]    
    { 
    \includegraphics[width=0.5\textwidth]{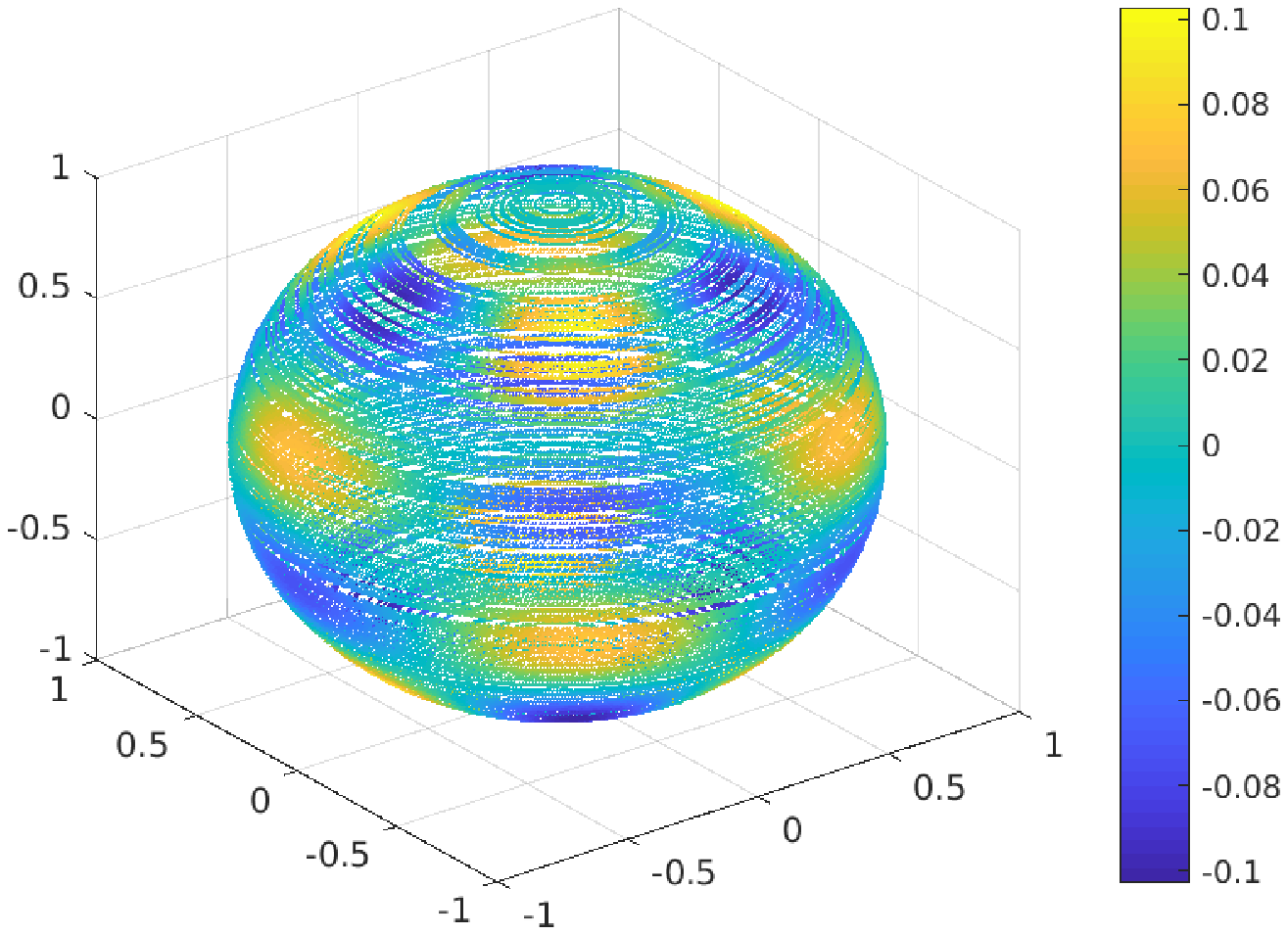}
    }
	\caption[sphere eigenvalues]
	{An eigenfunction (real part only) of the steerable graph Laplacian $\tilde{L}$, corresponding to $m=3$ and $k=4$ (see Theorem~\ref{thm:eigenfunctions and eigenvalues of L_tilde}), using $N=512$ points sampled uniformly from the sphere, and $\varepsilon=1$. On the left we show the values of the eigenfunction only for the original $512$ data points (given by the vector $\tilde{v}_{m,k}$ of~\eqref{eq:eigenfunctions and eigenvalues of L_tilde}), and on the right we show all values of the eigenfunction $\tilde{\Phi}_{m,k}$, including the angular part (given explicitly by Fourier modes $e^{\imath m \vartheta}$) assigning values to all rotations of the original data points (visible as orbits in the shape of horizontal rings covering the sphere).} \label{fig:sphere eigenfunctions}
\end{figure}

\subsection{Filtering image datasets by the steerable manifold harmonics} \label{subsection:Expanding image datasets by the steerable manifold harmonics}
Next, we propose to expand our dataset of images and all of their rotations by a carefully-chosen subset of the steerable manifold harmonics (the eigenfunctions of the steerable graph Laplacian $\tilde{L}$, see Theorem~\ref{thm:eigenfunctions and eigenvalues of L_tilde}).

Consider the function $F_{m,\ell}\in\mathcal{H}$ given by
\begin{equation}
F_{m,\ell}(i,\varphi)=x_{i,(m,\ell)}^\varphi,
\end{equation}
where $x_{i,(m,\ell)}^\varphi$ stands for the $(m,\ell)$'th coordinate of the $i$'th data-point rotated by $\varphi$ (via~\eqref{eq:extrinsic rotation}). In essence, the function $F_{m,\ell}$ describes the $(m,\ell)$'th coordinate of all points in the dataset and all of their rotations. As $F_{m,\ell} \in \mathcal{H}$, it can be expanded in the basis $\tilde{\Phi}$, and we can write
\begin{equation}
F_{m,\ell} = \sum_{m^{'}=-\infty}^\infty \sum_{k=1}^N  A_{({m^{'},k}),({m,\ell})} \tilde{\Phi}_{m^{'},k} \label{eq:X_ml full expansion}
\end{equation}
for all $(m,\ell)$ pairs, where $A_{({m^{'},k}),({m,\ell})}$ are some associated expansion coefficients. 
We propose to ``filter'' the functions $F_{m,\ell}$ for each pair $(m,\ell)$ by considering a truncated expansion of the form of~\eqref{eq:X_ml full expansion}, with expansion coefficients obtained by solving
\begin{equation}
\min_A \left\Vert F_{m,\ell} - \sum_{m^{'}=-M^{'}}^{M^{'}}\sum_{k=1}^{k_{m^{'}}} A_{({m^{'},k}),({m,\ell})}\tilde{\Phi}_{m^{'},k} \right\Vert_{\mathcal{H}}^2, \label{eq:manifold regression}
\end{equation}
where $A$ is a matrix of expansion coefficients with rows indexed by $({m^{'},k})$ and columns indexed by $({m,\ell})$.

As for the numbers of chosen basis functions $\left\{k_{m^{'}}\right\}$ and $M^{'}$, we propose the following natural truncation rule based on a cut-off frequency $\lambda_c\in \mathbb{R}_+$:
\begin{equation}
k_{m^{'}} = \max\left\{k:\tilde{\lambda}_{m^{'},k} < \lambda_c\right\}, \label{eq:truncation rule}
\end{equation}
where $\left\{ \tilde{\lambda}_{m,k} \right\}$ are the eigenvalues (sorted in non-decreasing order w.r.t $k$) of the normalized steerable graph Laplacian $\tilde{L}$. Then, $M^{'}$ is simply the largest $\left\vert m^{'} \right\vert$ s.t. $k_{m^{'}}>0$. Fundamentally, this truncation rule can be viewed as the analogue of the classical truncation of Fourier expansions.
Figure~\ref{fig:Lambda indices configuration} illustrates a typical configuration of index-pairs $(m^{'},k)$ resulting from the truncation rule of~\eqref{eq:truncation rule}.

We motivate the above-mentioned approach (series expansion and truncation rule) as follows. It is well known that for smooth and compact manifolds the Laplace-Beltrami operator $\Delta_\man$ admits a sequence of eigenvalues $\left\lbrace\mu_k\right\rbrace_{k=0}^\infty$ and eigenfunctions $\left\lbrace\phi_k\right\rbrace_{k=0}^\infty$, which are orthogonal and complete in the class of square-integrable functions on $\mathcal{M}$, denoted by $\mathcal{L}^2(\mathcal{M})$. 
Therefore, every function $f\in\mathcal{L}^2(\man)$ can be expanded as
\begin{equation}
f(x) = \sum_{k=1}^\infty a_k \phi_k(x), \quad\quad a_k = \int_\man f(x) {\phi_k^*}(x) dx.
\end{equation}
In this context, it is possible to consider the coordinates of $\man$ in the ambient space, i.e. $x_{m,\ell}$ for every $x\in\man$, as smooth functions over $\man$, which can be approximated by truncating the above-mentioned expansion. 
In particular, we provide the following proposition, which bounds the error in approximating the coordinates of $\man$ using a truncated series of manifold harmonics.

\begin{prop} \label{prop:man approx bias error}
Let $\left\{\phi_k\right\}_{k=1}^\infty$ and $\left\{ \mu_k \right\}_{k=1}^{\infty}$ be the eigenfunctions and eigenvalues (sorted in non-decreasing order), respectively, of the negative-defined Laplace-Beltrami operator $\Delta_\man$. Then, we have that
\begin{equation}
\frac{1}{\operatorname{Vol}{\left\lbrace\man\right\rbrace}} \int_{\man} \left\Vert \sum_{k=1}^{n} c_k \phi_k(x) - x \right\Vert^2_2 dx \leq \frac{d}{\mu_{n+1}}, \quad\quad c_k = \int_{\man} x \phi_k^* (x) dx,	\label{eq:bias error term}
\end{equation}
where $\operatorname{Vol}{\left\lbrace\man\right\rbrace}$ is the volume of $\man$, $d$ is the intrinsic dimension of $\man$, and $c_k\in\mathbb{C}^\mathcal{D}$ are the expansion coefficients of $x\in\CD$ (i.e. of every coordinate function of the embedded manifold) w.r.t $\phi_k$.
\end{prop}
\begin{proof}
The proof follows immediately from combining equation~($3.1$) in~\cite{aflalo2015optimality} and Proposition~$3.1$ in~\cite{osher2017low}.
\end{proof}
It is important to note that by the properties of $\Delta_\man$ we have that $\mu_n{\rightarrow} \infty$ when $n\rightarrow\infty$~\cite{rosenberg1997laplacian}, and therefore we can get an arbitrarily small approximation error for the coordinates of $\man$ using a sufficiently large number of manifold harmonics. As we have shown in Section~\eqref{subsection:Normalized steerable graph Laplacians and the Laplace-Beltrami operator} that $\tilde{L}$ approximates the Laplace-Beltrami operator $\Delta_\man$, we follow the common practice and use the eigenfunctions and eigenvalues of $\tilde{L}$, i.e. $\left\{\tilde{\Phi}_{m,k}\right\}$ and $\left\{\tilde{\lambda}_{m,k}\right\}$, as discrete proxies for $\left\{\phi_k\right\}$ and $\left\{\mu_k\right\}$ in~\eqref{eq:bias error term}.

Next, we proceed to derive a simple and efficient solution to problem~\eqref{eq:manifold regression}. By our construction of the Hilbert space $\mathcal{H}$, one can write~\eqref{eq:manifold regression} explicitly as 
\begin{equation}
\min_A \left\{ \sum_{i=1}^N \int_0^{2\pi} \left\vert x_{i,(m,\ell)}^\varphi - \sum_{m^{'}=-M^{'}}^{M^{'}}\sum_{k=1}^{k_{m^{'}}} A_{({m^{'},k}),({m,\ell})} \tilde{\Phi}_{m^{'},k}(i,\varphi) \right\vert^2 d\varphi\right\}, \label{eq:manifold regression explicit}
\end{equation}
which is interpreted as performing regression over the entire dataset of images and all of their planar rotations using the functions $\tilde{\Phi}_{m,k}$ restricted to $k\in\left\{1,\ldots,k_{m^{'}}\right\}, m\in\left\{-M^{'},\ldots,M^{'}\right\}$. 
Recall that by~\eqref{eq:extrinsic rotation}, we have that
\begin{equation}
x_{i,(m,\ell)}^\varphi = x_{i,(m,\ell)} e^{\imath m \varphi}, \label{eq:x_i steering}
\end{equation}
where $x_{i,(m,\ell)}$ stands for the $(m,\ell)$'th coordinate of the $i$'th data-point. It turns out that~\eqref{eq:manifold regression explicit} can be significantly simplified by substituting~\eqref{eq:x_i steering} into~\eqref{eq:manifold regression explicit} together with the steerable form of $\tilde{\Phi}_{m,\ell}$ (i.e.~\eqref{eq:eigenfunctions and eigenvalues of L_tilde}), while making use of the orthogonality of the Fourier modes $\left\{e^{\imath m \varphi}\right\}_{m=-\infty}^\infty$ over $[0,2\pi)$. It then immediately follows that the matrix of coefficients $A$ in the solution of~\eqref{eq:manifold regression explicit} is block-diagonal, where the blocks can be obtained by solving ordinary least-squares problems. In particular, we have that
\begin{equation}
A_{({m^{'},k}),({m,\ell})} = 
\begin{dcases}
B^{(m)}_{k,\ell}, & m=m^{'}, \\
0, &  m\neq m^{'},
\end{dcases}
\end{equation}
where $B^{(m)}$ is the $m$'th block on the diagonal of $A$, obtained by solving the least-squares system
\begin{equation}
\min_{B^{(m)}} \left\Vert
X^{(m)} - \widetilde{V}^{(m)}B^{(m)}\right\Vert^2_F, \label{eq:B^m least squares}
\end{equation}
where $\left\Vert \cdot\right\Vert_F$ stands for the Frobenius norm,
and $X^{(m)}$ and $\widetilde{V}^{(m)}$ are given by
\begin{equation}
X^{(m)} = 
\begin{pmatrix}
x_{1,(m,1)} & \ldots & x_{1,(m,\ell_{m})} \\
\vdots & \ddots & \vdots \\
x_{N,(m,1)} & \ldots & x_{N,(m,\ell_{m})}
\end{pmatrix},
\quad\quad\quad
\widetilde{V}^{(m)} = 
\begin{pmatrix}
| &  & | \\
\tilde{v}_{m,1} & \cdots & \tilde{v}_{m,k_m} \\
| &  & |
\end{pmatrix}, \label{eq:X^m and V_tilde^m def}
\end{equation}
with $\tilde{v}_{m,k}$ given by~\eqref{eq:eigenfunctions and eigenvalues of L_tilde} and~\eqref{eq:S_tilde_m matrix def}.
We mention that $k_{m^{'}}$ changes with the angular index $m^{'}$, and in particular, is typically smaller for higher angular frequencies (larger $|m|$). Therefore, the size of the blocks $B^{(m)}$ reduces with $|m|$, as illustrated by Figure~\ref{fig:A block-diagonal typical}. 
Once the coefficients matrices $\left\{B^{(m)}\right\}$ were obtained by solving~\eqref{eq:B^m least squares}, we define 
\begin{equation}
\hat{X}^{(m)} \triangleq \widetilde{V}^{(m)}B^{(m)} \label{eq:x_hat filtering}
\end{equation}
as the filtered dataset corresponding to the angular index $m$. 

\begin{figure}[!htb]
    \centering    
    \begin{minipage}[t]{0.48\textwidth}
        \centering
        \includegraphics[width=0.6\linewidth, height=0.4\textheight]{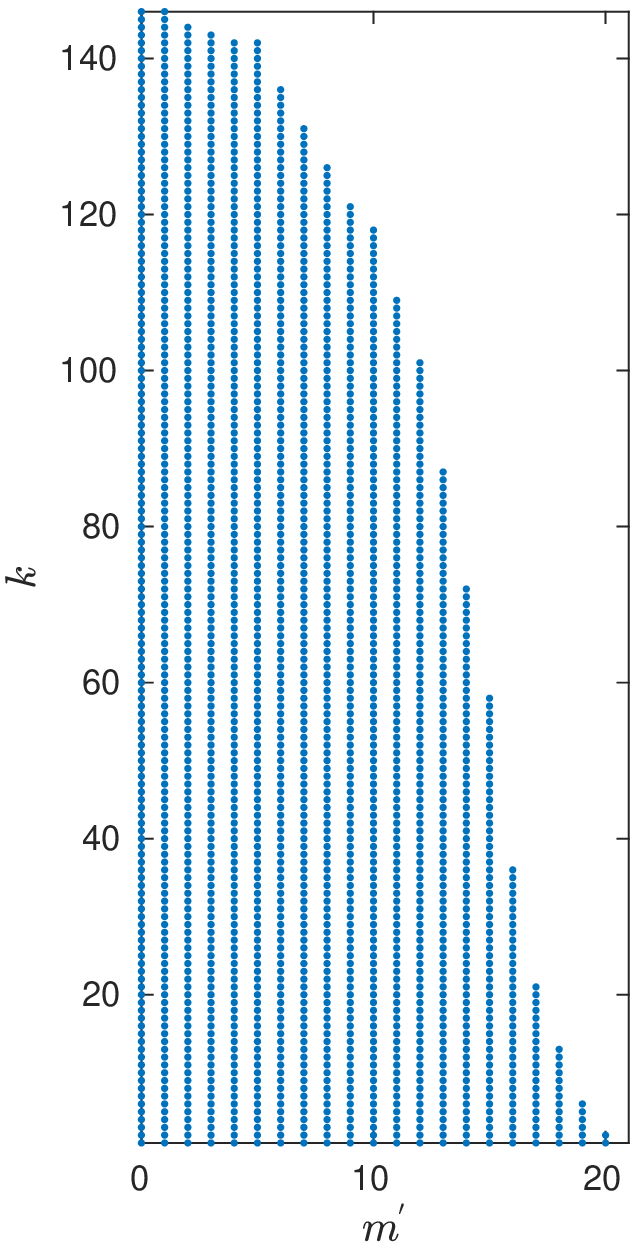}
        \caption{A typical configuration of index pairs $(m^{'},k)$ obtained by the truncation rule of~\eqref{eq:truncation rule}, for non-negative angular indices~$m$.}
        \label{fig:Lambda indices configuration}
    \end{minipage}
    \hfill
    \begin{minipage}[t]{.48\textwidth}
        \centering
        \includegraphics[width=0.6\linewidth, height=0.4\textheight]{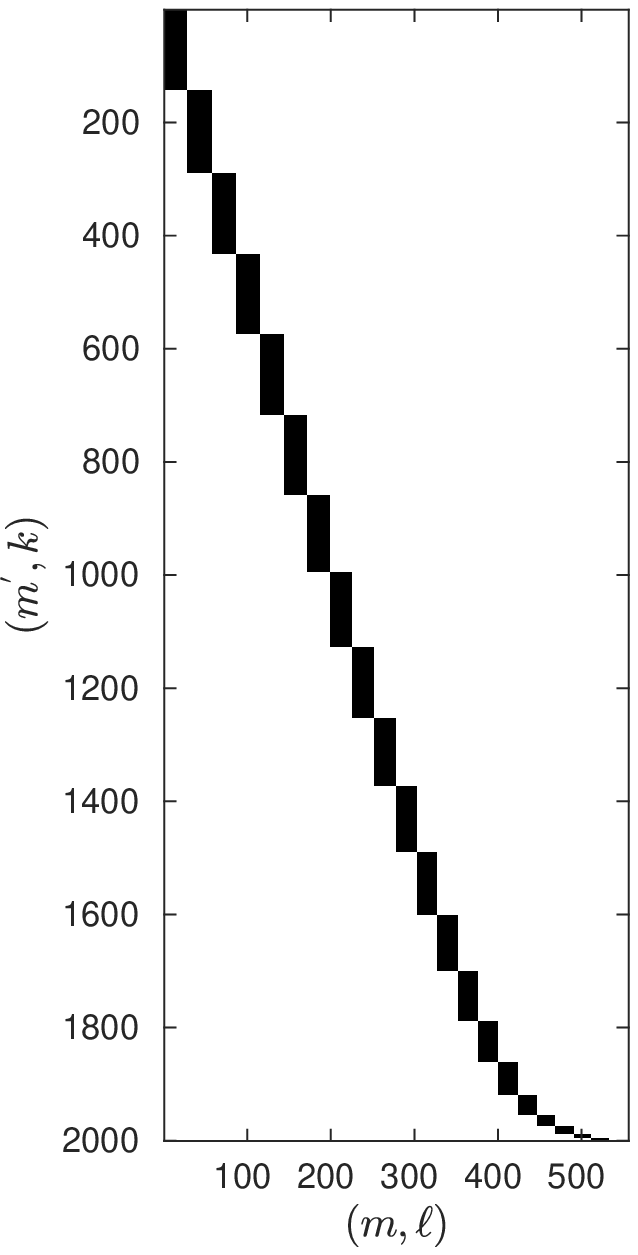}
        \caption{A typical structure of the matrix $A$ for non-negative angular indices $m$ and $m^{'}$. Black-coloured regions correspond non-zero entries, where each block corresponds to a different angular index~$m=m^{'}$.}
        \label{fig:A block-diagonal typical}
    \end{minipage}
\end{figure}

Lastly, a favorable interpretation of this procedure can be derived as follows. If we denote by $Q^{(m)}$ a matrix whose columns are orthonormal and span the columns of $\widetilde{V}^{(m)}$, then $\hat{X}^{(m)}$ can be written equivalently as
\begin{equation}
\hat{X}^{(m)} = Q^{(m)} \left[Q^{(m)}\right]^{*} X^{(m)} = C^{(m)} X^{(m)},	\label{eq:x_hat_m filtering}
\end{equation}
where $(\cdot)^*$ denotes complex-conjugate and transpose, and we defined the $N\times N$ ``filtering'' matrices 
\begin{equation}
C^{(m)} \triangleq Q^{(m)} \left[Q^{(m)}\right]^{*},
\end{equation}
which are applied to our dataset for every angular index separately. Essentially, $C^{(m)}$ acts as a ``low-pass filter'', in the sense that it retains only the contribution of steerable manifold harmonics with low frequencies (i.e. eigenvalues below the threshold $\lambda_c$). In this context,  the cut-off frequency $\lambda_c$ controls the rank of $C^{(m)}$, which is equal to $k_m$, and the degree to which $C^{(m)}$ suppresses oscillations in the data. 

\section{Algorithms summary and computational cost} \label{section:Algorithm summary and computational cost}
We outline the algorithms for evaluating the steerable manifold harmonics and employing them for filtering image datasets in Algorithms~\ref{alg:Evaluating the steerable manifold harmonics} and~\ref{alg:Rotationally-invariant dataset filtering}, respectively. We note that two optional modifications to the procedure of evaluating the steerable manifold harmonics are proposed in Section~\ref{section:Analysis under Gaussian noise} and Appendix~\ref{appendix: Non uniform sampling}, respectively. The first modification is for improving the robustness of the procedure to noise, and was added to Algorithm~\ref{alg:Evaluating the steerable manifold harmonics} in step~\ref{alg1-step:implicit debiasing} under the label ``Implicit debiasing (optional)''. The second modification, which is used for normalizing non-uniform sampling densities, was added to Algorithm~\ref{alg:Evaluating the steerable manifold harmonics} in step~\ref{alg1-step:density normalization} under the label ``Density normalization (optional)''.

\begin{algorithm}
\caption{Evaluating the steerable manifold harmonics}\label{alg:Evaluating the steerable manifold harmonics}
\begin{algorithmic}[1]
\Statex{\textbf{Required:} A dataset of $N$ points $\left\{x_1,\ldots,x_N\right\}\subset\CD$, where $x_{i,(m,\ell)}$ is the $(m,\ell)$'th coordinate of $x_i$ (see Section~\ref{subsection:Problem se-tup}).}
\State Choose a numerical-integration parameter $K$ (see~\eqref{eq:W_hat FFT approx}), and a Gaussian kernel parameter $\varepsilon$.
\State For every $1\leq i,j \leq N$, $k\in\left\{0,\ldots,K-1\right\}$, compute the affinities
\begin{equation}
{W}_{i,j}^{(k)} = \exp{\left\lbrace-{\left\Vert x_{i} - x_{j}^{(k)}  \right\Vert^2 }{/\varepsilon}\right\rbrace}, \quad\quad\quad x_{j,(m,\ell)}^{(k)} = x_{j,(m,\ell)} e^{\imath 2\pi m k/K}.
\end{equation}
\State \label{alg1-step:implicit debiasing} Implicit debiasing \textbf{(optional)}: Set ${W}_{i,i}^{(k)} = 0$ for $1\leq i \leq N$ and $k=0,\ldots,K-1$.
\State \label{alg1-step:w_hat} For every angular index $m=-M,\ldots,M$ and $1\leq i,j \leq N$, evaluate
\begin{equation}
\hat{W}_{i,j}^{(m)} = \frac{2\pi}{K} \sum_{k=0}^{K-1} {W}_{i,j}^{(k)} e^{ \imath 2\pi m k/K}, \quad\quad\quad D_{i} = \sum_{j=1}^N \hat{W}_{i,j}^{(0)}.
\end{equation}
\State \label{alg1-step:density normalization} Density normalization \textbf{(optional)}: 
\begin{enumerate} [label=(\alph*)]
\item For every angular index $m=-M,\ldots,M$ update:
\begin{equation}
\hat{W}^{(m)} \leftarrow D^{-1} \hat{W}^{(m)} D^{-1},
\end{equation} 
where $D$ is a diagonal matrix with $\left\{D_{i}\right\}_{i=1}^N$ on its diagonal.
\item For every $i=1,\ldots,N$ update:
\begin{equation}
D_{i} \leftarrow \sum_{j=1}^N \hat{W}_{i,j}^{(0)}.
\end{equation} 
\end{enumerate}
\State \label{alg1-step:eigen-decomposition} For every angular index $m=-M,\ldots,M$ form the matrix
\begin{equation}
\tilde{S}_m = I-D^{-1}\hat{W}^{(m)},
\end{equation}
and return its eigenvectors $\left\{\tilde{v}_{m,k}\right\}_{k=1}^N$ and eigenvalues $\left\{\tilde{\lambda}_{m,k}\right\}_{k=1}^N$.
\end{algorithmic}
\end{algorithm}

\begin{algorithm}
\caption{Rotationally-invariant dataset filtering}\label{alg:Rotationally-invariant dataset filtering}
\begin{algorithmic}[1]
\Statex{\textbf{Required:} 
\begin{enumerate}[label=(\alph*)]
\item A dataset of $N$ points $\left\{x_1,\ldots,x_N\right\}\subset\CD$, where $x_{i,(m,\ell)}$ is the $(m,\ell)$'th coordinate of $x_i$ (see Section~\ref{subsection:Problem se-tup}).
\item Eigenvectors $\left\{\tilde{v}_{m,1},\ldots,\tilde{v}_{m,N}\right\}_{m=-M}^M$ and eigenvalues $\left\{\tilde{\lambda}_{m,1},\ldots,\tilde{\lambda}_{m,N}\right\}_{m=-M}^M$ of $\left\{\tilde{S}_m\right\}_{m=-M}^M$ from Algorithm~\ref{alg:Evaluating the steerable manifold harmonics}.
\end{enumerate} }
\State Choose a cut-off frequency $\lambda_c$.
\State For $m=-M,\ldots,M$ do \label{alg2-step:filter procedure}
\begin{enumerate}[label=(\alph*)]
\item \label{alg2-step:truncation rule} Compute $k_m = \max\left\{k: \tilde{\lambda}_{m,k} < \lambda_c \right\}$, and form the matrices $X^{(m)}$ and $\widetilde{V}^{(m)}$ of~\eqref{eq:X^m and V_tilde^m def}.
\item \label{alg2-step:B est by ls} Estimate the coefficients matrix $B^{(m)}$ by solving the least squares system of~\eqref{eq:B^m least squares}.
\item \label{alg2-step:filter output} Compute $\hat{X}^{(m)} = \widetilde{V}^{(m)} B^{(m)}.$
\end{enumerate}
\State The filtered dataset is given by $\hat{X} = [\hat{X}^{(-M)} \cdots \hat{X}^{(M)}]$.
\end{algorithmic}
\end{algorithm}

We now turn our attention to the computational complexity of Algorithms~\ref{alg:Evaluating the steerable manifold harmonics} and~\ref{alg:Rotationally-invariant dataset filtering}.
We begin with Algorithm~\ref{alg:Evaluating the steerable manifold harmonics}. The first step is to compute all affinity measures $W^k_{i,j}$, which can be evaluated efficiently by the FFT if we notice that
\begin{align}
&\left\Vert x_{i} - x_{j}^{(k)}  \right\Vert^2 = \left\Vert x_i \right\Vert^2_2 + \left\Vert x_j \right\Vert^2_2 -2 \operatorname{Re}{ \left\{ \sum_{m=-M}^M c_{i,j}^{(m)} e^{-\imath 2\pi m k/K}\right\}}, \label{eq:W_ij^k computation efficient}
\end{align}
where we defined 
\begin{equation}
c_{i,j}^{(m)} =  \sum_{\ell=1}^{\ell_m} x_{i,(m,\ell)} x_{j,(m,\ell)}^*.
\end{equation}
Note that computing $c_{i,j}^{(m)}$ for all $i,j,m$ takes $O(N^2 \mathcal{D})$ operations. Therefore, if we denote 
\begin{equation}
\bar{M}=\max{\left\{M,K\right\}},
\end{equation}
the computational complexity of this step is $O(N^2D+N^2\bar{M}\log{\bar{M}})$ when using the FFT to compute~\eqref{eq:W_ij^k computation efficient}. In a similar fashion, computing $\hat{W}^m_{i,j}$ (of step~\ref{alg1-step:w_hat}) by the FFT takes $O(N^2\bar{M}\log{\bar{M}})$ operations. Lastly, forming the matrices $\left\{\tilde{S}_m\right\}_{m=-M}^M$ requires $O(MN^2)$ operations, and evaluating its eigenvectors and eigenvalues takes $O(MN^3)$ operations. Overall, the computational complexity of Algorithm~\ref{alg:Evaluating the steerable manifold harmonics} is therefore
\begin{equation}
O\left(MN^3+N^2(\mathcal{D}+\bar{M}\log{\bar{M}})\right).
\end{equation}
In practice, it is often the case that only a small fraction of pairs of indices $(i,j)$ contributes significantly to $W_{i,j}^k$, since only images which are similar up to a planar rotation admit a non-negligible affinity (assuming that $\varepsilon$ is sufficiently small). Hence, it is often possible to zero-out the small values of $W_{i,j}^k$, allowing for cheaper sparse-matrix computations. Additionally, computing the eigen-decomposition in step~\ref{alg1-step:eigen-decomposition} for large datasets (large $N$) may be accomplished more efficiently using randomized methods~\cite{halko2011finding,aizenbud2016matrix}.

As of Algorithm~\ref{alg:Rotationally-invariant dataset filtering}, in part~\ref{alg2-step:B est by ls} of step~\ref{alg2-step:filter procedure} we need to minimize $\left\Vert
X^{(m)} - \widetilde{V}^{(m)}B^{(m)}\right\Vert^2_F$ over ${B^{(m)}}$, where $X^{(m)}$ is of dimension $N\times \ell_m$, $\widetilde{V}^{(m)}$ is $N\times k_m$, and $B^{(m)}$ is $k_m\times \ell_m$. For each angular index $m$, this amounts to solving $\ell_m$ least-squares problems (one for each column of $B^{(m)}$), each of size of size $N\times k_m$. Assuming that $N\geq k_m$ and using the QR factorization to solve least-squares, this part requires $O(k_m^2 N + N k_m \ell_m + k_m^2 \ell_m)$ operations for every angular index $m$, as we need $O(k_m^2 N)$ operations to compute the QR decomposition of $V^{(m)}$ (which needs to be computed only once), $O(k_m \ell_m N)$ operations to apply $Q$ of the QR to $X^{(m)}$, and $O(k_m^2 \ell_m)$ operations to solve the resulting $\ell_m$ triangular systems. Then, since part~\ref{alg2-step:filter output} of step~\ref{alg2-step:filter procedure} takes $O(N k_m \ell_m)$ operations for each $m$, it follows that Algorithm~\ref{alg:Rotationally-invariant dataset filtering} requires
\begin{equation}
O\left(\sum_{m=-M}^M k_m^2 N + N k_m \ell_m + k_m^2 \ell_m\right) = O\left(MN\bar{k}^2 + \mathcal{D}\bar{k}(N+\bar{k})\right)
\end{equation}
operations, where $\bar{k}=\max_m{\left\{k_m\right\}}$. As typically $\bar{k}<<N$, the computational cost of Algorithm~\ref{alg:Rotationally-invariant dataset filtering} is negligible compared to that of Algorithm~\ref{alg:Evaluating the steerable manifold harmonics}.

\section{Analysis under Gaussian noise} \label{section:Analysis under Gaussian noise}
Next, we analyze our method under white Gaussian noise, and argue that in a certain sense the steerable graph Laplacian  is robust to noise (after zeroing-out the diagonal of the steerable affinity operator ${W}$). Moreover, we argue that the filtering procedure (described in Section~\ref{subsection:Expanding image datasets by the steerable manifold harmonics}) allows us to reduce the amount of noise in the filtered dataset proportionally to the number of images $N$. 

In this section, we consider the noisy data points
\begin{equation}
y_{i,(m,\ell)} = x_{i,(m,\ell)} + \eta_{i,(m,\ell)}, \label{eq:noisy data}
\end{equation}
where $x_{i,(m,\ell)}$ is the $(m,\ell)$'th coordinate of the $i$'th clean data point, and $\left\{ \eta_{i,(m,\ell)} \right\}$ are independent and normally distributed complex-valued noise variables with mean zero and variance $\sigma^2$ .

\subsection{Noise robustness of the steerable graph Laplacian} \label{subsection:Noise robustness of the steerable graph Laplacian}
We start by considering the effect of noise on the construction of $\tilde{L}$ (of~\eqref{eq:normalized steerable graph laplacian}). Clearly, the noise changes the pairwise distances computed in $W_{i,j}(\vartheta,\varphi)$, where we note that from Theorem~\ref{thm:eigenfunctions and eigenvalues of L_tilde} and~\eqref{eq:w_hat_ij mat fourier} it is sufficient to consider the effect of noise only on $D^{-1}_{i,i} W_{i,j}(0,\alpha)$, for all $\alpha\in [0,2\pi)$ and $i,j=1,\ldots,N$. To this end, consider the set of points $\mathcal{Y}_{i}^\alpha=\left\{y_1^\alpha,\cdots,y_{i-1}^\alpha,y_i,y_{i+1}^\alpha,\cdots,y_N^\alpha\right\}\subset\CD$, where all points except the $i$'th were replaced with their rotations by an angle $\alpha$ (via~\eqref{eq:extrinsic rotation}). We have that
\begin{equation}
y_{j}^\alpha = x_{j}^\alpha + \eta_{j}^\alpha, \quad\quad y_{j,(m,\ell)}^\alpha = x_{j,(m,\ell)} e^{\imath m\alpha}+ \eta_{j,(m,\ell)} e^{\imath m\alpha},
\end{equation}
for $j\neq i, \; j=1,\ldots,N$, and it is evident that the set of noise points $\left\{\eta_1^\alpha,\cdots,\eta_{i-1}^\alpha,\eta_i,\eta_{i+1}^\alpha,\cdots,\eta_N^\alpha\right\}$ are still i.i.d Gaussian.
Then, Theorem~$2.1$ (and specifically equation $(1)$) in~\cite{el2010information}, when applied to the set $\mathcal{Y}_{i}^\alpha$, asserts that if we denote $\gamma=\mathcal{D}\sigma^2$ and vary $\mathcal{D}$ and $\sigma^2$ such that $\gamma$ remains constant, then
\begin{equation}
W_{i,j}(0,\alpha) = \exp{\left\{-\left\Vert y_i - y_j^\alpha \right\Vert^2/\varepsilon\right\}} \underset{\mathcal{D}\rightarrow \infty}{\longrightarrow} \exp{\left\{-\left(\left\Vert x_i - x_j^\alpha \right\Vert^2 + 2\gamma\right)/\varepsilon\right\}}
\end{equation}
in probability, for all $j\neq i$. Essentially, this result is due to the concentration of measure of high-dimensional Gaussian random vectors, and in particular the fact that $\left\{ \eta_i^\alpha \right\}$ are uncorrelated and are concentrated around the surface of a sphere in $\CD$. Therefore, in the regime of high dimensionality and small noise-variance, the effect of the noise on the pairwise distances (between different data-points and rotations) is only an additive constant bias term. Note that even though the noise variance tends to zero, the overall noise magnitude $\gamma=\mathcal{D}\sigma^2$ is kept constant and may be large, corresponding to a low signal-to-noise ratio (SNR). We further mention that this constant-bias effect is not restricted to Gaussian white noise, as it occurs also when the noise admits a general covariance matrix $\Sigma$, and even when the noise takes other certain non-Gaussian distributions (see~\cite{el2010information} for specific details and conditions). 

Next, in order to correct for the bias in the distances, we follow~\cite{el2016graph} and zero-out the diagonal of $W$, that is, we update
\begin{equation}
W_{i,j}(0,\alpha) \leftarrow 
\begin{dcases}
W_{i,j}(0,\alpha), & i\neq j, \\
0, & i=j.
\end{dcases}
\end{equation}
Then, we expect $D^{-1}$ to correct (implicitly) for the bias in $W$, since
\begin{align}
D_{i,i} = &\sum_{j\neq i, j=1}^N \int_0^{2\pi} W_{i,j}(0,\alpha) d\alpha = \sum_{j\neq i, j=1}^N \int_0^{2\pi} \exp{\left\{-\left\Vert y_i - y_j^\alpha \right\Vert^2/\varepsilon\right\}} d\alpha \nonumber \\ 
\underset{\mathcal{D}\rightarrow \infty}{\longrightarrow} &e^{-2\gamma/\varepsilon} \sum_{j\neq i, j=1}^N \int_0^{2\pi} \exp{\left\{-\left\Vert x_i - x_j^\alpha \right\Vert^2/\varepsilon\right\}} d\alpha,
\end{align}
in probability, and thus 
\begin{equation}
D^{-1}_{i,i} W_{i,j}(0,\alpha) \underset{\mathcal{D}\rightarrow \infty}{\longrightarrow} \frac{\exp{\left\{-\left\Vert x_i - x_j^\alpha \right\Vert^2/\varepsilon\right\}}}{\sum_{j\neq i, j=1}^N \int_0^{2\pi} \exp{\left\{-\left\Vert x_i - x_j^{\alpha^{'}}\right\Vert^2/\varepsilon\right\}} d\alpha^{'}}
\end{equation}
in probability, which is equivalent to its clean counterpart for $i\neq j$ (after zeroing-out the diagonal).
Lastly, we argue that zeroing-out the diagonal of $W$ does not change the point-wise convergence rate of the clean steerable graph Laplacian (as reported by Theorem~\ref{thm:steerable graph Laplacian convergence}), as it results in an error which is negligible compared to the leading error terms (see the end of Section~\ref{subsection:the variance term} in the proof of Theorem~\ref{thm:steerable graph Laplacian convergence}, and~\cite{singer2006graph} for an analogous argument in the case of the standard graph Laplacian). 

In Figure 6, we show the error in estimating $\Delta_\mathcal{M}f(x_{0})$ in a noisy high-dimensional counterpart of the numerical example of Section~\ref{subsection:Numerical example: Steerable graph Laplacian convergence rate} (using the same setting of $N=2,000$ and the optimal choice $\varepsilon=2^{-0.75}$). To generate Figure~\ref{fig:debiased steerable laplacian toy exm}, we embedded the unit sphere in increasing dimensions $\mathcal{D}$ (using a random orthogonal transformation) and added white Gaussian noise with variance $\sigma^2$ to each dimension, such that $\mathcal{D}\sigma^2 = \gamma$ is kept fixed. We then compared the error in estimating $\Delta_\mathcal{M}f(x_{0})$ to the error obtained in the clean setting. Note that for the unit sphere the signal-to-noise ratio (SNR) is equal to $1/\gamma=1/(\mathcal{D}\sigma^2)$. It is evident that as predicted by our analysis, the debiased steerable graph Laplacian converges to the clean steerable graph Laplacian in the regime of high dimensionality and small noise variance. Particularly, in the case of $SNR=10$ ($\gamma=0.1$), already for $\mathcal{D}=100$ the error resulting only from the noise becomes comparable to the approximation error in the clean setting. In the case of $SNR=1$ ($\gamma=1$), this happens roughly at $\mathcal{D}=1,000$.

\begin{figure}
  \centering  	
    \includegraphics[width=0.6\textwidth]{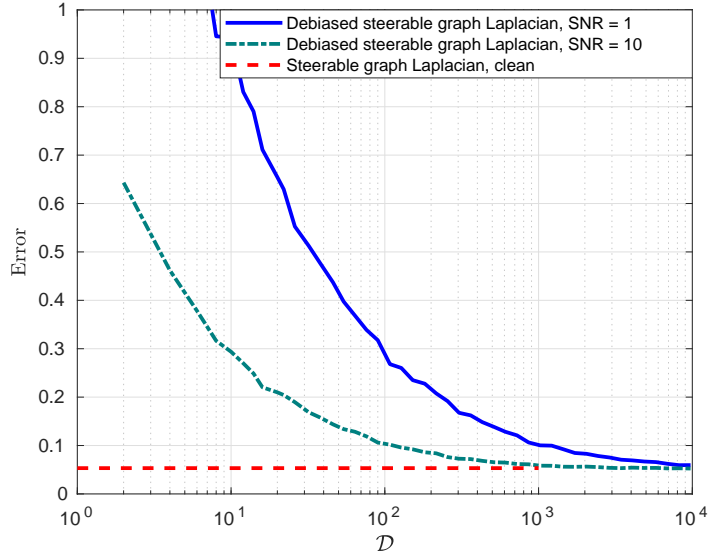}
	\caption[LaplacianError] 
	{Errors in approximating the Laplace-Beltrami operator on the unit sphere (following the numerical example in Section~\ref{subsection:Numerical example: Steerable graph Laplacian convergence rate}) using the debiased steerable graph Laplacian constructed from noisy data points, as function of the ambient dimension $\mathcal{D}$. The debiased steerable graph Laplacian is computed using $\varepsilon=2^{-0.75}$ (which was optimal for the clean case) and $N=2,000$ noisy measurements, where the noise is additive white Gaussian with variance $\sigma^2$, while $SNR = 1/(\mathcal{D}\sigma^2)$ is kept fixed. The red dashed horizontal line corresponds to the error obtained from the clean steerable graph Laplacian in this setting (see Figure~\ref{fig:steerable laplacian toy exm}).
	}  \label{fig:debiased steerable laplacian toy exm}
\end{figure}

In summary, the analysis and numerical example in this section suggest that when the dimension $\mathcal{D}$ is large, the noise variance $\sigma^2$ is small, and the overall noise magnitude $\gamma=\mathcal{D}\sigma^2$ is fixed (and may be large compared to the magnitude of the signal), the steerable graph Laplacian constructed from the noisy data after implicit debiasing (by zeroing-out the diagonal of $W$) is expected to be close to its clean counterpart.

\subsection{Performence of the filtering procedure} \label{subsection:Performence of the filtering procedure}
Next, we consider the eigenvectors and eigenvalues computed from the clean (normalized) steerable graph Laplacian, and analyze the effect of the filtering procedure (described in Section~\ref{subsection:Expanding image datasets by the steerable manifold harmonics}) on the noise in the dataset. From~\eqref{eq:x_hat_m filtering}, the de-noised data for angular frequency $m$ is given by
\begin{equation}
\widetilde{X}^{(m)} = C^{(m)} Y^{(m)} = Q^{(m)} \left[Q^{(m)}\right]^{*} Y^{(m)},
\end{equation}
where $Y^{(m)}$ is the matrix of noisy data points corresponding to angular index $m$, i.e. $Y^{(m)}_{i,\ell} = y_{i,(m,\ell)}$.
Since $Q^{(m)}$ consists of $k_m$ orthonormal column vectors independent of the noise, and recalling that $\gamma=\mathcal{D}\sigma^2$, $\mathcal{D} = \sum_{m=-M}^M \ell_m$, we have that
\begin{equation}
\frac{1}{N}\mathbb{E}\left\Vert \hat{X} - \widetilde{X} \right\Vert^2_F = \frac{1}{N}\sum_{m=-M}^M \mathbb{E}\left\Vert \hat{X}^{(m)} - \widetilde{X}^{(m)} \right\Vert^2_F = \frac{\sigma^2 \sum_{m=-M}^M k_m\ell_m}{N} \leq \max_{m}{\left\{ k_m\right\}}\frac{\gamma}{N},	\label{eq:variance error term}
\end{equation}
where $\hat{X} = [\hat{X}^{(-M)} \cdots \hat{X}^{(M)}]$ ($\hat{X}^{(m)}$ is defined in~\eqref{eq:x_hat_m filtering}) represents the clean filtered dataset, and $\widetilde{X}=[\widetilde{X}^{(-M)} \cdots \widetilde{X}^{(M)}]$ represents the noisy filtered dataset. Hence, larger datasets are expected to provide improved de-noising results, as the noise in the filtered dataset $\widetilde{X}$ reduces proportionally to $1/N$.

At this point, it is worthwhile to point out that the error bound of~\eqref{eq:variance error term} is significantly better than what we would expect from using the standard graph Laplacian and its eigenvectors (to filter the coordinates of the dataset). Fundamentally, this is due to the block diagonal structure of the coefficients matrix $A$ (see Section~\ref{subsection:Expanding image datasets by the steerable manifold harmonics}), and more specifically, the fact that we only need to use eigenfunctions with angular index $m$ to expand data coordinates with the same angular index, in contrast to using all eigenfunctions. In particular, since the limiting operators of the steerable and standard graph Laplacians are the same (the Laplace-Beltrami operator), we expect the truncation rule of~\eqref{eq:truncation rule} to provide a similar number of eigenfunctions/eigenvectors from both methods. Then, if we were to use the eigenvectors of the standard graph Laplacian to filter our dataset, we would be required to use all $\sim\sum_{m=-M}^M k_m$ eigenvectors, and by a computation equivalent to~\eqref{eq:variance error term} we would expect an error of $\sim \sum_{m=-M}^M k_m \frac{\gamma}{N}$, which is considerably larger than $\max_{m}{\left\{ k_m\right\}}\frac{\gamma}{N}$. In conclusion, as the steerable graph Laplacian is more informative than the standard graph Laplacian, in the sense that it provides us with the angular part of each eigenfunction, it allows us to be more precise when filtering our dataset by matching the angular frequencies of the eigenfunctions to those of the data, thereby reducing the computational complexity and improving the de-noising performance considerably. This feature of the steerable graph Laplacian stands on its own, and is separate from the improved convergence rate to the Laplace-Beltrami operator (Theorem~\ref{thm:steerable graph Laplacian convergence}), which improves the accuracy of the eigenfunctions and eigenvalues compared to those of the standard graph Laplacian.

Lastly, we mention that the error term~\eqref{eq:variance error term} can be viewed as a variance error term in a classical bias-variance trade-off, as we can write, conditioned on the clean dataset $X$, that
\begin{align}
\frac{1}{N}\mathbb{E}\left\Vert {X} - \widetilde{X} \right\Vert^2_F &= \frac{1}{N}\mathbb{E}\biggl\Vert {X} - \overbrace{\mathbb{E}\left[\widetilde{X} \right]}^{=\hat{X}} + \overbrace{\mathbb{E}\left[\widetilde{X} \right]}^{=\hat{X}} - \widetilde{X} \biggr\Vert^2_F \nonumber \\ 
&= \underbrace{\frac{1}{N}\mathbb{E}\left\Vert {X} - \hat{X} \right\Vert^2_F}_\text{Bias} +\underbrace{\frac{1}{N}\mathbb{E}\left\Vert \hat{X} - \widetilde{X} \right\Vert^2_F}_\text{Variance}. \label{eq:bias variance trade-off}
\end{align}
Consequently, the overall error cannot get arbitrarily small, as there exists a bias term when approximating the clean data points by finitely many eigenfunctions (see Proposition~\ref{prop:man approx bias error}). Therefore, in practice, the optimal de-noising results for a given dataset and noise variance would be attained as an optimum in a bias-variance trade-off, where a large cut-off frequency $\lambda_c$ would result in larger $\left\{ k_m\right\}$ values and a larger variance error (as noise is mapped to more expansion coefficients), and a smaller cut-off frequency $\lambda_c$ would result in smaller $\left\{ k_m\right\}$ values and a larger bias error.

\begin{remark}
While the discussion in this section suggests that our method is robust to noise in the high-dimensional regime, it is not to say that reducing the dimensionality of a given dataset (with a given and fixed noise variance $\sigma^2$) would degrade the accuracy of the quantities computed by our method. On the contrary, a close examination of the results in~\cite{el2010information} reveals that the errors in pairwise distances computed from noisy data points are dominated by $\sqrt{\mathcal{D}}\sigma^2$, meaning that projecting the data onto a lower-dimensional subspace (while retaining a sufficient approximation accuracy w.r.t the clean data) is encouraged -- as it improves the accuracy of the pairwise affinities on one hand, and reduces the overall noise magnitude $\gamma=\mathcal{D}\sigma^2$ on the other. 
\end{remark}

\section{Example: De-noising cryo-EM projection images} \label{section:Denoising rotationally-invariant image datasets}
In this section, we demonstrate how we can use our framework to de-noise single-particle cryo-electron microscopy (cryo-EM) image datasets. 

\subsection{Cryo-EM}
In single-particle cryo-EM~\cite{Frank,cheng2015primer}, one is interested in reconstructing a three-dimensional model of a macromolecule (such as a protein) from a set of two-dimensional images taken by an electron microscope. The procedure begins by embedding many copies of the macromolecule in a thin layer of ice, where due to the experimental set-up, the different copies are frozen at random unknown orientations. Then, an electron microscope acquires two-dimensional images of the these macromolecules (more precisely, it samples the Radon transform of the density function of the macromolecule). Consequently, it can be shown that the set of all projection images lies on a three-dimensional manifold diffeomorphic to the group SO(3). Thus, the manifold model assumption discussed in this work is natural for describing cryo-EM datasets. Note that due to the experimental set-up in cryo-EM, the in-plane rotation of each copy of the macromolecule is arbitrary, and therefore, so are the planar rotations of the two-dimensional images. Additionally, the images acquired in cryo-EM experiments are very noisy, with a typical SNR (Signal-to-Noise Ratio) of $1/10$ and lower.
Simulated clean and noisy cryo-EM images of the 70S ribosome subunit can be seen in Figure~\ref{fig:70S 1e4 snr -10dB denoised} (top two rows).

\subsection{De-noising recipe}
Given a collection of cryo-EM projection images $\left\{ I_1,\ldots,I_N\right\}$ sampled on a Cartesian grid, we start by performing steerable principal components analysis (sPCA), as described in~\cite{landa2017steerable}. This procedure provides us with steerable basis functions (the steerable principal components) $\left\{\psi_{m,\ell}\right\}$ of the form of~\eqref{eq:image subspace and psi polar form}, which are optimal for expanding the images in the dataset and all of their rotations. For each basis function $\psi_{m,\ell}$, the steerable PCA also returns its associated eigenvalue $\nu_{m,k}$, which encodes the contribution of $\psi_{m,\ell}$ to the expansion (analogously to the eigenvalues of the covariance matrix in standard PCA). 
Therefore, we have that
\begin{equation}
I_i \approx \sum_{m=-M}^M\sum_{\ell=1}^{\ell_m} y_{i,(m,\ell)} \psi_{m,\ell}, \label{eq:noisy images expansion in basis}
\end{equation}
where $y_{i,(m,\ell)}$ is the $(m,\ell)$'th expansion coefficient of the $i$'th image (provided by sPCA, see~\cite{landa2017steerable} for appropriate error bounds associated with~\eqref{eq:noisy images expansion in basis}). Expanding the image dataset using such basis functions
allows us to apply our filtering scheme in the domain of the expansion coefficients, as required by our algorithms.
We note that for images corrupted by additive white Gaussian noise, the noise variance $\sigma^2$ is estimated from the corners of the images (where no molecule is expected to be present), and the number of basis functions used in the expansion, governed by $M$ and $\left\{\ell_m\right\}$, is determined by estimating which eigenvalues $\nu_{m,k}$ are above the noise level (i.e. exceed the Baik-Ben Arous-P{\'e}ch{\'e} transition point~\cite{baik2005phase}, see also~\cite{zhao2013fourier,zhao2014fast}) via
\begin{equation}
\ell_m = \max\left\{\ell: \; \nu_{m,\ell}> \sigma^2 \left( 1+\sqrt{\frac{n_m}{N}}\right) ^2 \right\}, \label{eq:spca truncation automatic}
\end{equation}
where $\left\{n_m\right\}$ can be found in~\cite{landa2017steerable} ($n_m$ is the size of the $m$'th block in the block-diagonal covariance matrix associated with steerable PCA), and assuming that $\left\{\nu_{m,k}\right\}
_k$ are sorted in a non-increasing order for every $m$. Correspondingly, $M$ in~\eqref{eq:noisy images expansion in basis} is simply the largest $\left\vert m \right\vert$ s.t. $\ell_m>0$.

Using the above setting, the task of de-noising the images $\left\{I_i\right\}$ is reduced to the task of de-noising the sPCA coefficients $\left\{y_{i,(m,\ell)}\right\}$.
We then estimate the steerable manifold harmonics $\left\{\tilde{v}_{m,k}\right\}$ (as described by Algorithm~\ref{alg:Evaluating the steerable manifold harmonics}) from the dataset $\left\{y_{i}\right\}_{i=1}^N$, and follow by employing $\left\{\tilde{v}_{m,k}\right\}$ for filtering the dataset according to Algorithm~\ref{alg:Rotationally-invariant dataset filtering}. 
After obtaining the de-noised expansion coefficients $\left\{\hat{x}_{i,(m,\ell)}\right\}$, we can plug them back in the expansion~\eqref{eq:image subspace and psi polar form} to get de-noised images.
The procedure is summarized in Figure~\ref{fig:denoising schem illustration}. Note that since $\left\{I_i\right\}$ are real-valued images, their expansion coefficients satisfy the symmetry
\begin{equation}
y_{i,(-m,\ell)} = y_{i,(m,\ell)}^*.
\end{equation}
Therefore, it is sufficient to de-noise only the coefficients with non-negative angular frequencies.

\begin{figure}
  \centering  	
    \includegraphics[width=1.0 \textwidth]{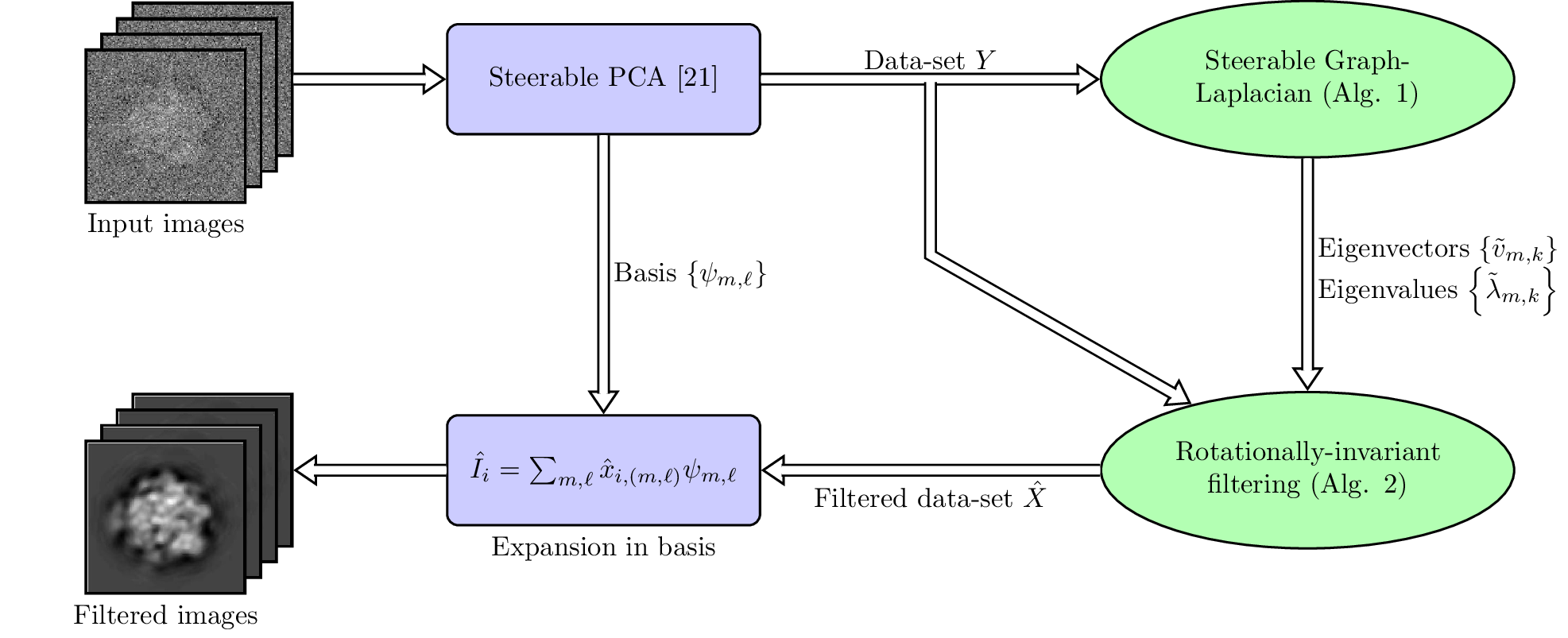}
	\caption[Denoised images] 
	{Schematic view of the de-noising procedure. We start by applying the steerable PCA~\cite{landa2017steerable} to the input images, obtaining basis functions $\left\{ \psi_{m,\ell} \right\}$ and associated expansion coefficients $\left\{ y_{i,(m,\ell)} \right\}$ (organized into the matrix $Y$ following the layout of~\eqref{eq:dataset X}), where the truncation of the expansion is due to~\eqref{eq:spca truncation automatic}. Then, the dataset $Y$ is used to construct the steerable graph Laplacian, whose eigenfunctions (the steerable manifold harmonics) are obtained via Algorithm~\ref{alg:Evaluating the steerable manifold harmonics} with implicit debiasing (step \ref{alg1-step:implicit debiasing}) and without denisty normalization (step \ref{alg1-step:density normalization}). Using the steerable manifold hamronics, we filter the dataset via Algorithm~\ref{alg:Rotationally-invariant dataset filtering}, and use the filtered coefficients $\left\{ \hat{x}_{i,(m,\ell)} \right\}$ in conjunction with the basis  functions $\left\{ \psi_{m,\ell} \right\}$ to get back the filtered images.} \label{fig:denoising schem illustration}
\end{figure}

\subsection{Experimental results}
We demonstrate the de-noising performance of our approach using simulated images of the 70S ribosome, of size $128\times 128$ pixels, after applying a filter to all images corresponding to a typical Contrast Transfer Function (CTF)~\cite{Frank} of the electron microscope. As described previously, we first map all images to their sPCA coefficients via~\cite{landa2017steerable} (with $T=10$ and half-Nyquist bandlimit), and then proceed according to our filtering scheme (Algorithms~\ref{alg:Evaluating the steerable manifold harmonics} and~\ref{alg:Rotationally-invariant dataset filtering}). We mention that throughout our experiments the choice $K=256$ was found satisfactory, and that $\varepsilon$ and $\lambda_c$ were chosen automatically for every experimental set-up (determined by the number of images $N$ and noise variance $\sigma^2$) as described in Appendix~\ref{appendix: Automatic choice of parameters}. In every experiment, we compare the de-noised images resulting from our method to the images obtained directly from the sPCA coefficients (i.e. images computed from the coefficients $\left\{y_{i,(m,\ell)}\right\}$), and to images obtained after applying a shrinkage to the sPCA coefficients via $y_{i,(m,\ell)} w_{m,\ell}$, where the weights $\left\{w_{m,1},\ldots,w_{m,\ell_m}\right\}_{m=-M}^M$, which were computed as described in~\cite{zhao2013fourier}, correspond to the asymptotically-optimal Wiener filter~\cite{singer2013two}. Essentially, this is the optimal filter for the expansion coefficients in the sense of minimizing the mean squared error. 

First, we demonstrate our method on $10,000$ projection images at signal-to-noise ratio of $1/20$. The de-noised images can be seen in Figure~\ref{fig:70S 1e4 snr -10dB denoised}, where it is visually evident that the final de-noised images using our method contain many more details compared to sPCA Wiener filtering, which results in somewhat blurred images due to the aggressive shrinkage of sPCA coefficients. In terms of performance measures, our method  (which we term ``sMH filtering'', where sMH stands for ``steerable manifold harmonics'') results in an average peak-SNR (pSNR) of $25.37$dB, where the sPCA Wiener filter provided $21.17$dB pSNR, and sPCA alone resulted in $17.64$dB pSNR.

It is therefore evident that Wiener filtering of sPCA coefficients is far from optimal in terms of de-noising and image recovery, as it essentially applies a single linear operator on the individual images, which is only optimal when the data resides on a linear subspace. However, in the case of cryo-EM, as the data resides on a manifold, it is reasonable to apply non-linear methods which account for the geometry and topology of the manifold. In this respect, our method is able to make use of all images and their rotations simultaneously to accurately estimate the structure of the manifold, and thereby provides an improved de-noising of the image dataset.

\begin{figure}
  \centering  	
    \includegraphics[width=0.65\textwidth]{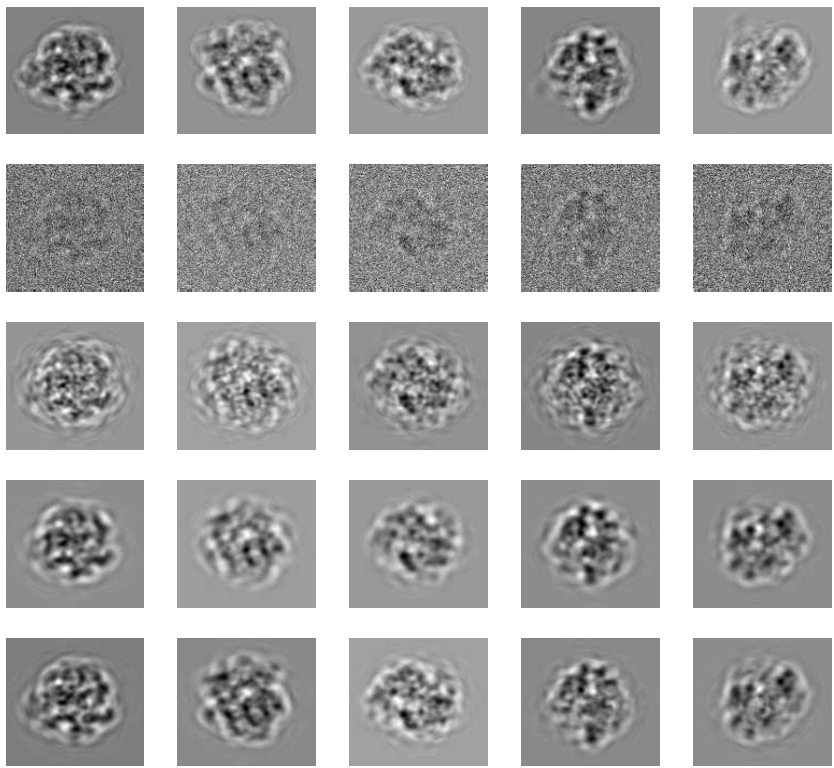}
	\caption[Denoised images] 
	{Images after de-noising, for $N=10,000$ and SNR~$=1/20$. Different colummns correspond to different images from the dataset (different in-plane rotations and viewing directions of the molecule), while different rows correspond to (from top to bottom): clean images, noisy images, sPCA (pSNR~$=17.64$dB), sPCA Wiener filter (pSNR~$=21.17$dB), and sMH filtering (this paper, pSNR~$=25.37$dB).}  \label{fig:70S 1e4 snr -10dB denoised}
\end{figure}

Next, Figure~\ref{fig:70S snr -10dB different N denoised} demonstrates the performance of our method for SNR~$=1/10$ and different values of $N$ (dataset size). As anticipated, our method is able to exploit larger datasets for improved de-noising, whereas the sPCA Wiener filtering offers only a mild gain beyond $2,000$ images. The reason for that is that the Wiener filtering is applied to each image separately, and therefore reaches saturation once the estimation of the sPCA from the noisy data is sufficiently accurate (approaches the sPCA of the clean data). Note that the pSNR from the projection onto the sPCA components (without shrinkage) reduces with $N$, because more basis functions $\psi_{m,\ell}$ are chosen (according to~\eqref{eq:spca truncation automatic}) as $N$ increases, even if their contribution to expanding the dataset is negligible. Therefore, the dimension $\mathcal{D}$ increases, and with it also the overall noise magnitude $\mathcal{D}\sigma^2$.
It is important to mention that even though the variance error term in~\eqref{eq:variance error term} behaves like $1/N$, the improvement in the pSNR of our method is not expected to follow this trend, since the overall error also includes a bias error term (see~\eqref{eq:bias variance trade-off}), such that the minimal error for every value of $N$ is attained as a different optimum in the bias-variance trade-off. 

\begin{figure}
  \centering  	
    \includegraphics[width=0.65\textwidth]{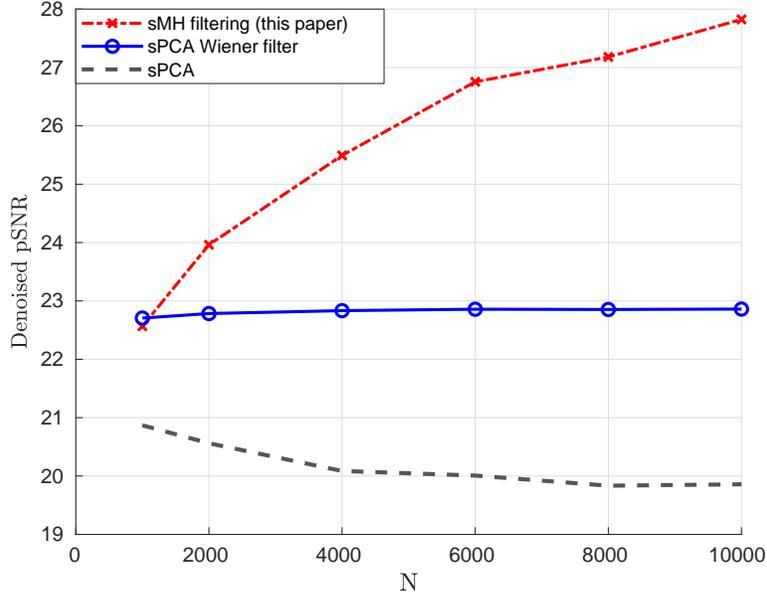}
	\caption[Denoised images] 
	{pSNR of de-noised images for SNR~$=1/10$ and different dataset sizes $N$.}  \label{fig:70S snr -10dB different N denoised}
\end{figure}

Lastly, we evaluated the de-noising quality for $N=10,000$ images and varying amounts of noise. The results are displayed in Figure~\ref{fig:70S N=1E4 different SNR denoised}, where we can see that our method outperforms the sPCA Wiener filter considerably in a wide range of SNRs. We remark that as the SNR decreases the asymptotics considered in Section~\ref{section:Analysis under Gaussian noise} become less valid, thus at some point the steerable graph Laplacian becomes too noisy, and the performance gain of our method drops. This phenomenon is mostly evident for SNRs below $-14$dB, and our method eventually under-performs the sPCA Wiener filter at $-20$dB SNR.

\begin{figure}
  \centering  	
    \includegraphics[width=0.65\textwidth]{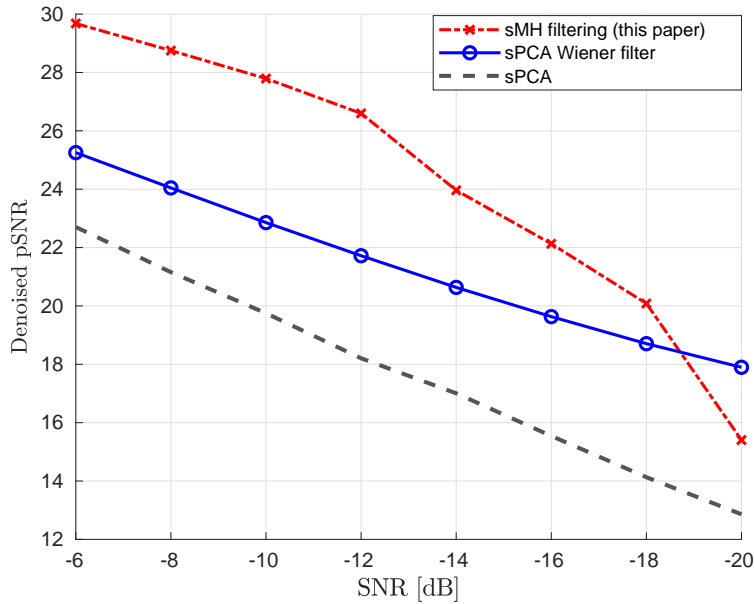}
	\caption[Denoised images] 
	{pSNR of de-noised images for $N=10,000$ and various SNR levels.}  \label{fig:70S N=1E4 different SNR denoised}
\end{figure}

\section{Conclusions and discussion} \label{section:Summary and discussion}
In this work, we introduced the steerable graph Laplacian, which generalizes the standard graph Laplacian by incorporating all planar rotations of all images in the dataset. 
We demonstrated that the (normalized) steerable graph Laplacian is both more accurate and more informative than the standard graph Laplacian, in the sense that it allows for an improved approximation of the Laplace-Beltrami operator on one hand, and admits eigenfunctions with a closed-form analytic expression of their angular part (i.e. the angular Fourier modes) on the other. This closed-form expression is essentially what allows for the efficient filtering procedure of the data coordinates (see Section~\ref{subsection:Expanding image datasets by the steerable manifold harmonics}), as we only need to estimate a block-diagonal coefficients matrix. Then, we have shown that under a suitable modification, the (normalized) steerable graph Laplacian is robust to noise in the regime of high dimensionality due to the concentration of measure of Gaussian noise. Moreover, we have seen that the proposed filtering procedure reduces the noise proportionally to the number of images in the dataset, which was corroborated by the experiments of de-noising cryo-EM projection images, where we demonstrated that our method can provide excellent de-noising results on highly noisy image datasets.

It is interesting to point-out that the steerable graph Laplacian, while utilized for filtering image datasets, can be employed for many other purposes. One application immediately coming to mind is the filtering of datasets consisting of periodic signals (see last remark in Section~\ref{subsection:Problem se-tup}). However, and more importantly, the steerable graph Laplacian can replace the standard graph Laplacian in all applications where the domain is known to be rotationally-invariant (by our definition in Section~\ref{subsection:Problem se-tup}). For instance, it can be used for regularization over general signal/data recovery inverse problems, or for dimensionality reduction in the framework of Diffusion Maps~\cite{coifman2006diffusion} and Laplacian Eigenmaps~\cite{belkin2003laplacian}. In this context, we mention the method of Vector Diffusion Maps (VDM)~\cite{singer2012vector}, which allows for diffusion-based dimensionality reduction for manifold data in the presence of nuisance parameters (such as planar rotations). However, as VDM computes the group ratios only between pairs of data points (e.g. optimal rotational alignments), it may be of interest to compare it to the steerable graph Laplacian, which essentially considers all planar rotations of all images. Lastly, we note that a possible future research direction is the extension of our techniques to other group actions (such as SO(3)).

\section*{Acknowledgements }
We would like to thank Amit Singer for useful comments and suggestions regarding this work.
This research was supported by THE ISRAEL SCIENCE FOUNDATION grant No. 578/14, by Award Number R01GM090200 from the NIGMS, and by the European Research Council (ERC) under the European Union’s Horizon 2020 research and innovation programme (grant agreement 723991 - CRYOMATH).

\begin{appendices}

\section{Choosing $\varepsilon$ and $\lambda_c$ } \label{appendix: Automatic choice of parameters}
In this section, we provide some guidelines on how to choose the parameters $\varepsilon$ and $\lambda_c$, and propose a method for determining them automatically. 

In the clean setting, it is reasonable to choose the optimal value of $\varepsilon$ by minimizing the bias-variance error related to the convergence of the steerable graph Laplacian (i.e. equation~\eqref{eq:steerable graph Laplacian convegence}). However, in general this cannot be achieved without prior knowledge on geometrical quantities of $\man$, such as its curvature (see~\cite{singer2006graph} and the related discussion therein). Nonetheless, methods for automatic picking of $\varepsilon$ for the standard graph Laplacian were proposed~(see for example~\cite{coifman2008graph}). Unfortunately, such methods cannot be applied to noisy data in a straightforward manner, due to the effects of noise on the pairwise distances, and we are not aware of any work detailing the automatic choice of $\varepsilon$ for this case. We note however, that when the distances are computed from the noisy data-points $\left\{y_1,\ldots,y_N \right\}$, and if $N$ is sufficiently large such that the noise is the dominant factor in determining $\varepsilon$, then it is reasonable to choose 
\begin{equation}
\varepsilon \propto \sqrt{\mathcal{D}}\sigma^2,	\label{eq:eps rule of thumb}
\end{equation}
as this expression dominates the standard deviation of the errors in the pairwise (squared) distances (see~\cite{el2010information} and analysis therein), and so can be used to set a region in which it is likely to find the optimal $\varepsilon$.

In order to find the optimal $\varepsilon$ and cut-off frequency $\lambda_c$ for a given dataset, we propose to use a cross-validation procedure. In particular, we choose $\varepsilon$ and $\lambda_c$ such that the log-likelihood of a noisy set of points, given a de-noised set (where the sets are disjoint), is maximized. In more detail, suppose that $\left\{ y_1, \ldots, y_N\right\}$ is a collection of noisy data-points as in~\eqref{eq:noisy data}, i.e.
\begin{equation}
y_i = x_i + \eta_i,
\end{equation}
where $\eta_i\in \CD$ are i.i.d Gaussian vectors with mean zero and covariance $\sigma^2 I$, and $x_i$ are sampled from the manifold $\man$ with uniform distribution. Then, the log-likelihood of obtaining the subset $\left\{ y_{N^{'}+1}, \ldots, y_N\right\}$, where $N^{'}<N$, is given (up to additive constants) by
\begin{equation}
\sum_{i={N^{'}+1}}^N \log \left[ \frac{1}{\operatorname{Vol}\left\lbrace\man\right\rbrace} \int_\man \exp{\left\{-\left\Vert y_i - x\right\Vert^2/2 \sigma^2 \right\}} dx \right].
\end{equation}
We then propose to approximate this log-likelihood by Monte-Carlo integration via the points 
\begin{equation}
\left\{ \hat{x}_{1}(\varepsilon,\lambda_c), \ldots, \hat{x}_{N^{'}}(\varepsilon,\lambda_c)\right\},
\end{equation}
obtained from the de-noising of $\left\{ y_{1}, \ldots, y_{N^{'}}\right\}$ using the parameters $(\varepsilon,\lambda_c)$. That is, we define the empirical log-likelihood of $\left\{ y_{N^{'}+1}, \ldots, y_N\right\}$ (as a function of the parameters $\varepsilon$ and $\lambda_c$) by
\begin{equation}
J(\varepsilon,\lambda_c) = \sum_{i={N^{'}+1}}^N \log{\left[\sum_{j=1}^{N^{'}}\frac{2\pi}{K}\sum_{k=0}^{K-1} \exp{\left\{-\left\Vert y_i - \hat{x}^{2\pi k/K}_j(\varepsilon,\lambda_c)\right\Vert^2/2 \sigma^2 \right\}} \right]},
\end{equation}
where $\hat{x}^{2\pi k/K}_j(\varepsilon,\lambda_c)$ stands for the rotation of $\hat{x}_j(\varepsilon,\lambda_c)$ by an angle of $2\pi k/K$ according to~\eqref{eq:extrinsic rotation}. We mention that summing over the rotations improves the accuracy of the Monte-Carlo integration, as we account for one of the dimensions of the manifold $\man$ in the integration. Finally, we choose the parameters $\varepsilon $ and $\lambda_c$ by maximizing $J$, i.e.
\begin{equation}
\left\{\varepsilon^{opt},\lambda_c^{opt}\right\} = \argmax_{(\varepsilon,\lambda_c)}{\left\{ J(\varepsilon,\lambda_c) \right\}}.
\end{equation}

Essentially, we expect the realizations of the noisy data-points $\left\{ y_{N^{'}+1}, \ldots, y_N\right\}$ to be best explained when the de-noised points $\left\{ \hat{x}_{1}(\varepsilon,\lambda_c), \ldots, \hat{x}_{N^{'}}(\varepsilon,\lambda_c)\right\}$ (and their rotations) lay as close as possible to $\man$, and we seek the best parameters for achieving that. For example, if we take $\lambda_c$ to be too small, then the de-noising error will be dominated by a bias term as we did not take a sufficient number of components to represent the features of the clean images accurately. Therefore we would expect the empirical log-likelihood to be small as the de-noised images will not be close to $\man$. On the other hand, if we take $\lambda_c$ to be too large, then over-fitting will occur, in the sense that some features of the noise will be preserved in the de-noised images. In that case, as the empirical likelihood is computed on a set independent from the set of de-noised points, we would again expect the empirical log-likelihood to be small.

\section{Quadratic form of $L$} \label{appendix:Quadratic form of L}
In what follows, we derive the quadratic form of $L$ appearing in equation~\eqref{eq:L quad form}. First, for $f=\left[f_1(\vartheta),\ldots,f_N(\vartheta)\right]^T$ we have that
\begin{equation}
(Lf)(i,\vartheta) = D_{i,i} f_i(\vartheta) - \sum_{j=1}^N \int_0^{2\pi} W_{i,j}(\vartheta,\varphi)f_j(\varphi) d\varphi,
\end{equation}
and therefore
\begin{equation}
\left\langle f, Lf \right\rangle_{\mathcal{H}} = \sum_{i=1}^N D_{i,i} \int_0^{2\pi} \left\vert f_i(\vartheta) \right\vert^2 d\vartheta - \sum_{i,j=1}^N \int_0^{2\pi} \int_0^{2\pi} f_i^*(\vartheta) W_{i,j}(\vartheta,\varphi) f_j(\varphi) d\vartheta d\varphi.
\end{equation}
Then, if we notice that
\begin{equation}
D_{i,i} = \sum_{j=1}^N \int_0^{2\pi} W_{i,j} (0,\alpha) d\alpha = \sum_{j=1}^N \int_0^{2\pi} W_{i,j} (\vartheta,\varphi) d\varphi
\end{equation}
for every $\vartheta\in [0,2\pi)$ (by changing the integration parameters), we have that
\begin{equation}
\sum_{i=1}^N D_{i,i} \int_0^{2\pi} \left\vert f_i(\vartheta) \right\vert^2 d\vartheta = \sum_{i,j=1}^N \int_0^{2\pi} \int_0^{2\pi} W_{i,j}(\vartheta,\varphi) \left\vert f_i(\vartheta) \right\vert^2 d\vartheta d\varphi = \sum_{i,j=1}^N \int_0^{2\pi} \int_0^{2\pi} W_{i,j}(\vartheta,\varphi) \left\vert f_j(\varphi) \right\vert^2 d\vartheta d\varphi,	\label{eq:D_ii sum 1}
\end{equation}
due to the symmetry $W_{i,j}(\vartheta,\varphi) = W_{j,i}(\varphi,\vartheta)$ (see the definition of $W$ in~\eqref{eq:steerable affinity matrix}).
Finally, using~\eqref{eq:D_ii sum 1} we can write
\begin{align}
\left\langle f, Lf \right\rangle_{\mathcal{H}} &= \frac{1}{2} \sum_{i,j=1}^N \int_0^{2\pi} \int_0^{2\pi} W_{i,j}(\vartheta,\varphi) \left[ \left\vert f_i(\vartheta) \right\vert^2 + \left\vert f_j(\varphi) \right\vert^2 -f_i(\vartheta) f_j^*(\varphi) -f_i^*(\vartheta) f_j(\varphi)\right]   d\vartheta d\varphi \nonumber \\
&= \frac{1}{2} \sum_{i,j=1}^N \int_0^{2\pi} \int_0^{2\pi} W_{i,j}(\vartheta,\varphi) \left\vert f_i(\vartheta) - f_j(\varphi) \right\vert^2   d\vartheta d\varphi.
\end{align}

\section{Linear rotationally-invariant operators} \label{appendix: LRI operators}
We start with the following definition of linear rotationally-invariant operators over $\mathcal{H}$.
\begin{defn} [\textbf{LRI operators}] \label{def:LRI operators}
An operator $G:\mathcal{H}\rightarrow\mathcal{H}$ is linear rotationally-invariant (LRI) over $\mathcal{H}$, if
\begin{enumerate}
\item For any fixed $(i,\vartheta)\in \Gamma$ (where $\Gamma$ defined in Section~\ref{subsection:The steerable graph Laplacian for image-manifolds}), the functional $\left\{ Gf\right\}(i,\vartheta)$ is linear and continuous in $f$. 
\item $G$ satisfies
\begin{equation}
\left\{ Gf\right\}(i,\vartheta-\alpha) = \left\{ Gf^\alpha\right\}(i,\vartheta), \quad\quad f^\alpha(i,\vartheta) \triangleq f_i(\vartheta-\alpha), \label{eq:operator rotatioal invariance}
\end{equation}    
for all $f\in\mathcal{H}$, $\alpha\in [0,2\pi)$, and $(i,\vartheta)\in\Gamma$.
\end{enumerate}
\end{defn}
In the first requirement of Definition~\ref{def:LRI operators}, the continuity property essentially means that if $f_1$ and $f_2$ are close (in $\mathcal{H}$), then $Gf_1(i,\vartheta)$ and $Gf_2(i,\vartheta)$ are also close (in absolute value). As for the second requirement (rotational-invariance), loosely speaking, it means that shifting the output of the operator cyclically by an angle $\alpha$ is equivalent to shifting the input by $\alpha$, hence the action of the operator itself does not depend on the angle $\vartheta$ (of $(i,\vartheta)\in\Gamma$). 
We mention that this property of linearity and rotational-invariance can be viewed as analogous to that of Linear-Time Invariant (LTI) operators, native to classical signal processing.
\begin{remark}
Our definition of LRI operators is somewhat more restrictive than the name suggests (compared also to classical LTI operators) because of our requirement for continuity of every functional $\left\{ G\cdot\right\}(i,\vartheta)$. We note that while this requirement can be removed, allowing for a broader class of operators, it is simpler to handle and sufficient for our purposes.
\end{remark}

The next lemma characterizes the form of LRI operators explicitly.
\begin{lem} [\textbf{Explicit form of LRI operators}] \label{lem:LRI operator explicit form}
Let $G$ be an LRI operator over $\mathcal{H}$. Then, there exist unique $\left\{G_{i,j}\right\}_{i,j=1}^N\in\mathcal{L}^2(\mathbb{S}^1)$, s.t. for any $f\in\mathcal{H}$ and $(i,\vartheta)\in \Gamma$
\begin{equation}
\left\{ Gf\right\}(i,\vartheta) 
= \sum_{j=1}^N \int_0^{2\pi} G_{i,j}(\varphi-\vartheta) f_j(\varphi) d\varphi. \label{eq:LRI operator explicit form}
\end{equation}
\end{lem}
\begin{proof}
By the Riesz representation theorem, if $\left\{ G\cdot\right\}(i,\vartheta)$ is a linear and continuous functional over $\mathcal{H}$, then there exists a unique $g_{(i,\vartheta)}\in\mathcal{H}$ such that
\begin{equation}
\left\{ Gf\right\}(i,\vartheta)
= \left\langle g_{(i,\vartheta)},f\right\rangle_{\mathcal{H}} =\sum_{j=1}^N \int_0^{2\pi} g^*_{i,j}(\vartheta,\varphi) f_j(\varphi) d\varphi,
\end{equation}
where $g_{(i,\vartheta)} \triangleq \left[ g_{i,1}(\vartheta,\cdot),\ldots, g_{i,N}(\vartheta,\cdot)\right]^T$.
Additionally, from~\eqref{eq:operator rotatioal invariance}, we have that
\begin{equation}
\left\{ Gf\right\}(i,\vartheta) = \left\{ Gf\right\}(i,\vartheta+\alpha-\alpha) = \left\{ Gf^\alpha\right\}(i,\vartheta+\alpha) = \sum_{j=1}^N \int_0^{2\pi} g^*_{i,j}(\vartheta+\alpha,\varphi^{'}+\alpha) f_j(\varphi^{'}) \varphi^{'},
\end{equation}
when changing the integration parameter via $\varphi^{'}= \varphi-\alpha$. Lastly, taking $\alpha=-\vartheta$ and defining $G_{i,j}(\varphi-\vartheta) = g^*_{i,j}(0,\varphi-\vartheta)$ concludes the proof.
\end{proof}

The main contribution of Lemma~\ref{lem:LRI operator explicit form} is to point out that every LRI operator can be characterized by a finite number of functions $\left\{G_{i,j}\right\}_{i,j=1}^N$ which can be expanded in a Fourier series. Therefore, $G$ can be mapped to (and described by) a sequence of matrices $\left\{ \hat{G}^{(m)}\right\}_{m=-\infty}^\infty$ defined by
\begin{equation}
\hat{G}^{(m)}_{i,j} = \int_0^{2\pi} G_{i,j}(\alpha) e^{ \imath m \alpha} d\alpha. \label{eq:G fourier transform}
\end{equation}

It is important to notice that from~\eqref{eq:extrinsic rotation} and~\eqref{eq:steerable affinity matrix}, it immediately follows that the steerable affinity operator $W$ of~\eqref{eq:steerable affinity matrix} is LRI with $G_{i,j} = W_{i,j}(0,\vartheta-\varphi)$. However, we note that for any fixed $(i,\vartheta)\in\Gamma$, it is evident that $Df(i,\vartheta)$ (where $D$ is the diagonal matrix defined in~\eqref{eq:steerable graph Laplacian}) is not a continuous functional of $f$ (in $\mathcal{H}$), as small perturbations in $f$ may lead to arbitrarily large changes in $Df(i,\vartheta)$ as it depends on point-wise values of $f$.  Therefore, the steerable graph Laplacian $L$ (from~\eqref{eq:steerable graph Laplacian}) is not LRI.
Nonetheless, as we shall see next, we can still employ the properties of LRI operators to characterize a broader family of operators, which includes the steerable affinity operator $W$ and the steerable graph Laplacian $L$ as special cases, with eigenfunctions admitting a particularly convenient form. This is the subject of the next proposition.

\begin{prop} \label{prop:eigenfunctions and eigenvalues of H}
Consider an operator $H:\mathcal{H}\rightarrow\mathcal{H}$ of the form
\begin{equation}
H = A + G,	\label{eq:generalized LRI operator}
\end{equation}
where $G$ is LRI and $A\in\mathbb{C}^{N\times N}$ is a complex-valued matrix. If $(\lambda,v)$ is an eigenvalue-eigenvector pair of the matrix 
\begin{equation}
\hat{H}^{(m)} = A+\hat{G}^{(m)},
\end{equation}
where $\hat{G}^{(m)}$ is from~\eqref{eq:G fourier transform}, then $\Phi = v \cdot e^{\imath m \vartheta}$ is an eigenfunction of $H$ with eigenvalue $\lambda$.
\end{prop}
\begin{proof}
The proof follows directly from Lemma~\ref{lem:LRI operator explicit form} and the Fourier expansion of $G_{i,j}$.
Let us write
\begin{equation}
\left\{ H \Phi \right\}(i,\vartheta) = \sum_{j=1}^N A_{i,j} \Phi(j,\vartheta) + \sum_{j=1}^N \int_0^{2\pi} G_{i,j}(\varphi-\vartheta) \Phi(j,\varphi) d\varphi,
\end{equation}
where we have used the explicit form of the LRI operator $G$ given by~\eqref{eq:LRI operator explicit form}.
Then, if we expand $G_{i,j}(\cdot)$ in a Fourier series as
\begin{equation}
G_{i,j}(\varphi-\vartheta) = \frac{1}{2\pi}\sum_{m=-\infty}^{\infty} \hat{G}_{i,j}^{(m)} e^{ -\imath m (\varphi-\vartheta)}, \quad\quad  \hat{G}^{(m)}_{i,j} = \int_0^{2\pi} G_{i,j}(\alpha) e^{ \imath m \alpha} d\alpha,
\end{equation} 
we have that
\begin{equation}
\left\{ H \Phi \right\}(i,\vartheta) = \sum_{j=1}^N A_{i,j} \Phi(j,\vartheta) + \sum_{j=1}^N \frac{1}{2\pi}\sum_{m^{'}=-\infty}^\infty \hat{G}_{i,j}^{m^{'}} e^{\imath m^{'} \vartheta} \int_0^{2\pi} \Phi(j,\varphi) e^{-\imath m^{'} \varphi} d\varphi.
\end{equation}
Therefore, by substituting $\Phi(i,\vartheta) = v_{i}e^{\imath m \vartheta}$, where $v_{i}$ stands for $i$'th element of $v$, we get
\begin{equation}
\left\{ H \Phi \right\}(i,\vartheta) = e^{\imath m \vartheta}\sum_{j=1}^N A_{i,j} v_{j} + e^{\imath m \vartheta} \sum_{j=1}^N \hat{G}_{i,j}^{m} v_{j},
\end{equation}
where we have used the orthogonality of $\left\{ e^{\imath m \vartheta}\right\}_{m=-\infty}^\infty$ over $[0,2\pi)$. Finally, it follows that 
\begin{equation}
 H \Phi  = \left[(A+\hat{G}^{(m)}) v\right] e^{\imath m \vartheta} = \lambda v e^{\imath m \vartheta} = \lambda \Phi,
\end{equation}
since $v$ is an eigenvector of $A+\hat{G}^{(m)}$ with eigenvalue $\lambda$.
\end{proof}

Therefore, even though the steerable graph Laplacian $L$ is not strictly LRI according to Definition~\ref{def:LRI operators}, it still takes the form of the operators considered in Proposition~\ref{prop:eigenfunctions and eigenvalues of H}, and consequently, we can derive its eigen-decomposition by making use of the sequence of matrices $\left\{\hat{W}^{(m)}\right\}_{m=-\infty}^\infty$ of~\eqref{eq:w_hat_ij mat fourier}. 

\section{Proof of Theorem~\ref{thm:eigenfunctions and eigenvalues of L}}
\label{appendix:proof of spectral properties of L}
\begin{proof}
First, we note that by~\eqref{eq:S_m matrix def} and~\eqref{eq:steerable affinity matrix}, it follows directly that $\hat{W}^{(m)}$ is Hermitian, which implies that $S_m$ is also Hermitian, and therefore can be diagonalized by a set of orthogonal eigenvectors.
Next, as $L$ is of the form $A+G$ as required by Proposition~\ref{prop:eigenfunctions and eigenvalues of H} (where $A$ is a matrix and $G$ is an LRI operator), we can obtain a sequence of eigenfunctions and eigenvalues of $L$ by diagonalizing the matrices $S_m = D-\hat{W}^{(m)}$ for every $m\in\mathbb{Z}$. Then, the eigenvalues $\left\{\lambda_{m,k}\right\}_{k=1}^N$ must be real-valued (since $S_m$ is Hermitian), and moreover, by the quadratic form of $L$~\eqref{eq:L quad form}
it follows that $\left\langle f, Lf \right\rangle_{\mathcal{H}} \geq 0$, which implies that $L$ is semi-positive definite and thus the eigenvalues $\left\{\lambda_{m,k}\right\}_{k=1}^N$ are non-negative.

Lastly, the fact that $\left\{\Phi_{m,k}\right\}_{m,k}$ are orthogonal and complete follows from the orthogonality and completeness of $\left\{ e^{\imath m \vartheta}\right\}_{m=-\infty}^\infty$ over $\mathcal{L}^2(\mathbb{S}^1),$ and the orthogonality and completeness of $\left\{v_{m,k}\right\}_{k=1}^N$ over $\mathbb{C}^N$ (since $D-\hat{W}_m$ is Hermitian) for every $m\in\mathbb{Z}$. In particular, it easily follows that we can expand every $f\in \mathcal{H}$ as
\begin{equation}
f(i,\vartheta) = \sum_{m=\infty}^\infty \alpha_m^i e^{\imath m \vartheta} = \sum_{m=\infty}^\infty \sum_{j=1}^N \beta_{m,j} v_{i,(m,j)} e^{\imath m \vartheta},
\end{equation}
where $v_{i,(m,j)}$ stands for the $i$'th element of the vector $v_{m,j}$ (which is the $j$'th eigenvector of $S_m = D-\hat{W}^{(m)}$), which concludes the proof.
\end{proof}

\section{Proof of Theorem~\ref{thm:steerable graph Laplacian convergence}} \label{appendix:proof of steerable graph Laplacian convergence}
\subsection{The limit and bias terms}\label{subsection:The limit and bias terms}
By~\eqref{eq:normalized steerable graph laplacian} and~\eqref{eq:steerable graph Laplacian}, we can write
\begin{align}
\frac{4}{\varepsilon} \left\{\tilde{L}g\right\}(i,\vartheta) &= \frac{4}{\varepsilon}\left[f(x_i^\vartheta) - \sum_{j=1}^N \int_0^{2\pi} D^{-1}_{i,i}{W}_{i,j}(\vartheta,\varphi)f(x_j^\varphi)d\varphi \right] \nonumber \\
 &= \frac{4}{\varepsilon}\left[f(x_i^\vartheta) -  \frac{\frac{1}{N}\sum_{j=1}^N \int_0^{2\pi} \exp{\left\lbrace-{\left\Vert x_i^\vartheta - x_j^\varphi\right\Vert^2}{/\varepsilon}\right\rbrace}f(x_j^\varphi) d\varphi}{\frac{1}{N} \sum_{j=1}^{N}  \int_0^{2\pi} \exp{\left\lbrace-{\left\Vert x_i^\vartheta -x_j^{\varphi}\right\Vert^2}{/\varepsilon}\right\rbrace} d\varphi}\right]. \label{eq:steerable graph laplacian fraction}
\end{align}
We begin by deriving the limit of~\eqref{eq:steerable graph laplacian fraction} for $N\rightarrow\infty$ and a fixed $\varepsilon>0$, showing that it is essentially the Laplace-Beltrami operator $\Delta_\man$ with an additional bias error term of $O(\varepsilon)$. 
First, let us focus our attention on the expression
\begin{equation}
C_{i,N}^1(\vartheta) \triangleq \frac{1}{N}\sum_{j=1}^N \int_0^{2\pi} \exp{\left\lbrace-{\left\Vert x_i^\vartheta - x_j^\varphi\right\Vert^2}{/\varepsilon}\right\rbrace}f(x_j^\varphi) d\varphi, \label{eq:steerable graph laplacian numerator}
\end{equation}
which is the numerator of the second term of~\eqref{eq:steerable graph laplacian fraction} (inside the brackets).
Before we proceed with the evaluation of the expression in~\eqref{eq:steerable graph laplacian numerator}, we construct a convenient parametrization of the manifold $\man$. To this end, since our manifold $\man$ is rotationally-invariant, it would be beneficial to parametrize it by a rotationally-invariant coordinate $z$ coupled with a rotation angle $\beta\in [0,2\pi)$ (analogously to polar coordinates in $\mathbb{R}^2$). In particular, in Section~\ref{subsection:construction of a rotationally-invariant parametrizaion} we construct a parametrization $x\mapsto (z,\beta)$ for every $x\in \man^{'}$, where $\man^{'}\subset \man$ is a certain smooth neighbourhood of $x_i^\vartheta$ (defined explicitly in Section~\ref{subsection:construction of a rotationally-invariant parametrizaion}), such that 
\begin{equation}
x = \mathcal{R}(z,\beta) = z^\beta, \quad\quad\quad z\in\mathcal{N}, \quad\quad\quad \beta\in [0,2\pi),
\end{equation}
and $\mathcal{N}\subset\man^{'}$ is a smooth $(d-1)$-dimensional submanifold.

Next, let us continue with the evaluation of~\eqref{eq:steerable graph laplacian numerator}, and define
\begin{equation}
H_i^\vartheta(x) \triangleq \int_0^{2\pi} \exp{\left\lbrace-{\left\Vert x_i^\vartheta - x^{\varphi}\right\Vert^2}{/\varepsilon}\right\rbrace}f(x^{\varphi})d\varphi. \label{eq:H def}
\end{equation}
Now, using the rotationally-invariant parametrization $x\mapsto (z,\beta)$ for every $x\in \man^{'}$, we can write
\begin{equation}
x^\varphi = \mathcal{R}(x,\varphi) = \mathcal{R}(z^{\beta},\varphi) = z^{\varphi+\beta},
\end{equation}
and thus 
\begin{equation}
H_i^\vartheta(x) = \int_0^{2\pi} \exp{\left\lbrace-{\left\Vert x_i^\vartheta - z^{\varphi+\beta}\right\Vert^2}{/\varepsilon}\right\rbrace}f(z^{\varphi+\beta})d\varphi 
= \int_0^{2\pi} \exp{\left\lbrace-{\left\Vert x_i^\vartheta - z^{\varphi}\right\Vert^2}{/\varepsilon}\right\rbrace}f(z^{\varphi})d\varphi = H_i^\vartheta(z), \label{eq:H_i_theta invariant to beta}
\end{equation}
establishing that $H_i^\vartheta$ is only a function of $z\in\mathcal{N}$ for all points $x\in\man^{'}$. 
Now, since $\left\lbrace x_i \right\rbrace$ are i.i.d samples from $\man$, then by the law of large numbers 
\begin{align}
&\lim_{N\rightarrow\infty} C_{i,N}^1(\vartheta) = \lim_{N\rightarrow\infty} \frac{1}{N}\sum_{j=1}^N H_i^\vartheta(x_j) = \lim_{N\rightarrow\infty} \frac{1}{N}\sum_{j\neq i,j=1}^N H_i^\vartheta(x_j) = \mathbb{E}{\left[ H_i^\vartheta(x) \right]} = \int_{\man} H_i^\vartheta(x) p(x) dx. \label{eq:C1 lim}
\end{align}
By our construction of the rotationally-invariant parametrization, and in particular the set $\man^{'}$ (see Section~\eqref{subsection:construction of a rotationally-invariant parametrizaion}), we have that $\left\Vert x_i^\vartheta - x\right\Vert^2 > (\delta^{'})^2$ for all $x \notin \man^{'}$ and some positive constant $\delta^{'}>0$. Hence
\begin{equation}
\int_{\man} H_i^\vartheta(x) p(x) dx = \int_{\man^{'}} H_i^\vartheta(x) p(x) dx + O\left(e^{-(\delta^{'})^2 /\varepsilon}\right), \label{eq:H_i_theta int replacing M with M_prime}
\end{equation}
and we mention that the exponential term $O\left(e^{-(\delta^{'})^2 /\varepsilon}\right)$ is negligible w.r.t to any polynomial asymptotic expansion in $\varepsilon$, and is therefore omitted in subsequent analysis.

Continuing, we are interested in changing the integration variable from $x$ to $(z,\beta)$, as considered by the following lemma.
\begin{lem} [Integration and volume form on $\man^{'}$] \label{lem:rotation invariant measure}
For any smooth $h:\man^{'}\rightarrow \mathbb{R}$, we have that
\begin{equation} \label{eq:manifold polar parametrization}
\int_{\man^{'}} h(x) dx = \int_{z\in\mathcal{N}}\int_{\beta=0}^{2\pi} h(z^\beta) V(z^\beta) dz d\beta,
\end{equation}
where $V(x)$ is associated with the volume form of $\man^{'}$ at $x$ when integrating by $(z,\beta)$, and is invariant to $\beta$, i.e.
\begin{equation}
V(z^\beta) = V(z) \label{eq:volume form invariance}
\end{equation}
for all $\beta\in [0,2\pi)$.
\end{lem}
The proof of Lemma~\ref{lem:rotation invariant measure} is provided in Section~\ref{subsection:angle invariant measure lemma proof }.
Hence, By Lemma~\ref{lem:rotation invariant measure}, equations~\eqref{eq:H_i_theta invariant to beta}~--~\eqref{eq:H_i_theta int replacing M with M_prime}, and the fact that $p(x)={1}/{\operatorname{Vol}\left\lbrace\man\right\rbrace}$, we have
\begin{equation}
\lim_{N\rightarrow\infty} C_{i,N}^1(\vartheta) = \int_{\mathcal{N}} \int_0^{2\pi} H_i^\vartheta(z) V(z) p(z^\beta) dz d\beta = \frac{2\pi}{\operatorname{Vol}\left\lbrace\man\right\rbrace} \int_{\mathcal{N}} H_i^\vartheta(z) V(z) dz. \label{eq:C1 lim in z}
\end{equation}
Then, by substituting~\eqref{eq:H def} into~\eqref{eq:C1 lim in z} we obtain
\begin{equation}
\lim_{N\rightarrow\infty} C_{i,N}^1(\vartheta) = \frac{2\pi}{\operatorname{Vol}\left\lbrace\man\right\rbrace} \int_{\mathcal{N}} \int_0^{2\pi} \exp{\left\lbrace-{\left\Vert x_i^\vartheta - z^\varphi\right\Vert^2}{/\varepsilon}\right\rbrace}f(z^\varphi) V(z) dz d\varphi. \label{eq:steerable graph laplacian lim integral chg var}
\end{equation}
Eventually, if we change parametrization from $(z,\varphi)$ back to $x$, via $x=z^\varphi$, then by Lemma~\ref{lem:rotation invariant measure} we arrive at
\begin{align}
\lim_{N\rightarrow\infty} C_{i,N}^1(\vartheta) &= \frac{2\pi}{\operatorname{Vol}\left\lbrace\man\right\rbrace} \int_{\man^{'}}  \exp{\left\lbrace-{\left\Vert x_i^\vartheta - x \right\Vert^2}{/\varepsilon}\right\rbrace}f(x) dx \nonumber \\
&= \frac{2\pi}{\operatorname{Vol}\left\lbrace\man\right\rbrace} \int_{\man}  \exp{\left\lbrace-{\left\Vert x_i^\vartheta - x \right\Vert^2}{/\varepsilon}\right\rbrace}f(x) dx + O\left(e^{-(\delta^{'})^2 /\varepsilon}\right) \nonumber \\
& = \frac{2\pi}{\operatorname{Vol}\left\lbrace\man\right\rbrace} \int_{\man}  \exp{\left\lbrace-{\left\Vert x_i^\vartheta - x \right\Vert^2}{/\varepsilon}\right\rbrace}f(x) dx, \label{eq:num lim}
\end{align}
where we again used the fact that $\left\Vert x_i^\vartheta - x\right\Vert^2 > (\delta^{'})^2$ for all $x \notin \man^{'}$ (see Section~\eqref{subsection:construction of a rotationally-invariant parametrizaion}), and omitted the resulting $ O\left(e^{-(\delta^{'})^2 /\varepsilon}\right)$ term.

In a similar fashion, if we consider the denominator of the second term in~\eqref{eq:steerable graph laplacian fraction}
\begin{equation}
C_{i,N}^2(\vartheta) \triangleq \frac{1}{N}\sum_{j=1}^N \int_0^{2\pi} \exp{\left\lbrace-{\left\Vert x_i^\vartheta - x_j^\varphi\right\Vert^2}{/\varepsilon}\right\rbrace} d\varphi, \label{eq:C2 def}
\end{equation}
and by repeating the calculations for $C_{i,N}^1(\vartheta)$ with $f\equiv 1$, we get that
\begin{equation}
\lim_{N\rightarrow\infty} C_{i,N}^2(\vartheta) = \frac{2\pi}{\operatorname{Vol}\left\lbrace\man\right\rbrace} \int_{\man}  \exp{\left\lbrace-{\left\Vert x_i^\vartheta - x \right\Vert^2}{/\varepsilon}\right\rbrace} dx. \label{eq:denom lim}
\end{equation}

Lastly, if we substitute~\eqref{eq:num lim} and~\eqref{eq:denom lim} into~\eqref{eq:steerable graph laplacian fraction}, we have that
\begin{align}
\lim_{N\rightarrow\infty} \frac{4}{\varepsilon} \left\{\tilde{L}g\right\}(i,\vartheta) &=  \frac{4}{\varepsilon}\left[f(x_i^{\vartheta}) -  \frac{\frac{1}{\operatorname{Vol}\left\lbrace\man\right\rbrace} \int_{\man}  \exp{\left\lbrace-{\left\Vert x_i^{\vartheta} - x\right\Vert^2}{/\varepsilon}\right\rbrace}f(x) dx}{\frac{1}{\operatorname{Vol}\left\lbrace\man\right\rbrace} \int_{\man}  \exp{\left\lbrace-{\left\Vert x_i^{\vartheta} - x\right\Vert^2}{/\varepsilon}\right\rbrace} dx}\right] \label{eq:fraction limit}
\\ &= \Delta_\man f(x_i^\vartheta) + O(\varepsilon), \label{eq:bias error in proof}
\end{align}
where the last simplification (eq.~\eqref{eq:bias error in proof}) is justified in~\cite{singer2006graph}.

\subsection{The variance term} \label{subsection:the variance term}
The variance error term in the convergence of the steerable graph Laplacian to the Laplace Beltrami operator arises from the discrepancy between the values of $C_{i,N}^1(\vartheta)$, $C_{i,N}^2(\vartheta)$ for finite $N$, and their limits~\eqref{eq:num lim},~\eqref{eq:denom lim}, respectively, when $N\rightarrow\infty$. 
To prove the improved convergence rate of steerable graph Laplacian, we follow the technique used in~\cite{singer2006graph} which makes use of the Chernoff tail inequality. Such an inequality provides a bound for the probability that a sum of random variables deviates from its mean by a certain quantity. 

Let us begin by defining
\begin{equation}
G_i^\vartheta(x) \triangleq \int_0^{2\pi} \exp{\left\lbrace-{\left\Vert x_i^\vartheta - x^\varphi \right\Vert^2}{/\varepsilon}\right\rbrace}d\varphi. \label{eq:G def}
\end{equation}
Then, we are interested in evaluating the probabilities
\begin{align}
p_+(N,\alpha) &= Pr\left\lbrace \frac{\sum_{j\neq i}^N H_i^\vartheta(x_j)}{\sum_{j\neq i}^N G_i^\vartheta(x_j)}  - \frac{\mathbb{E}\left[H_i^\vartheta\right]}{\mathbb{E}\left[G_i^\vartheta\right]} > \alpha\right\rbrace, \\ \quad p_-(N,\alpha)  &= 
Pr\left\lbrace \frac{\sum_{j\neq i}^N H_i^\vartheta(x_j)}{\sum_{j\neq i}^N G_i^\vartheta(x_j)}  - \frac{\mathbb{E}\left[H_i^\vartheta\right]}{\mathbb{E}\left[G_i^\vartheta\right]} < -\alpha\right\rbrace,
\end{align}
where we mention that the use of ${\sum_{j\neq i}^N H_i^\vartheta(x_j)}/{\sum_{j\neq i}^N G_i^\vartheta(x_j)}$ instead of ${\sum_{j=1}^N H_i^\vartheta(x_j)}/{\sum_{j=1}^N G_i^\vartheta(x_j)}$ (i.e. without the diagonal) is justified at the end of the proof.
We proceed by evaluating $p_+(N,\alpha)$, as $p_-(N,\alpha)$ can be obtained in a similar fashion.
As was shown in~\cite{singer2006graph}, $p_+(N,\alpha)$ can be rewritten as 
\begin{equation}
p_+(N,\alpha) = Pr\left\lbrace \sum_{j\neq i}^N J_i^\vartheta(x_j) > \alpha(N-1)\left(\mathbb{E}\left[G_i^\vartheta\right]\right)^2\right\rbrace,
\end{equation}
where $J_i^\vartheta(x_j)$ are zero-mean i.i.d random variables (indexed by $j$), given by
\begin{equation}
J_i^\vartheta(x_j) \triangleq \mathbb{E}\left[G_i^\vartheta\right] H_i^\vartheta(x_j) - \mathbb{E}\left[H_i^\vartheta\right] G_i^\vartheta(x_j) + \alpha\mathbb{E}\left[G_i^\vartheta\right]\left(\mathbb{E}\left[G_i^\vartheta\right] - G_i^\vartheta(x_j) \right).
\end{equation}

At this point, making use of Chernoff's inequality gives
\begin{equation}
p_+(N,\alpha) \leq \exp{\left\{ -\frac{\alpha^2(N-1)^2\left(\mathbb{E}\left[G_i^\vartheta\right]\right)^4}{2 (N-1) \mathbb{E}\left[\left(J_i^\vartheta \right)^2\right] + O(\alpha)}\right\}},
\end{equation}
and it remains to evaluate the variance term $\mathbb{E}\left[\left(J_i^\vartheta \right)^2\right]$, which can be expressed as
\begin{equation}
\mathbb{E}\left[\left(J_i^\vartheta \right)^2\right] = \left(\mathbb{E}\left[G_i^\vartheta\right]\right)^2 \mathbb{E}\left[\left(H_i^\vartheta \right)^2\right]
 - 2\mathbb{E}\left[G_i^\vartheta\right]\mathbb{E}\left[H_i^\vartheta\right]\mathbb{E}\left[H_i^\vartheta G_i^\vartheta\right]
  + \left(\mathbb{E}\left[H_i^\vartheta\right]\right)^2\mathbb{E}\left[\left(G_i^\vartheta \right)^2\right] + O(\alpha). \label{eq:J_i_theta asym}
\end{equation}
Now, the integral expressions of $\mathbb{E}\left[H_i^\vartheta\right]$ and $\mathbb{E}\left[G_i^\vartheta\right]$ (obtained in the previous section in equations~\eqref{eq:num lim} and~\eqref{eq:denom lim}, respectively), admit asymptotic expansions via the following proposition.
\begin{prop}{\cite{singer2006graph}} \label{prop:asym analysis}
Let $\widetilde{\man}$ be a smooth and compact $\tilde{d}$-dimensional submanifold, and let $\tilde{f}:\widetilde{\man}\rightarrow \mathbb{R}$ be a smooth function. Then, for any $y\in\widetilde{\man}$
\begin{equation}
\left(\pi \varepsilon\right)^{-\tilde{d}/2} \int_{\widetilde{\man}} \exp{\left\lbrace-{\left\Vert  y - x\right\Vert^2}{/\varepsilon}\right\rbrace} \tilde{f}(x)dx = \tilde{f}(y) + \frac{\varepsilon}{4}\left[ E(y)\tilde{f}(y) + \Delta_{\widetilde{\man}}\tilde{f}(y)\right] + O(\varepsilon^2), \label{eq:asym analysis general int formula} 
\end{equation}
where $E(y)$ is a scalar function of the curvature of $\widetilde{\man}$ at $y$.
\end{prop} 
Then, it follows from Proposition~\ref{prop:asym analysis} (see also~\cite{coifman2006diffusion,belkin2008towards}) that
\begin{align}
\mathbb{E}\left[H_i^\vartheta\right] &= \frac{2\pi}{\operatorname{Vol}\left\lbrace\man\right\rbrace} \int_{\man}  \exp{\left\lbrace-{\left\Vert x_i^\vartheta - x \right\Vert^2}{/\varepsilon}\right\rbrace}f(x) dx = 
\frac{2\pi}{\operatorname{Vol}(\man)} (\pi\varepsilon)^{d/2} \left[ f(x_i^\vartheta) + O(\varepsilon) \right], \label{eq:H_i_theta asym} \\ 
\mathbb{E}\left[G_i^\vartheta\right] &= \frac{2\pi}{\operatorname{Vol}\left\lbrace\man\right\rbrace} \int_{\man}  \exp{\left\lbrace-{\left\Vert x_i^\vartheta - x \right\Vert^2}{/\varepsilon}\right\rbrace} dx = \frac{2\pi}{\operatorname{Vol}(\man)} (\pi\varepsilon)^{d/2} \left[1 + O(\varepsilon)\right]. \label{eq:G_i_theta asym}
\end{align}
Thus, it remains to evaluate the second order moments $\mathbb{E}\left[\left(H_i^\vartheta \right)^2\right]$,$\mathbb{E}\left[\left(G_i^\vartheta \right)^2\right]$, and $\mathbb{E}\left[H_i^\vartheta G_i^\vartheta\right]$, which is the subject of the next lemma, whose proof is given in Section~\ref{subsection:proof of lemma second order moments asymptotics}.
\begin{lem} \label{lem:second order moments asymptotics}
Let $\mu(x)\triangleq\sqrt{\sum_{m=-M}^M \sum_{\ell=1}^{\ell_m} m^2 \left\vert x_{m,\ell}^2 \right\vert^2} >0$ for every $x\in\mathcal{N}$. Then
\begin{align}
\mathbb{E}\left[\left(H_i^\vartheta \right)^2\right] &= \frac{\left( \pi \varepsilon\right)^{(d+1)/2}}{2^{(d-1)/2} } \left[\frac{f^2(x_i^{\vartheta}) p_{\mathcal{N}}(x_i^{\vartheta})}{\mu^2(x_i^{\vartheta})} +O(\varepsilon) \right], \\
\mathbb{E}\left[\left(G_i^\vartheta \right)^2\right] &= \frac{\left( \pi \varepsilon\right)^{(d+1)/2}}{2^{(d-1)/2} } \left[\frac{p_{\mathcal{N}}(x_i^{\vartheta})}{\mu^2(x_i^{\vartheta})} +O(\varepsilon) \right], \\
\mathbb{E}\left[H_i^\vartheta G_i^\vartheta\right] &= \frac{\left( \pi \varepsilon\right)^{(d+1)/2}}{2^{(d-1)/2} } \left[\frac{f(x_i^{\vartheta}) p_{\mathcal{N}}(x_i^{\vartheta})}{\mu^2(x_i^{\vartheta})} +O(\varepsilon) \right],
\end{align} 
where $p_{\mathcal{N}}(x) = 2\pi V(x)/\operatorname{Vol}\left\{\man\right\}$ (and $V(x)$ is from Lemma~\ref{lem:rotation invariant measure}).
\end{lem}
We mention that since we required (in Theorem~\ref{thm:steerable graph Laplacian convergence}) that $\sum_{m\neq 0}\sum_{\ell=1}^{\ell_m} \left\vert x_{m,\ell}\right\vert^2 >0 $ for all $x\in\man$ (up to a set of measure zero on $\man$), then it is also the case that $\mu(x)>0$ for all $x\in\mathcal{N}$ with probability one.

Now, using Lemma~\ref{lem:second order moments asymptotics},~\eqref{eq:H_i_theta asym},~\eqref{eq:G_i_theta asym}, and substituting all quantities into~\eqref{eq:J_i_theta asym}, we get that
\begin{align}
\mathbb{E}\left[\left(J_i^\vartheta \right)^2\right] &= \left(\frac{2\pi}{\operatorname{Vol}(\man)}\right)^2 \left( \pi \varepsilon\right)^d \frac{\left( \pi \varepsilon\right)^{(d+1)/2}}{2^{(d-1)/2} } \left[\frac{f^2(x_i^{\vartheta}) p_{\mathcal{N}}(x_i^{\vartheta})}{\mu^2(x_i^{\vartheta})} +O(\varepsilon) \right] \nonumber \\
 &-2 \left(\frac{2\pi}{\operatorname{Vol}(\man)}\right)^2 \left( \pi \varepsilon\right)^d \frac{\left( \pi \varepsilon\right)^{(d+1)/2}}{2^{(d-1)/2} } \left[\frac{f^2(x_i^{\vartheta}) p_{\mathcal{N}}(x_i^{\vartheta})}{\mu^2(x_i^{\vartheta})} +O(\varepsilon) \right] \nonumber \\
 &+ \left(\frac{2\pi}{\operatorname{Vol}(\man)}\right)^2 \left( \pi \varepsilon\right)^d \frac{\left( \pi \varepsilon\right)^{(d+1)/2}}{2^{(d-1)/2} } \left[\frac{f^2(x_i^{\vartheta}) p_{\mathcal{N}}(x_i^{\vartheta})}{\mu^2(x_i^{\vartheta})} +O(\varepsilon) \right]  + O(\alpha) \nonumber \\
 &= \left(\frac{2\pi}{\operatorname{Vol}(\man)}\right)^2 \left( \pi \varepsilon\right)^d \frac{\left( \pi \varepsilon\right)^{(d+1)/2}}{2^{(d-1)/2} } \cdot O(\varepsilon) + O(\alpha) = O(\varepsilon^{3d/2 + 3/2}) + O(\alpha).
\end{align}
Additionally, from~\eqref{eq:G_i_theta asym} we have
\begin{equation}
\mathbb{E}\left(\left[G_i^\vartheta \right]\right)^4 = O(\varepsilon^{2d}),
\end{equation}
and thus
\begin{equation}
p_+(N,\alpha) \leq \exp{\left\{ -\frac{\alpha^2}{O(\varepsilon^{-d/2+3/2}/N) + O(\alpha)}\right\}}.
\end{equation}
Henceforth, by taking $\alpha=O(\varepsilon^{-d/4+3/4}/\sqrt{N})$ we can make $p_+(N,\alpha)$ arbitrarily small with exponential decay. Additionally, we mention that $p_-(N,\alpha)$ leads to the same asymptotic expression. Therefore, it follows that with high probability
\begin{equation}
\left\vert \frac{\sum_{j\neq i}^N H_i^\vartheta(x_j)}{\sum_{j\neq i}^N G_i^\vartheta(x_j)}  - \frac{\mathbb{E}\left[H_i^\vartheta\right]}{\mathbb{E}\left[G_i^\vartheta\right]} \right\vert = \left\vert \alpha\right\vert = O(\frac{\varepsilon^{-d/4+3/4}}{{N}^{1/2}}) = O(\frac{1}{{N}^{1/2} \varepsilon^{(d-1)/4-1/2}}).
\end{equation}
Continuing, we can write (using~\eqref{eq:bias error in proof})
\begin{align}
&\frac{4}{\varepsilon}\left\vert \frac{\sum_{j\neq i}^N H_i^\vartheta(x_j)}{\sum_{j\neq i}^N G_i^\vartheta(x_j)} - \frac{\mathbb{E}\left[H_i^\vartheta\right]}{\mathbb{E}\left[G_i^\vartheta\right]} \right\vert 
= \left\vert \frac{4}{\varepsilon}\left( f(x_i^\vartheta) - \frac{\sum_{j\neq i}^N H_i^\vartheta(x_j)}{\sum_{j\neq i}^N G_i^\vartheta(x_j)}\right)  - \frac{4}{\varepsilon}\left( f(x_i^\vartheta) - \frac{\mathbb{E}\left[H_i^\vartheta\right]}{\mathbb{E}\left[G_i^\vartheta\right]} \right) \right\vert \nonumber \\
&= \left\vert \frac{4}{\varepsilon}\left( f(x_i^\vartheta) - \frac{\sum_{j\neq i}^N H_i^\vartheta(x_j)}{\sum_{j\neq i}^N G_i^\vartheta(x_j)} \right)  - \left( \Delta_\man f(x_i^\vartheta) +O(\varepsilon) \right) \right\vert = O(\frac{1}{{N}^{1/2} \varepsilon^{(d-1)/4+1/2}}),
\end{align}
which gives us that
\begin{equation}
\frac{4}{\varepsilon}\left( f(x_i^\vartheta) - \frac{\sum_{j\neq i}^N H_i^\vartheta(x_j)}{\sum_{j\neq i}^N G_i^\vartheta(x_j)} \right) = \Delta_\man f(x_i^\vartheta) + O(\frac{1}{{N}^{1/2} \varepsilon^{(d-1)/4-1/2}}) +O(\varepsilon). \label{eq:asym variance error without diag}
\end{equation}
The last step of the proof is to justify that removing the diagonal of the steerable affinity operator $W$ (i.e. computing all sums with $j\neq i$) does not change the convergence rate. Indeed, this is the case since
\begin{align}
&\frac{\sum_{j=1}^N H_i^\vartheta(x_j)}{\sum_{j=1}^N G_i^\vartheta(x_j)} - \frac{\sum_{j\neq i}^N H_i^\vartheta(x_j)}{\sum_{j\neq i}^N G_i^\vartheta(x_j)} 
= \frac{\sum_{j=1}^N H_i^\vartheta(x_j)}{\sum_{j=1}^N G_i^\vartheta(x_j)} - \frac{\sum_{j\neq i}^N H_i^\vartheta(x_j)}{\sum_{j=1}^N G_i^\vartheta(x_j)} \frac{\sum_{j=1}^N G_i^\vartheta(x_j)}{\sum_{j\neq i}^N G_i^\vartheta(x_j)} \nonumber \\
& = \frac{\sum_{j=1}^N H_i^\vartheta(x_j)}{\sum_{j=1}^N G_i^\vartheta(x_j)} - \frac{\sum_{j\neq i}^N H_i^\vartheta(x_j)}{\sum_{j=1}^N G_i^\vartheta(x_j)} \left( 1 + \frac{ G_i^\vartheta(x_i)}{\sum_{j\neq i}^N G_i^\vartheta(x_j)} \right) 
= \frac{H_i^\vartheta(x_i)}{\sum_{j=1}^N G_i^\vartheta(x_j)} - \frac{ G_i^\vartheta(x_i) }{\sum_{j=1}^N G_i^\vartheta(x_j)} \frac{\sum_{j\neq i}^N H_i^\vartheta(x_j) }{\sum_{j\neq i}^N G_i^\vartheta(x_j)} \nonumber \\
&= O\left( \frac{G_i^\vartheta(x_i)}{\sum_{j=1}^N G_i^\vartheta(x_j)} \right),
\end{align}
where we used the fact that since $f$ is smooth then it is also bounded on $\man$, satisfying $\left\vert f(x) \right\vert \leq c$, and hence $\left\vert H_i^\vartheta(x_i) \right\vert \leq c \left\vert G_i^\vartheta(x_i) \right\vert$. Therefore, by using Proposition~\ref{prop:asym analysis} (specifically, retaining the zero-order element in the asymptotic expansion in~\eqref{eq:asym analysis general int formula}) it follows that
\begin{equation}
\frac{G_i^\vartheta(x_i)}{\sum_{j=1}^N G_i^\vartheta(x_j)} = \frac{\frac{1}{N} G_i^\vartheta(x_i)}{\frac{1}{N} \sum_{j=1}^N G_i^\vartheta(x_j)}  = O(\frac{\varepsilon^{1/2}/N}{\varepsilon^{d/2}}) = O(\frac{1}{N \varepsilon^{(d-1)/2}}),
\end{equation}
which is negligible compared to variance error term of~\eqref{eq:asym variance error without diag}. Overall, we get that
\begin{equation}
\frac{4}{\varepsilon}\left\{ \tilde{L}g\right\}(i,\vartheta) = \frac{4}{\varepsilon}\left( f(x_i^\vartheta) - \frac{\sum_{j=1}^N H_i^\vartheta(x_j)}{\sum_{j=1}^N G_i^\vartheta(x_j)} \right) = \Delta_\man f(x_i^\vartheta) + O(\frac{1}{{N}^{1/2} \varepsilon^{(d-1)/4-1/2}}) +O(\varepsilon),
\end{equation}
which concludes the proof.

\subsection{Construction of a rotationally-invariant parametrization} \label{subsection:construction of a rotationally-invariant parametrizaion}
We construct a parametrization $(z,\beta)$ of all points $x$ in a certain neighbourhood of $x_i^\vartheta$. This parametrization has favorable properties for our purposes, and is specific for every index $i$ and rotation angle $\vartheta$. 
The parametrization is defined by the mapping $x\mapsto (z,\beta)$, given by
\begin{equation}
z(x) = \mathcal{R}(x,{\hat{\alpha}(x)}), \quad\quad\quad \beta(x) = -\hat{\alpha}(x), \quad\quad\quad \hat{\alpha}(x) = \argmin_{\alpha\in [0,2\pi)}{\left\Vert x^\alpha - x_i^\vartheta\right\Vert_2^2}. \label{eq:rot inv parametrizaion explicit}
\end{equation}
That is, $z$ is the rotation of $x$ (by~\eqref{eq:extrinsic rotation}) which is closest to $x_i^\vartheta$, and therefore this parametrization satisfies
\begin{equation}
x = \mathcal{R}(z,\beta) = z^\beta.
\end{equation}
Note that this mapping is invariant to the intrinsic rotation, that is, different values of $x$ which differ only by a rotation will be mapped to the same $z$. Therefore, the parametrization $(z,\beta)$ can be perceived as a form of a polar parametrization, where coordinates are parametrized by a radius (the equivalent of $z$) and a rotation angle (the equivalent of $\beta$).

Now, a solution to
\begin{equation}
\argmin_{\alpha\in [0,2\pi)} \left\Vert x^\alpha - x_i^\vartheta\right\Vert_2^2 \label{eq:rot inv param opt alpha}
\end{equation}
must exist (since the set of minimizers is compact), but may not be unique for all $x\in\man$. We start by showing that it is guaranteed to be unique for $x$ in a sufficiently small neighbourhood of $x_i^\vartheta$. To this end, note that due to our requirements in Theorem~\eqref{thm:steerable graph Laplacian convergence}, we have with probability one that $\sum_{\left\vert m\right\vert>0} \sum_{\ell} \left\vert x_{i,(m,\ell)} \right\vert^2 >0$ (i.e. the image corresponding to $x_i^\vartheta$ is not radially symmetric), and since $\man$ is smooth, there exists a neighbourhood of $x_i^\vartheta$ for which $\sum_{\left\vert m\right\vert>0} \sum_{\ell} \left\vert x_{(m,\ell)} \right\vert^2 >0$. 
Let us consider a ball of radius $\delta>0$ around $x_i^\vartheta$, denoted by $B_\delta(x_i^\vartheta)$, such that $\sum_{\left\vert m\right\vert>0} \sum_{\ell} \left\vert x_{(m,\ell)} \right\vert^2 >0$ for all points $x\in B_\delta(x_i^\vartheta) \cap \man$. Clearly, fixing some $x\in B_\delta(x_i^\vartheta) \cap \man$, we have that $\min_{\alpha\in [0,2\pi)}{\left\{\left\Vert x^\alpha - x_i^\vartheta\right\Vert_2^2\right\}} \leq \delta^2$, and moreover, the curve $\left\lbrace x^\alpha \right\rbrace_{\alpha=0}^{2\pi}$ has bounded curvature and does not self intersect, which means that 
there exists a sufficiently small $\delta^{'}\leq \delta$ such that any $\alpha$ minimzing $\left\Vert x^\alpha - x_i^\vartheta\right\Vert_2^2$ must satisfy $x^\alpha \in B_{\delta^{'}}(x_i^\vartheta) \cap \man$. Moreover, we choose $\delta^{'}$ such that the solution is unique, which is justified by the fact that the problem $\min_{\alpha\in [0,2\pi)}\left\Vert x^\alpha - x_i^\vartheta\right\Vert_2^2$ s.t. $x^\alpha \in B_{\delta^{'}}(x_i^\vartheta) \cap \man$, must be convex for a sufficiently small $\delta^{'}$ (due to the smoothness of the curve $\left\lbrace x^\alpha \right\rbrace_{\alpha=0}^{2\pi}$). To conclude, for points $x$ in a sufficiently small neighbourhood of $x_i^\vartheta$, the mapping $x\mapsto (z,\beta)$ given by~\eqref{eq:rot inv parametrizaion explicit} is unique.

Next, it is of interest to characterize the resulting set of feasible points of $z$, and we proceed by showing that
\begin{equation}
\mathcal{N} \triangleq \left\{z(x): \; x \in B_{\delta^{'}}(x_i^\vartheta) \cap \man \right\}
\end{equation}
is a smooth and compact $(d-1)$-dimensional submanifold.
Now, from~\eqref{eq:rot inv param opt alpha} it follows that $\alpha=\hat{\alpha}(x)$ must be a solution of
\begin{equation}
\operatorname{Re}{\left\{\left\langle \frac{\partial x^\alpha}{\partial \alpha},x^\alpha - x_i^\vartheta\right\rangle\right\}} = 0, \label{eq:rot inv param alpha explicit eq}
\end{equation}
which can be written explicitly via~\eqref{eq:extrinsic rotation} as 
\begin{align}
\operatorname{Re}{\left\{\left\langle \frac{\partial x^\alpha}{\partial \alpha},x^\alpha - x_i^\vartheta\right\rangle\right\}} &= \sum_{m=-M}^M\sum_{\ell=1}^{\ell_m} \operatorname{Re}{\left\{ \imath m \cdot  x_{(m,\ell)}^* e^{-\imath m \alpha} \cdot \left(x_{(m,\ell)}e^{\imath m \alpha} - x_{i,(m,\ell)}^\vartheta \right) \right\}} \nonumber \\
&= \sum_{m=-M}^M\sum_{\ell=1}^{\ell_m} \operatorname{Re}{\left\{ \imath m \cdot  \left\vert x_{(m,\ell)} \right\vert^2 - \imath m \cdot  x_{(m,\ell)}^* e^{-\imath m \alpha} \cdot x_{i,(m,\ell)}^\vartheta \right\}} \nonumber \\ 
&= -\sum_{m=-M}^M\sum_{\ell=1}^{\ell_m} \operatorname{Re}{\left\{ \imath m  x_{(m,\ell)}^*  x_{i,(m,\ell)}^\vartheta \cdot e^{-\imath m \alpha} \right\}}. \label{eq:rot inv param trig poly for alpha}
\end{align}
Next, since at $\alpha=\hat{\alpha}(x)$ we have that $z=x^\alpha$, then by~\eqref{eq:rot inv parametrizaion explicit} each $z$ satisfies
\begin{equation}
\operatorname{Re}{\left\{\left\langle \frac{\partial z^\beta}{\partial \beta}\bigg{\vert}_{\beta=0},z - x_i^\vartheta\right\rangle\right\}} = 0, \label{eq:parametrization orthogonality}
\end{equation}
which can be written explicitly via~\eqref{eq:rot inv param trig poly for alpha} as 
\begin{equation}
\sum_{m=-M}^M\sum_{\ell=1}^{\ell_m} \operatorname{Re}{\left\{ \imath m \cdot z_{(m,\ell)}^* \cdot x_{i,(m,\ell)}^\vartheta \right\}} = 0, \label{eq:rot-inv parametrization linear subspace}
\end{equation}
where $z_{(m,\ell)}$ denotes the $(m,\ell)$`th coordinate of $z$.
Essentially, equation~\eqref{eq:rot-inv parametrization linear subspace} defines a linear subspace of the ambient space, such that all feasible points for $z$ ($z=x^\alpha$ for $\alpha$ satisfying~\eqref{eq:rot inv param opt alpha}) must lie in the intersection between ${B}_{\delta^{'}}(x_i^\vartheta)\cap \man$ and this subspace. 
In particular, the submanifold $\mathcal{N}$ can be explicitly defined through
\begin{equation}
\mathcal{N} = \left\lbrace x: \; \sum_{m=-M}^M\sum_{\ell=1}^{\ell_m} \operatorname{Re}{\left\{ \imath m \cdot  x_{(m,\ell)}^* \cdot x_{i,(m,\ell)}^\vartheta\right\}} = 0,\; x\in {B}_{\delta^{'}}(x_i^\vartheta)\cap \man \right\rbrace,
\end{equation}
which is a smooth and compact submanifold due to the smoothness and compactness of $\man$, and is of intrinsic dimension $(d-1)$ due to the additional linear constraint (note that this constraint is not degenerate since $\sum_{\left\vert m\right\vert>0} \sum_{\ell} \left\vert x_{(m,\ell)} \right\vert^2 >0$ for $x\in B_\delta(x_i^\vartheta)\cap \man$).

Lastly, we make the observation that all rotations of any  point in the neighbourhood $B_{\delta^{'}}(x_i^\vartheta) \cap \man$, i.e. $x^\varphi$ for all $\varphi \in [0,2\pi)$ and any $x\in B_{\delta^{'}}(x_i^\vartheta) \cap \man$, share the same solution (the same $x^\alpha$ for $\alpha$ from~\eqref{eq:rot inv param opt alpha}) as the point $x$. This allows us to extend the neighbourhood in which our parametrization is valid by taking all rotations of all points in this neighbourhood, and in particular, we conclude that for all $x\in\man^{'}$, where
\begin{equation}
\man^{'} \triangleq \left\{x^\varphi: \; x\in B_{\delta^{'}}(x_i^\vartheta) \cap \man, \; \varphi \in [0,2\pi) \right\},
\end{equation}
the parametrization $x\mapsto (z,\beta)$, where $x=z^\beta$, is unique. Additionally, it is evident that $\man^{'}\subset\man$ is also a smooth and compact $d$-dimensional submanifold.

\subsection{Proof of Lemma~\ref{lem:rotation invariant measure} } \label{subsection:angle invariant measure lemma proof }
\begin{proof}
Let $\mathcal{N}$ be parametrized locally by $u=\left[ u_1,\ldots,u_{(d-1)}\right]\in\mathbb{R}^{d-1}$ around a point $z_0$. That is, every coordinate $z_{m,\ell}$ of the manifold $\mathcal{N}$ is expressed as a function of $u$ in the vicinity of the point $z_0$. Then, $\man^{'}$ can be parametrized locally around $x_0=z_0^{\beta_0}$ by $[ u,\beta]\in\mathbb{R}^d$, where $\beta\in [0,2\pi)$ is the rotation angle from our rotationally-invariant parametrization (see section~\ref{subsection:construction of a rotationally-invariant parametrizaion}), in the sense that every coordinate $x_{m,\ell}$ of $\man^{'}$ can be expressed as a function of $[u,\beta]$ in the neighbourhood of the point $x_0$. Hence, the integral of a function $h(x)$ over $\man^{'}$ can be expressed through the parametrization $x\mapsto (z,\beta)$ by
\begin{equation}
\int_{\man^{'}} h(x) dx = \int_{z\in\mathcal{N}} \int_{\beta=0}^{2\pi} h(z^\beta) dV(z^\beta) ,
\end{equation}
where 
\begin{equation}
dV(x) = \sqrt{\left\vert\det{ \left\{ g_{\man^{'}}(x)\right\} }\right\vert} du_1 \ldots du_{d-1} d\beta
\end{equation}
is the volume form at the point $x$, $g_{\man^{'}}(x)$ is the metric tensor on $\man^{'}$, given by pull-back as
\begin{equation}
g_{\man^{'}}(x) = \operatorname{Re}\left\{J_{\man^{'}}^*(x) J_{\man^{'}}(x)\right\},
\end{equation}
and $J_{\man^{'}}(x)$ is the Jacobian matrix
\begin{equation}
J_{\man^{'}}(x) = 
\begin{bmatrix}
J_{u} & J_\beta
\end{bmatrix}
, \quad\quad
J_u = 
\begin{bmatrix}
\frac{\partial x_{(-M,1)}}{\partial u_1} & \ldots & \frac{\partial x_{(-M,1)}}{\partial u_{(d-1)}} & \\
\vdots & \ldots & \vdots \\
\frac{\partial x_{(-M,\ell_m)}}{\partial u_1} & \ldots & \frac{\partial x_{(-M,\ell_m)}}{\partial u_{(d-1)}}
\\
\textbf{\vdots} & & \textbf{\vdots} \\
\frac{\partial x_{(M,1)}}{\partial u_1} & \ldots & \frac{\partial x_{(M,1)}}{\partial u_{(d-1)}} \\
\vdots & \ldots & \vdots \\
\frac{\partial x_{(M,\ell_m)}}{\partial u_1} & \ldots & \frac{\partial x_{(M,\ell_m)}}{\partial u_{(d-1)}}
\end{bmatrix}
,\quad\quad
J_\beta = 
\begin{bmatrix}
\frac{\partial x_{(-M,1)}}{\partial \beta} \\
\vdots \\
\frac{\partial x_{(-M,\ell_m)}}{\partial \beta} \\
\textbf{\vdots} \\
\frac{\partial x_{(M,1)}}{\partial \beta} \\
\vdots \\
\frac{\partial x_{(M,\ell_m)}}{\partial \beta}
\end{bmatrix}. \label{eq: Jacobian def}
\end{equation}
Note that since $x=z^\beta$ we have $x_{m,\ell}=z_{m,\ell} e^{\imath m \beta}$ (from~\eqref{eq:extrinsic rotation}), and thus
\begin{equation}
\frac{\partial x_{(m,\ell)}}{\partial u_i} = \frac{\partial z_{(m,\ell)}}{\partial u_i} \cdot e^{\imath m \beta},\quad i=1,\ldots,d-1; \quad\quad\quad \frac{\partial x_{(m,\ell)}}{\partial \beta} = \imath m \cdot z_{m,\ell} \cdot e^{\imath m \beta}. \label{eq:metric tensor partial derivatives explicit}
\end{equation}
Therefore, it is evident that the metric tensor
\begin{equation}
g_{\man^{'}}(x) = \operatorname{Re}\left\{J_{\man^{'}}^*(x) J_{\man^{'}}(x)\right\} = 
\begin{bmatrix}
\operatorname{Re}\left\{J_u^*(x) J_u(x)\right\} & \operatorname{Re}\left\{J_u^*(x) J_\beta(x)\right\} \\
\operatorname{Re}\left\{J_\beta^*(x) J_u(x)\right\} & \operatorname{Re}\left\{J_\beta^*(x) J_\beta(x)\right\}
\end{bmatrix}
\end{equation}
does not depend on $\beta$, i.e. $g_{\man^{'}}(x)=g_{\man^{'}}(z)$, since the term $e^{\imath m \beta}$ cancels-out in all entries of $g_{\man^{'}}(x)$. Consequently, we have that
\begin{equation}
dV(x) = \sqrt{\left\vert\det{\left\{ g_{\man^{'}}(z)\right\} }\right\vert } du_1 \ldots du_{d-1} d\beta = {V}(z) dz d\beta,
\end{equation}
where we denoted 
\begin{equation}
{V}(z)=\sqrt{\left\vert\det{\left\{ g_{\man^{'}}(z)\right\} }\right\vert}, \quad\quad\quad dz = du_1 \ldots du_{d-1}, \label{eq:V_tilde def}
\end{equation}
and it follows that
\begin{equation}
\int_{\man^{'}} h(x) dx = \int_{z\in\mathcal{N}} \int_{\beta=0}^{2\pi} h(z^\beta) {V}(z) dz d\beta.
\end{equation}
\end{proof}

\subsection{Proof of Lemma~\ref{lem:second order moments asymptotics}} \label{subsection:proof of lemma second order moments asymptotics}
We put our focus on evaluating the term $\mathbb{E}\left[\left(H_i^\vartheta \right)^2\right]$, as the other second-order terms $\mathbb{E}\left[\left(G_i^\vartheta \right)^2\right]$ and $\mathbb{E}\left[H_i^\vartheta G_i^\vartheta\right]$ can be obtained in a similar fashion. 
Thus, we are interested in evaluating the term
\begin{equation}
\mathbb{E}\left[\left(H_i^\vartheta \right)^2\right] = \int_{\mathcal{\man}} \left(H_i^\vartheta(x) \right)^2 p(x) dx, \label{eq:E[H] explicit}
\end{equation}
recalling that
\begin{equation}
H_i^\vartheta(x) = \int_0^{2\pi} \exp{\left\lbrace-{\left\Vert x_i^\vartheta - x^{\varphi}\right\Vert^2}{/\varepsilon}\right\rbrace}f(x^{\varphi})d\varphi.
\end{equation}
Using the construction of our rotationally-invariant parametrization (see Section~\ref{subsection:construction of a rotationally-invariant parametrizaion}),  and in particular the submanifold $\man^{'}$, we can write
\begin{equation}
\mathbb{E}\left[\left(H_i^\vartheta \right)^2\right] = \int_{\mathcal{\man^{'}}} \left(H_i^\vartheta(x) \right)^2 p(x) dx + O\left(e^{-(\delta^{'})^2 /\varepsilon}\right),
\end{equation}
for some constant $\delta^{'}>0$. Next, as the $O\left(e^{-(\delta^{'})^2 /\varepsilon}\right)$ term is negligible w.r.t to any polynomial asymptotic expansion in $\varepsilon$, we omit it. Then, by Lemma~\ref{lem:rotation invariant measure} (i.e. change of integration variables $x\mapsto (z,\beta)$) and~\eqref{eq:H_i_theta invariant to beta}, we have
\begin{align}
\mathbb{E}\left[\left(H_i^\vartheta \right)^2\right] &= \int_{\mathcal{N}}\int_0^{2\pi} \left(H_i^\vartheta(z^\beta) \right)^2 p(z^{\beta}) V(z) dz d\beta = \frac{2\pi}{\operatorname{Vol}\left\{\man\right\}} \int_{\mathcal{N}} \left(H_i^\vartheta(z) \right)^2 V(z) dz \nonumber \\
&=  \int_{\mathcal{N}} \left(H_i^\vartheta(z) \right)^2 p_{\mathcal{N}}(z) dz,
\end{align} 
where we defined
\begin{equation}
p_{\mathcal{N}}(z) =\frac{2\pi}{\operatorname{Vol}\left\{\man\right\}} V(z).
\end{equation}

We start by deriving an asymptotic expression for $H_i^\vartheta(z)$.
Let us write
\begin{align}
\left\Vert x_i^\vartheta - z^{\varphi}\right\Vert^2 &= \left\Vert \left( x_i^\vartheta - z\right) + \left( z - z^{\varphi}\right)\right\Vert^2 \nonumber\\ 
&= \left\Vert  x_i^\vartheta - z \right\Vert^2 + 2\operatorname{Re}\left\{ \left\langle x_i^\vartheta - z,z - z^\varphi \right\rangle\right\} + \left\Vert  z - z^\varphi\right\Vert^2,
\end{align}
and denote
\begin{equation}
\delta_i^{\vartheta}(z,x) \triangleq 2\operatorname{Re}\left\{ \left\langle x_i^\vartheta - z,z - x \right\rangle\right\}. \label{eq:Delta def}
\end{equation}
Therefore, we have that
\begin{equation}
H_i^\vartheta(z) = \exp{\left\lbrace-{\left\Vert x_i^\vartheta - z \right\Vert^2}{/\varepsilon}\right\rbrace} \int_0^{2\pi} \exp{\left\lbrace-{\left\Vert  z - z^\varphi\right\Vert^2}{/\varepsilon}\right\rbrace} \exp{\left\lbrace-{ \delta_i^{\vartheta}(z,z^\varphi)}{/\varepsilon}\right\rbrace} f(z^{\varphi})d\varphi. \label{eq:H_i_theta full integral}
\end{equation}
Next, we use Taylor expansion to write
\begin{equation}
\exp{\left\lbrace-{ \delta_i^{\vartheta}(z,z^\varphi)}{/\varepsilon}\right\rbrace} = 1 - \frac{\delta_i^{\vartheta}(z,z^\varphi)}{\varepsilon} + O\left(\frac{\left[\delta_i^{\vartheta}(z,z^\varphi)\right]^2}{\varepsilon^2}\right),
\end{equation}
which gives us that
\begin{align}
H_i^\vartheta(z) = \exp{\left\lbrace-{\left\Vert x_i^\vartheta - z \right\Vert^2}{/\varepsilon}\right\rbrace} &\left[ \int_0^{2\pi} \exp{\left\lbrace-{\left\Vert  z - z^\varphi\right\Vert^2}{/\varepsilon}\right\rbrace} f(z^{\varphi})d\varphi \right. \nonumber \\ 
&\left. - \frac{1}{\varepsilon} \int_0^{2\pi} \exp{\left\lbrace-{\left\Vert  z - z^\varphi\right\Vert^2}{/\varepsilon}\right\rbrace} \delta_i^{\vartheta}(z,z^\varphi) f(z^{\varphi})d\varphi \right. \nonumber \\
&\left. + O\left( \frac{1}{\varepsilon^2} \int_0^{2\pi} \exp{\left\lbrace-{\left\Vert  z - z^\varphi\right\Vert^2}{/\varepsilon}\right\rbrace} \left[ \delta_i^{\vartheta}(z,z^\varphi)\right]^2 \left| f(z^{\varphi}) \right| d\varphi \right) \right]. \label{eq:H_i_theta error terms expansion}
\end{align}

In what follows, we evaluate the terms in the square brackets of~\eqref{eq:H_i_theta error terms expansion} one by one, where we mention that Proposition~\ref{prop:asym analysis} is the main workhorse for obtaining the asymptotic expansions of the integrals taking part in our analysis. 
To this end, let us first define the set of points $\mathcal{C}_z=\left\{ z^\varphi\right\}_{\varphi=0}^{2\pi}$, which is a smooth curve in $\CD$. We then change the integration parameter in~\eqref{eq:H_i_theta error terms expansion} from the angle $\varphi$ to the variable $x=z^\varphi$ (which is equivalent to parametrizing by arc-length), and if we recall that $f(x)$ is a smooth function, then by Proposition~\ref{prop:asym analysis} we get that
\begin{align}
\int_0^{2\pi} \exp{\left\lbrace-{\left\Vert  z - z^\varphi\right\Vert^2}{/\varepsilon}\right\rbrace} f(z^{\varphi})d\varphi &= \frac{1}{\mu(z)}\int_{x\in \mathcal{C}_z} \exp{\left\lbrace-{\left\Vert  z - x\right\Vert^2}{/\varepsilon}\right\rbrace} f(x) dx \nonumber \\
&= \frac{\sqrt{\pi\varepsilon}}{\mu(z) }\left[f(z) + O(\varepsilon) \right], \label{eq:inner int asym term 1}
\end{align}
where $\mu(z) = \sqrt{\left\vert \det{\left[ \operatorname{Re}{\left\{ J_\beta^*(z) J_\beta(z) \right\}} \right] } \right\vert}$ ($J_\beta$ is defined in~\eqref{eq: Jacobian def}) is associated with the change of the integration variable, and is given explicitly by
\begin{equation}
\mu(z) = \sqrt{ \sum_{m=-M}^M \sum_{\ell=1}^{\ell_m} m^2 \left\vert z_{m,\ell} \right\vert^2} .
\end{equation}

Next, we evaluate the second term in the square brackets of~\eqref{eq:H_i_theta error terms expansion}. Since $\delta_i^{\vartheta}(z,z^\varphi)$ is a smooth function in $\varphi$, and using the previous change of variable $x=z^\varphi$ together with Proposition~\ref{prop:asym analysis}, we have
\begin{align}
&\frac{1}{\varepsilon} \int_0^{2\pi} \exp{\left\lbrace-{\left\Vert  z - z^\varphi\right\Vert^2}{/\varepsilon}\right\rbrace} \delta_i^{\vartheta}(z,z^\varphi) f(z^{\varphi})d\varphi = 
\frac{1}{\varepsilon \mu(z)} \int_{x\in\mathcal{C}_z} \exp{\left\lbrace-{\left\Vert  z - x\right\Vert^2}{/\varepsilon}\right\rbrace} \delta_i^{\vartheta}(z,x) f(x)dx \nonumber \\ 
&= \frac{\sqrt{\pi \varepsilon}}{\varepsilon \mu(z) } \left[ \delta_i^{\vartheta}(z,z)f(z) + \frac{\varepsilon}{4}\left[ E(y)\delta_i^{\vartheta}(z,z)f(z) + \Delta_{\mathcal{C}_z}\left\{\delta_i^{\vartheta}(z,x)f(x)\right\}\bigg\vert_{x=z}\right] + O(\varepsilon^2) \right] \nonumber \\
&= \frac{\sqrt{\pi \varepsilon}}{\varepsilon \mu(z) } \left[ \frac{\varepsilon}{4} \Delta_{\mathcal{C}_z}\left\{\delta_i^{\vartheta}(z,x)f(x)\right\}\bigg\vert_{x=z} + O(\varepsilon^2) \right]
\end{align}
since it is clear from~\eqref{eq:Delta def} that $\delta_i^{\vartheta}(z,z)=0$. Moreover, we have (see Lemma~3.3 in~\cite{rosenberg1997laplacian}) that
\begin{align}
\Delta_{\mathcal{C}_z}\left\{\delta_i^{\vartheta}(z,x)f(x)\right\}\vert_{x=z} &= \Delta_{\mathcal{C}_z}f(x)\vert_{x=z} \cdot\delta_i^{\vartheta}(z,z) -2\left\langle \nabla_{\mathcal{C}_z} \delta_i^{\vartheta}(z,x)\vert_{x=z}, \nabla_{\mathcal{C}_z} f(z)\right\rangle + f(z)\cdot \Delta_{\mathcal{C}_z}\delta_i^{\vartheta}(z,x)\vert_{x=z} \nonumber \\
&= f(z)\cdot \Delta_{\mathcal{C}_z}\delta_i^{\vartheta}(z,x)\vert_{x=z},
\end{align}
where we have used the fact that $\delta_i^{\vartheta}(z,z)=0$, and moreover, that (using~\eqref{eq:Delta def})
\begin{equation}
\nabla_{\mathcal{C}_z} \delta_i^{\vartheta}(z,x)\vert_{x=z} = -2\operatorname{Re}\left\{ \left\langle x_i^\vartheta - z, \nabla_{\mathcal{C}_z} x \vert_{x=z} \right\rangle\right\} = -2\operatorname{Re}\left\{ \left\langle x_i^\vartheta - z, \frac{1}{\mu(z)} \frac{\partial z^\beta}{\partial \beta}\bigg{\vert}_{\beta=0} \right\rangle\right\} = 0
\end{equation}
as $x_i^\vartheta - z$ is perpendicular to $\frac{\partial z^\beta}{\partial \beta} \bigg\vert_{\beta=0}$ by our rotationally-invariant parametrization (equation~\eqref{eq:parametrization orthogonality} in Section~\ref{subsection:construction of a rotationally-invariant parametrizaion}). Therefore, we are left with 
\begin{align}
\frac{1}{\varepsilon} \int_0^{2\pi} \exp{\left\lbrace-{\left\Vert  z - z^\varphi\right\Vert^2}{/\varepsilon}\right\rbrace} \delta_i^{\vartheta}(z,z^\varphi) f(z^{\varphi})d\varphi = \frac{\sqrt{\pi \varepsilon}}{\varepsilon \mu(z) } \left[ \varepsilon q(z) + O(\varepsilon^2) \right]
= \frac{ \sqrt{\pi \varepsilon}}{\mu(z)} \left[ q(z) + O(\varepsilon) \right], \label{eq:inner int asym term 2}
\end{align} 
where we defined the function $q(z)$ as
\begin{equation}
q(z) = \frac{f(z)}{4} \Delta_{\mathcal{C}_z}\delta_i^{\vartheta}(z,x)\vert_{x=z} = -\frac{f(z)}{2} \operatorname{Re}\left\{ \left\langle x_i^\vartheta - z, \Delta_{\mathcal{C}_z}x \vert_{x=z} \right\rangle\right\}, \label{eq:q(z) def}
\end{equation}
with the second equality following from~\eqref{eq:Delta def}. The notation $\Delta_{\mathcal{C}_z}x \vert_{x=z}$ denotes the Laplacian of each coordinate of $x$, taken w.r.t to the curve $\mathcal{C}_z$, and sampled at the point $z$. Consequently, $q(z)$ is a smooth function satisfying
\begin{equation}
q(x_i^\vartheta)=0, \quad\quad\quad q(z)  = O(\left\Vert x_i^\vartheta - z  \right\Vert), \label{eq:q(z) properties}
\end{equation}
where we applied the Cauchy-Schwarz inequality to~\eqref{eq:q(z) def}, together with the fact that $\left\Vert \Delta_{\mathcal{C}_z}x \vert_{x=z} \right\Vert$ is bounded (since $\mathcal{C}_z$ is smooth).

Now, as for the last term in the square brackets of~\eqref{eq:H_i_theta error terms expansion}, we first mention that since $f(x)$ is smooth, it is bounded (on a compact domain) and therefore
\begin{multline}
O\left( \frac{1}{\varepsilon^2} \int_0^{2\pi} \exp{\left\lbrace-{\left\Vert  z - z^\varphi\right\Vert^2}{/\varepsilon}\right\rbrace} \left[ \delta_i^{\vartheta}(z,z^\varphi)\right]^2 \left| f(z^{\varphi}) \right| d\varphi \right) = \\ 
O\left( \frac{1}{\varepsilon^2} \int_0^{2\pi} \exp{\left\lbrace-{\left\Vert  z - z^\varphi\right\Vert^2}{/\varepsilon}\right\rbrace} \left[ \delta_i^{\vartheta}(z,z^\varphi)\right]^2 d\varphi \right). \label{eq:delta^2 int}
\end{multline}
Moreover, if we expand
\begin{equation}
z^\varphi_{m,\ell} = z_{m,\ell} e^{\imath m \varphi} = z_{m,\ell} + \frac{\partial z_{m,\ell}^\varphi}{\partial \varphi} \bigg\vert_{\varphi=0} \cdot \varphi + O(\varphi^2),
\end{equation}
then we have that
\begin{equation}
z^\varphi-z=\frac{\partial z^\varphi}{\partial \varphi} \bigg\vert_{\varphi=0} \cdot \varphi + O(\varphi^2), \label{eq:z_varphi asym}
\end{equation}
and it is evident that
\begin{align}
\delta_i^{\vartheta}(z,z^\varphi) &=  -2\operatorname{Re}\left\{ \left\langle x_i^\vartheta - z,z - z^\varphi \right\rangle\right\} = 2\operatorname{Re}\left\{ \left\langle x_i^\vartheta - z,\frac{\partial z^\varphi}{\partial \varphi} \bigg\vert_{\varphi=0} \right\rangle\right\} \cdot \varphi + O(\varphi^2 \left\Vert x_i^\vartheta - z\right\Vert) \nonumber \\
&= O(\varphi^2 \left\Vert x_i^\vartheta - z\right\Vert),
\end{align}
where we used Cauchy-Schwarz inequality and again the fact that $x_i^\vartheta - z$ is perpendicular to $\frac{\partial z^\varphi}{\partial \varphi} \bigg\vert_{\varphi=0}$ due to our rotationally-invariant parametrization (eq.~\eqref{eq:parametrization orthogonality}).
Eventually, we obtain that
\begin{equation}
\left[ \delta_i^{\vartheta}(z,z^\varphi)\right]^2 = O(\varphi^4 \left\Vert x_i^\vartheta - z\right\Vert^2).
\end{equation}
Continuing, from~\eqref{eq:z_varphi asym} it is clear that
\begin{equation}
\varphi = O(\left\Vert z-z^\varphi \right\Vert),
\end{equation}
and therefore
\begin{equation}
\left[ \delta_i^{\vartheta}(z,z^\varphi)\right]^2 = O(\left\Vert z-z^\varphi \right\Vert^4 \left\Vert x_i^\vartheta - z\right\Vert^2). \label{eq:delta_i_theta suqared}
\end{equation}
When plugging~\eqref{eq:delta_i_theta suqared} back into~\eqref{eq:delta^2 int} and changing the integration parameter from $\varphi$ to $x=z^\varphi$, we arrive at
\begin{align}
&O\left( \frac{1}{\varepsilon^2} \int_0^{2\pi} \exp{\left\lbrace-{\left\Vert  z - z^\varphi\right\Vert^2}{/\varepsilon}\right\rbrace} \left[ \delta_i^{\vartheta}(z,z^\varphi)\right]^2 d\varphi \right) = O\left( \frac{1}{\varepsilon^2 \mu(z)} \int_{x\in\mathcal{C}_z} \exp{\left\lbrace-{\left\Vert  z - x\right\Vert^2}{/\varepsilon}\right\rbrace} \left[ \delta_i^{\vartheta}(z,x)\right]^2 dx \right) \nonumber \\
&= O\left( \frac{\left\Vert x_i^\vartheta - z\right\Vert^2}{\varepsilon^2 \mu(z)} \int_{x\in\mathcal{C}_z} \exp{\left\lbrace-{\left\Vert  z - x \right\Vert^2}{/\varepsilon}\right\rbrace} \left\Vert z-x \right\Vert^4 dx \right) = O\left( \frac{\left\Vert x_i^\vartheta - z\right\Vert^2}{\varepsilon^2 \mu(z)} \cdot {\varepsilon^2\sqrt{\pi\varepsilon}} \right) = O\left( \frac{\left\Vert x_i^\vartheta - z\right\Vert^2}{\mu(z)} \sqrt{\pi\varepsilon} \right), \label{eq:inner int asym term 3}
\end{align}
where we used the asymptotic expansion in Proposition~\ref{prop:asym analysis} together with the fact that the function $\left\Vert z-x \right\Vert^4$ (in $z$) and its Laplacian vanish at $z=x$ (leaving only the $O(\varepsilon^2)$ term in the asymptotic expansion of Proposition~\ref{prop:asym analysis}). 

Altogether, when plugging~\eqref{eq:inner int asym term 1},~\eqref{eq:inner int asym term 2}, and~\eqref{eq:inner int asym term 3} into~\eqref{eq:H_i_theta error terms expansion}, we get
\begin{align}
H_i^\vartheta(z) = \frac{\exp{\left\lbrace-{\left\Vert x_i^\vartheta - z \right\Vert^2}{/\varepsilon}\right\rbrace}}{\mu(z)} \sqrt{\pi\varepsilon} \left[ f(z) + q(z) + O(\left\Vert x_i^\vartheta - z\right\Vert^2) + O(\varepsilon)\right],
\end{align}
where $q(z)$ was defined in~\eqref{eq:q(z) def}. Therefore, we have
\begin{align}
\left( H_i^\vartheta(z) \right)^2 &= \frac{\exp{\left\lbrace-2{\left\Vert x_i^\vartheta - z \right\Vert^2}{/\varepsilon}\right\rbrace}}{\mu^2(z) } {\pi\varepsilon} \left[ f(z) + q(z) + O(\left\Vert x_i^\vartheta - z\right\Vert^2) + O(\varepsilon)\right]^2 \nonumber \\
&= \frac{\exp{\left\lbrace-2{\left\Vert x_i^\vartheta - z \right\Vert^2}{/\varepsilon}\right\rbrace}}{\mu^2(z)} {\pi\varepsilon} \left[ f^2(z) +2f(z)q(z) + O(\left\Vert x_i^\vartheta - z\right\Vert^2) + O(\varepsilon)\right], \label{eq:H_i_theta asym final}
\end{align}
where we used the fact (from~\eqref{eq:q(z) properties}) that $q(z)  = O(\left\Vert x_i^\vartheta - z  \right\Vert)$, and retained only the asymptotically dominant factors inside the square brackets.

We are now ready to evaluate $\mathbb{E}\left[\left(H_i^\vartheta \right)^2\right]$ by plugging ~\eqref{eq:H_i_theta asym final} into~\eqref{eq:E[H] explicit}. We have
\begin{align}
&\mathbb{E}\left[\left(H_i^\vartheta \right)^2\right] = \int_{\mathcal{N}} \left(H_i^\vartheta(z) \right)^2 p_{\mathcal{N}}(z) dz \nonumber \\
&= \pi\varepsilon\int_{\mathcal{N}} \frac{\exp{\left\lbrace-2{\left\Vert x_i^\vartheta - z \right\Vert^2}{/\varepsilon}\right\rbrace}}{\mu^2(z) } \left[ f^2(z) +2f(z)q(z) + O(\left\Vert x_i^\vartheta - z\right\Vert^2) + O(\varepsilon)\right] p_{\mathcal{N}}(z) dz. \label{eq:E[H]^2 int initial}
\end{align}
Before we proceed with the asymptotic analysis, we mention that if $\mu(z)> 0$ for all $z\in\mathcal{N}$, then $1/\mu^2(z)$ is a smooth function. Additionally, the smoothness of $p_{\mathcal{N}}(z)=2\pi V(z)/\operatorname{Vol}\left\{\man\right\}$ is guaranteed by the smoothness of $\man$ and $\mathcal{N}$ (see the definition of $V(x)$ in Section~\ref{subsection:angle invariant measure lemma proof }). Then, we expand the square brackets in~\eqref{eq:E[H]^2 int initial} and evaluate (asymptotically) all resulting integrals by applying Proposition~\ref{prop:asym analysis}. We get
\begin{equation}
\int_{\mathcal{N}} \exp{\left\lbrace-2{\left\Vert x_i^\vartheta - z \right\Vert^2}{/\varepsilon}\right\rbrace} \frac{f^2(z) p_{\mathcal{N}}(z)}{\mu^2(z) }  dz = {\left( \pi \varepsilon /2\right)^{(d-1)/2}} \left[ \frac{f^2(x_i^\vartheta) p_{\mathcal{N}}(x_i^\vartheta)}{\mu^2(x_i^\vartheta) } + O(\varepsilon) \right],
\end{equation}
\begin{align}
\int_{\mathcal{N}} \exp{\left\lbrace-2{\left\Vert x_i^\vartheta - z \right\Vert^2}{/\varepsilon}\right\rbrace} \frac{f(z) q(z) p_{\mathcal{N}}(z)}{\mu^2(z) }  dz &= {\left( \pi \varepsilon /2\right)^{(d-1)/2}} \left[ \frac{f(x_i^\vartheta) q(x_i^\vartheta) p_{\mathcal{N}}(x_i^\vartheta)}{\mu^2(x_i^\vartheta) } + O(\varepsilon) \right] \nonumber \\
&= {\left( \pi \varepsilon /2\right)^{(d-1)/2}} \cdot O(\varepsilon),
\end{align}
since $q(x_i^\vartheta)=0$ (see~\eqref{eq:q(z) properties}), and 
\begin{equation}
\int_{\mathcal{N}} \frac{\exp{\left\lbrace-2{\left\Vert x_i^\vartheta - z \right\Vert^2}{/\varepsilon}\right\rbrace}}{\mu^2(z) } O(\left\Vert x_i^\vartheta - z\right\Vert^2) p_{\mathcal{N}}(z) dz = {\left( \pi \varepsilon /2\right)^{(d-1)/2}} \cdot O(\varepsilon),
\end{equation}
where we used the fact that $\left\Vert x_i^\vartheta - z\right\Vert^2$ is smooth and vanishes at $z=x_i^\vartheta$ (using Proposition~\ref{prop:asym analysis}, we are left only with the $O(\varepsilon)$ term in the expansion).

Finally, by substituting all of the above asymptotic integral expansions into~\eqref{eq:E[H]^2 int initial}, it follows that
\begin{equation}
\mathbb{E}\left[\left(H_i^\vartheta \right)^2\right] = {\pi\varepsilon} {\left( \pi \varepsilon /2\right)^{(d-1)/2}} \left[ \frac{f^2(x_i^\vartheta) p_{\mathcal{N}}(x_i^\vartheta)}{\mu^2(x_i^\vartheta) } + O(\varepsilon) \right] = \frac{\left( \pi \varepsilon\right)^{(d+1)/2}}{2^{(d-1)/2} } \left[ \frac{f^2(x_i^\vartheta) p_{\mathcal{N}}(x_i^\vartheta)}{\mu^2(x_i^\vartheta) } + O(\varepsilon) \right].
\end{equation}
Then, $\mathbb{E}\left[\left(G_i^\vartheta \right)^2\right]$ and $\mathbb{E}\left[H_i^\vartheta G_i^\vartheta\right]$ can be obtained in exactly the same manner, and we omit the derivation for the sake of brevity (note that to compute $\mathbb{E}\left[\left(G_i^\vartheta \right)^2\right]$ it is sufficient to take $f\equiv 1$ throughout the derivation).

\section{Non-uniform sampling distribution} \label{appendix: Non uniform sampling}
Let us consider the case where the sampling distribution $p(x)$ is not uniform, and analyze the resulting limiting operator by following the analysis of the bias error term in Section~\ref{subsection:The limit and bias terms}. 
From~\eqref{eq:C1 lim in z}, we have that
\begin{align}
&\lim_{N\rightarrow\infty} C_{i,N}^1(\vartheta) = \lim_{N\rightarrow\infty} \frac{1}{N}\sum_{j=1}^N \int_0^{2\pi} \exp{\left\lbrace-{\left\Vert x_i^\vartheta - x_j^\varphi\right\Vert^2}{/\varepsilon}\right\rbrace}f(x_j^\varphi) d\varphi \nonumber \\
 &= \int_{\mathcal{N}} \int_0^{2\pi} H_i^\vartheta(z) V(z) p(z^\beta) dz d\beta = {2\pi} \int_{\mathcal{N}} H_i^\vartheta(z) V(z) \tilde{p}(z) dz,
\end{align}
where the submanifold $\mathcal{N}$ is from the rotationally-invariant parametrization $x\mapsto (z,\beta)$ of Section~\ref{subsection:construction of a rotationally-invariant parametrizaion}, and we defined
\begin{equation}
\tilde{p}(x) = \frac{1}{2\pi} \int_0^{2\pi} p(x^\varphi) d\varphi.
\end{equation} 
Then, by following the derivation in Section~\ref{subsection:The limit and bias terms} we get that 
\begin{align}
\lim_{N\rightarrow\infty} C_{i,N}^1(\vartheta) = {2\pi} \int_{\man}  \exp{\left\lbrace-{\left\Vert x_i^\vartheta - x \right\Vert^2}{/\varepsilon}\right\rbrace} f(x) \tilde{p}(x) dx,
\end{align}
which is the same expression as in the case of uniform distribution except for the added density $\tilde{p}$.
In a similar way, we also get that the analogue of~\eqref{eq:denom lim} in the case of non-uniform density is
\begin{align}
\lim_{N\rightarrow\infty} C_{i,N}^2(\vartheta) = \lim_{N\rightarrow\infty} \frac{1}{N}\sum_{j=1}^N \int_0^{2\pi} \exp{\left\lbrace-{\left\Vert x_i^\vartheta - x_j^\varphi\right\Vert^2}{/\varepsilon}\right\rbrace} d\varphi = {2\pi} \int_{\man}  \exp{\left\lbrace-{\left\Vert x_i^\vartheta - x \right\Vert^2}{/\varepsilon}\right\rbrace} \tilde{p}(x) dx.
\end{align}
If we use these results, then the equivalent of~\eqref{eq:fraction limit} for non-uniform density is
\begin{align}
\lim_{N\rightarrow\infty}\frac{4}{\varepsilon} \left\{\tilde{L}g\right\}(i,\vartheta) = 
\frac{4}{\varepsilon}\left[f(x_i^\vartheta) -  \frac{\int_{\man}  \exp{\left\lbrace-{\left\Vert x_i^\vartheta - x \right\Vert^2}{/\varepsilon}\right\rbrace} f(x) \tilde{p}(x) dx}{\int_{\man}  \exp{\left\lbrace-{\left\Vert x_i^\vartheta - x \right\Vert^2}{/\varepsilon}\right\rbrace} \tilde{p}(x) dx}\right],
\end{align}
and from the results of~\cite{coifman2006diffusion} it directly follows that
\begin{align}
\lim_{\varepsilon\rightarrow 0}\lim_{N\rightarrow\infty}\frac{4}{\varepsilon} \left\{\tilde{L}g\right\}(i,\vartheta) &= \frac{\Delta_\man \left( f \cdot \tilde{p}\right)(x_i^\vartheta)}{\tilde{p}(x_i^\vartheta)} - \frac{\Delta_\man \tilde{p}(x_i^\vartheta)}{\tilde{p}(x_i^\vartheta)}\cdot f(x_i^\vartheta) \nonumber \\ 
&= \Delta_\man f(x_i^\vartheta) - 2\frac{\left\langle \nabla_\man f(x_i^\vartheta),\nabla_\man \tilde{p}(x_i^\vartheta) \right\rangle}{\tilde{p}(x_i^\vartheta)}.
\end{align}
Therefore, it is evident that the steerable graph Laplacian $\tilde{L}$ does not converge to the Laplace-Beltrami operator $\Delta_\man$, but rather to a Fokker-Planck operator which depends on the rotationally-invariant distribution $\tilde{p}$. Note that if $p$ was uniform, i.e. $p(x) = 1/\operatorname{Vol}\left\{\man\right\}$, then the two operators would coincide.

Next, following~\cite{coifman2006diffusion}, we propose to normalize the sampling density by considering a re-weighted version of the steerable affinity operator $W_{i,j}(\vartheta,\varphi)$. Specifically, we define
\begin{align}
&\bar{W}_{i,j}(\vartheta,\varphi) = \frac{W_{i,j}(\vartheta,\varphi)}{D_{i,i} D_{j,j}}, \\
&\bar{D}_{i,i} = \sum_{j=1}^{N} \int_0^{2\pi} \bar{W}_{i,j}(0,\alpha) d\alpha,
\end{align} 
and then the density-normalized steerable graph Laplacian $\bar{L}$ is defined via
\begin{equation}
\bar{L}f = f - \bar{D}^{-1} \bar{W} f.
\end{equation}
Note that we can write
\begin{equation}
\left\{\bar{D}^{-1} \bar{W}\right\}_{i,j}(\vartheta,\varphi) = \frac{W_{i,j}(\vartheta,\varphi) / D_{j,j}}{\sum_{j=1}^{N} \int_0^{2\pi} \left[ W_{i,j}(\vartheta,\varphi) / D_{j,j} \right] d\alpha} 
= \frac{N^{-1}{W_{i,j}(\vartheta,\varphi)}/{(D_{j,j}/N)}}{N^{-1}\sum_{j=1}^{N} \int_0^{2\pi} \left[ W_{i,j}(\vartheta,\varphi) / (D_{j,j}/N) \right] d\alpha},
\end{equation}
where we have that
\begin{align}
\lim_{N\rightarrow\infty}\frac{1}{N}{D}_{j,j} = \lim_{N\rightarrow \infty} \frac{1}{N} \sum_{k=1}^{N} \int_0^{2\pi} {W}_{j,k}(0,\alpha) d\alpha = \lim_{N\rightarrow\infty} C_{j,N}^2(0)
= {2\pi} \int_{\man}  \exp{\left\lbrace-{\left\Vert x_j - x \right\Vert^2}{/\varepsilon}\right\rbrace} \tilde{p}(x) dx,
\end{align}
and therefore
\begin{equation}
\lim_{N\rightarrow\infty}\frac{4}{\varepsilon} \left\{\bar{L}g\right\}(i,\vartheta) = \frac{4}{\varepsilon}\left[ f(x_i^\vartheta) - \frac{\int_{\man}{\exp{\left\lbrace-{\left\Vert x_i^\vartheta - x \right\Vert^2}{/\varepsilon}\right\rbrace} f(x) \hat{p}(x) dx} }{\int_{\man}{\exp{\left\lbrace-{\left\Vert x_i^\vartheta - x \right\Vert^2}{/\varepsilon}\right\rbrace} \hat{p}(x) dx} }\right],
\end{equation}
where $\hat{p}(x)$ is a ``corrected'' density given by
\begin{equation}
\hat{p}(x) = \frac{\tilde{p}(x)}{\int_{\man} \exp{\left\lbrace-{\left\Vert x - y \right\Vert^2}{/\varepsilon}\right\rbrace} \tilde{p}(y) dy}.
\end{equation}
Lastly, the derivation in~\cite{coifman2006diffusion} establishes that 
\begin{equation}
\lim_{\varepsilon\rightarrow 0} \lim_{N\rightarrow\infty}\frac{4}{\varepsilon} \left\{\bar{L}g\right\}(i,\vartheta) =  \Delta_{\man} f(x_i^\vartheta).
\end{equation}

\section{Proof of Theorem~\ref{thm:eigenfunctions and eigenvalues of L_tilde}}
\label{appendix:proof of spectral properties of L_tilde}
We mention that this proof follows very closely the proof of Theorem~\ref{thm:eigenfunctions and eigenvalues of L}.
\begin{proof}
First, as
\begin{equation}
\left\{D^{-1}Wf\right\}(i,\vartheta) = \sum_{j=1}^N \int_0^{2\pi} \left(W_{i,j}(\vartheta,\varphi)/D_{i,i}\right) f_j(\varphi) d\varphi,
\end{equation}
it is evident that $D^{-1}W$ is also LRI (since $D^{-1}W$ merely alters $W_{i,j}$ by constant factors independent of $\vartheta$), and hence $L = I-D^{-1}W$ is of the form $A+G$ as required by Proposition~\ref{prop:eigenfunctions and eigenvalues of H}.
Therefore, we can obtain a sequence of eigenfunctions and eigenvalues of $L$ by diagonalizing the matrices $\tilde{S}_m = I-D^{-1}\hat{W}^{(m)}$ for every $m\in\mathbb{Z}$. However, it is important to mention that in contrast to $S_m=D-\hat{W}^{(m)}$, the matrix $\tilde{S}_m=I-D^{-1}\hat{W}^{(m)}$ is not Hermitian. Nonetheless, if we make the observation that $\tilde{S}_m$ is similar to the Hermitian matrix $S_m^{'} = I - D^{-1/2}\hat{W}^{(m)} D^{-1/2}$ by
\begin{equation}
D^{1/2} \tilde{S}_m D^{-1/2} = S^{'}_m,
\end{equation}
then it follows that $\tilde{S}_m$ can be diagonalized with a set of eigenvectors complete in $\mathbb{C}^N$ and the eigenvalues of $S_m^{'}$, which are real-valued.
Then, as it follows from Theorem~\ref{thm:eigenfunctions and eigenvalues of L} that the eigenvalues of $S_m$ are non-negative, we have that the eigenvalues of $S_m^{'} = D^{-1/2} S_m D^{-1/2}$ must be also non-negative (since surely $v^* D^{-1/2} S_m D^{-1/2} v \geq 0 $ for any $v\in\mathbb{C}^N$), which lastly implies that the eigenvalues of $\tilde{S}_m$ are non-negative.
As to the fact that $\left\{\Phi_{m,k}\right\}_{m,k}$ are complete in $\mathcal{H}$, the same arguments as in the proof of Theorem~\ref{thm:eigenfunctions and eigenvalues of L} hold.
\end{proof}

\end{appendices}

\bibliographystyle{plain}
\bibliography{mybib}

\end{document}